\documentclass{article}

\PassOptionsToPackage{numbers, compress}{natbib}

\usepackage[preprint]{neurips_2021}

\usepackage[utf8]{inputenc} 
\usepackage[T1]{fontenc}    
\usepackage{hyperref}
\usepackage{booktabs}
\usepackage{amsfonts}
\usepackage{nicefrac}
\usepackage{microtype}      
\usepackage{xcolor}         

\usepackage{amsmath,amssymb,amsthm,bm,bbm}
\usepackage{mathrsfs}

\usepackage{graphicx}
\usepackage{subcaption}
\usepackage{wrapfig}

\usepackage{algorithm}
\usepackage{algpseudocode}

\usepackage{enumitem}
\usepackage{etoc}
\usepackage{multibib}
\usepackage{pifont}

\newtheorem{theorem}{Theorem}
\newtheorem{lemma}{Lemma}

\newtheorem{innercustomgeneric}{\customgenericname}
\providecommand{\customgenericname}{}
\newcommand{\newcustomtheorem}[2]{%
  \newenvironment{#1}[1]
  {%
   \renewcommand\customgenericname{#2}%
   \renewcommand\theinnercustomgeneric{##1}%
   \innercustomgeneric
  }
  {\endinnercustomgeneric}
}
\newcustomtheorem{customtheorem}{Theorem}
\newcustomtheorem{customlemma}{Lemma}

\theoremstyle{definition}

\DeclareMathOperator*{\argmax}{arg\,max}

\newcommand{\tr}[1]{{#1}^\intercal}

\makeatletter
\algnewcommand{\LineComment}[1]{\Statex \hskip\ALG@thistlm \(\backslash \backslash\,\) \textcolor{gray}{#1}}
\algnewcommand{\IndentLineComment}[1]{\Statex \hskip\ALG@tlm \(\backslash \backslash\,\) \textcolor{gray}{#1}}
\makeatother

\makeatletter
\newcommand{\appendixref}[2]{\@ifundefined{r@#1}{#2}{\ref{#1}}}
\makeatother
\newif\ifappendix
\appendixtrue 

\allowdisplaybreaks

\newcommand{\cmark}{\ding{51}}
\newcommand{\xmark}{\ding{55}}

\newcites{supp}{References}

\addtocontents{toc}{\setcounter{tocdepth}{-10}}

\title{Risk-Aware Transfer in Reinforcement Learning \\
using Successor Features}

\author{%
  Michael Gimelfarb\thanks{Affiliate to Vector Institute, Toronto, Canada.} \\
   University of Toronto \\
   \texttt{mike.gimelfarb@mail.utoronto.ca} \\
   \And
   Andr\'e Barreto \\
   DeepMind \\
   \texttt{andrebarreto@google.com} \\
   \AND
   Scott Sanner\footnotemark[1] \\
   University of Toronto \\
   \texttt{ssanner@mie.utoronto.ca} \\
   \And
   Chi-Guhn Lee \\
   University of Toronto \\
   \texttt{cglee@mie.utoronto.ca}
}

\begin{document}

\maketitle

\begin{abstract}
Sample efficiency and risk-awareness are central to the development of practical reinforcement learning (RL) for complex decision-making. The former can be addressed by transfer learning and the latter by optimizing some utility function of the return. However, the problem of transferring skills in a risk-aware manner is not well-understood. In this paper, we address the problem of risk-aware policy transfer  between tasks in a common domain that differ only in their reward functions, in which risk is measured by the variance of reward streams. Our approach begins by extending the idea of generalized policy improvement to maximize entropic utilities, thus extending policy improvement via dynamic programming to sets of policies \emph{and} levels of risk-aversion. Next, we extend the idea of successor features (SF), a value function representation that decouples the environment dynamics from the rewards, to capture the variance of returns. Our resulting risk-aware successor features (RaSF) integrate seamlessly within the RL framework, inherit the superior task generalization ability of SFs, and incorporate risk-awareness into the decision-making. Experiments on a discrete navigation domain and control of a simulated robotic arm demonstrate the ability of RaSFs to outperform alternative methods including SFs, when taking the risk of the learned policies into account. 
\end{abstract}

\section{Introduction}

\emph{Reinforcement learning} (RL) is a general framework for solving sequential decision-making problems, in which an agent interacts with an environment and receives continuous feedback in the form of rewards. However, many classical algorithms in RL do not explicitly address the need for \emph{safety}, making them unreliable and difficult to deploy in many real-world applications \citep{dulac2019challenges}. One reason for this is the relative \emph{sample inefficiency} of model-free RL algorithms, which often require millions of costly or dangerous interactions with the environment or fail to converge altogether \citep{van2018deep,yu2018towards}. Transfer learning addresses these problems by incorporating prior knowledge or skills \citep{lazaric2012transfer,taylor2009transfer}. Despite this, using the expected return as a measure of optimality could still lead to undesirable behavior such as excessive risk-taking, since low-probability catastrophic outcomes with negative reward and high variance could be underrepresented \citep{moldovan2012}. For this reason, risk-awareness is becoming an important aspect in the design of practical RL \citep{garcia2015comprehensive}. Thus, an ultimate goal of developing reliable systems should be to ensure that they are both sample efficient and risk-aware.


We take a step in this direction by studying the problem of risk-aware policy transfer between tasks with different goals. 
A powerful way to tackle this problem in the risk-neutral setting is the \emph{GPI/GPE} framework, of which \emph{successor features} (SF) are a notable example \citep{barreto2017successor}. Here, \emph{generalized policy improvement} (GPI) provides a theoretical framework for transferring policies with monotone improvement guarantees, while \emph{generalized policy evaluation} (GPE) facilitates the efficient evaluation of policies on novel tasks and is a key component in satisfying the assumptions of GPI in practice. Together, GPI/GPE provide strong transfer benefits in novel task instances even before any direct interaction with them has taken place, a phenomenon we call \emph{task generalization}. The key to the superb generalization of GPI/GPE lies in their ability to directly exploit the structure of the task space, taking advantage of subtle differences and commonalities between task goals to transfer skills seamlessly in a \emph{composable} manner. This property could be an effective way of tackling problems in \emph{offline RL} \citep{levine2020offline}, such as the transfer of skills learned in a simulator to a real-world environment. However in many cases, such as helicopter flight control for example \citep{gehring2013smart}, making one wrong decision could lead to catastrophic outcomes. Hence, being risk-aware could offer one way to avoid worst-case outcomes when transferring skills in real-world settings.


\paragraph{Contributions.} We contribute a novel successor feature framework for transferring policies with the goal of maximizing the entropic utility of return in episodic (Section \ref{subsec:entropic_maximization}) and discounted (Section \appendixref{subsec:discounted_setting}{A.2}) MDPs. Intuitively, the entropic utility encourages agents to follow policies with predictable and controllable returns characterized by low variance, thus providing a natural way to incorporate risk-awareness. Furthermore, while our theoretical framework could be extended to other classes of utility functions, the entropic utility has many favorable mathematical properties \citep{follmer2002convex,kupper2009representation} that we exploit directly in this work to achieve \emph{optimal} transfer (Lemma \ref{lem:entropic_properties}, Theorem \ref{thm:gpi_aux} and \ref{thm:gpi}). We first show that \emph{risk-neutral} policy evaluation can break the optimality guarantees of GPI with respect to the entropic utility, even when the source policies optimize entropic utilities (Section \ref{subsec:counterexample}). We then show that by incorporating \emph{risk-sensitive} policy evaluation into GPE, the strong theoretical guarantees of GPI carry through to the risk-aware setting (Section \ref{subsec:gpi}). Next, we derive a form of risk-aware GPE based on the mean-variance approximation, in which the sufficient statistics of the return distribution can be computed directly (Section \ref{subsec:gpe}) or by leveraging recent developments in distributional RL \citep{bellemare2017distributional}. 

Our resulting approach, which we call \emph{Risk-Aware Successor Features} (RaSF), exploits the task structure to achieve task generalization, where emphasis is placed on avoiding high volatility of returns. Our approach is also complementary to other advances in successor features, including feature learning \citep{barreto2018transfer}, universal approximation \citep{borsa2018universal}, exploration \citep{janz2018successor}, and non-stationary reward preferences \citep{barreto2020fast}. Empirical evaluations on discrete navigation and continuous robot control domains (Section \ref{sec:experiments}) demonstrate the ability of RaSFs to better manage the trade-off between return and risk and avoid catastrophic outcomes, while providing excellent generalization on novel tasks in the same domain.

\paragraph{Related Work.}
The entropic and mean-variance objectives are popular ways of incorporating risk-awareness in RL \citep{bisi2020risk,fei2020risk,jain2021variance,mannor2013algorithmic,mao2019variance,nass2019entropic,shen2014risk, tamar2012policy,whiteson2021mean}. However, transferring learned skills between tasks while taking risk into account is a difficult problem. One way to implement risk-aware transfer is to learn a critic \citep{srinivasan2020learning} or teacher \citep{turchetta2020safe} that can guide an agent toward safer behaviors on future tasks. The risk-aware transfer of a policy from a simulator to a real-world setting has also been studied in the area of robotics \citep{held2017probabilistically}. Another approach for reusing policies is the \emph{probabilistic policy reuse} of \citet{garcia2019}, but requires strong assumptions on the task space. \emph{Hierarchical RL} (HRL) is another related approach that relies on hierarchical abstractions, enabling an agent to decompose tasks into a hierarchy of sub-tasks, and facilitating the transfer of temporally-extended skills from sub-tasks to the parent task. The \emph{CISR} approach of \citet{mankowitz2016situational} is the first to investigate safety explicitly within HRL, followed up by work on \emph{safe options} \citep{jain2021safe,mankowitz2018learning}. However, none of the existing work takes advantage of the compositional structure of task rewards to transfer skills while optimizing the variance-adjusted return, which is the problem we tackle in this paper (see Table \ref{table:literature}).

\begin{table}[!tb]
\centering
\small
    \begin{tabular}{c|c|c|c} \toprule
         & Transfers Skills & Exploits Task Structure & Risk-Sensitive \\ \midrule
         RL \citep{bisi2020risk,fei2020risk, jain2021variance,mannor2013algorithmic,mao2019variance,nass2019entropic,shen2014risk, tamar2012policy,whiteson2021mean} & \xmark & \xmark & \cmark \\ \midrule
         Transfer \citep{garcia2019,held2017probabilistically,jain2021safe,mankowitz2018learning,mankowitz2016situational,srinivasan2020learning,turchetta2020safe} & \cmark & \xmark & \cmark \\ \midrule
         Successor Features \citep{barreto2017successor,barreto2018transfer,barreto2020fast,borsa2018universal} & \cmark & \cmark & \xmark \\ \midrule
         \textbf{RaSF (Ours)} & \cmark & \cmark & \cmark \\
        \bottomrule
    \end{tabular}
\normalsize
\caption{Comparison of RaSF with relevant work in transfer learning and risk-aware RL.}
\label{table:literature}
\end{table}

\section{Preliminaries}

\subsection{Markov Decision Process}

Sequential decision-making in this paper follows the \emph{Markov decision process} (MDP), defined as a four-tuple $\langle \mathcal{S}, \mathcal{A}, r, P\rangle$: $\mathcal{S}$ is a set of states; $\mathcal{A}$ is a finite set of actions; $r : \mathcal{S} \times \mathcal{A} \times \mathcal{S} \to \mathbb{R}$ is a bounded reward function, where $r(s,a,s')$ is the immediate reward received upon transitioning to state $s'$ after taking action $a$ in state $s$; and $P : \mathcal{S} \times \mathcal{A} \times \mathcal{S} \to [0, \infty)$ is the transition function for state dynamics, where $P(s' | s, a)$ is the probability of transitioning to state $s'$ immediately after taking action $a$ in state $s$.

In the episodic MDP setting, decisions are made over a horizon $\mathcal{T} = \lbrace 0, 1, \dots T \rbrace$ where $T \in \mathbb{N}$. We define a \emph{stochastic Markov policy} as a mapping $\pi : \mathcal{S} \times \mathcal{T} \to \mathscr{P}(\mathcal{A})$, where $\mathscr{P}(\mathcal{A})$ denotes the set of all probability distributions over $\mathcal{A}$. Similarly, a \emph{deterministic Markov policy} is a mapping $\pi : \mathcal{S} \times \mathcal{T} \to \mathcal{A}$. In the risk-neutral setting, the goal is to find a policy $\pi$ that maximizes the expected sum of future rewards after initially taking action $a$ in state $s$,
\begin{equation*}
    Q_h^\pi(s,a) = \mathbb{E}_{s_{t+1}\sim P(\cdot | s_t, a_t)}\left[\sum_{t=h}^T r(s_t, a_t, s_{t+1}) \,\Big|\, s_h = s,\, a_h = a,\, a_t \sim \pi_t(s_t) \right].
\end{equation*}
In this case, it is possible to show that a deterministic Markov policy $\pi^*$ is optimal \citep{puterman2014markov}. The theoretical framework in this paper also allows for time-dependent reward or transition functions.

\subsection{Entropic Utility Maximization}
\label{subsec:entropic_maximization}

We incorporate risk-awareness into the decision-making by maximizing the \emph{entropic utility} $U_\beta$ of the cumulative reward, defined for a fixed $\beta \in \mathbb{R}$ as
\begin{equation}
\label{eqn:entropic}
    U_\beta[R] = \frac{1}{\beta} \log\mathbb{E}\left[e^{\beta R}\right],
\end{equation}
for real-valued random variables $R$ on a bounded support $\Omega \subset \mathbb{R}$. An important property of the entropic utility is the Taylor expansion $U_\beta[R] = \mathbb{E}[R] + \frac{\beta}{2} \mathrm{Var}[R] + O(\beta^2)$. Interpreting the risk as return variance, $\beta$ can now be interpreted as the risk aversion of the agent: choosing $\beta < 0$ ($\beta > 0$) leads to \emph{risk-averse} (\emph{risk-seeking}) behavior, while $\beta = 0$ is \emph{risk-neutral}, e.g. $U_0[R] = \mathbb{E}[R]$. 

Specializing (\ref{eqn:entropic}) to the MDP setting, the goal is to maximize
\begin{equation}
\label{eqn:entropic_return}
    \mathcal{Q}_{h,\beta}^\pi(s,a) = U_\beta\left[\sum_{t=h}^T r(s_t, \pi_t(s_t), s_{t+1})\right]
\end{equation}
over all policies starting from $s_h = s$ and $a_h = a$. As in the risk-neutral setting, it is possible to show that a deterministic Markov policy is optimal \citep{bauerle2014more}. Furthermore, $\mathcal{Q}_{h,\beta}^\pi$ can be computed iteratively through time using the \emph{Bellman equation} \citep{dowson2020multistage,osogami2012robustness}:
\begin{equation}
\label{eqn:entropic_bellman}
    \begin{aligned}
        \mathcal{Q}_{h,\beta}^\pi(s,a) 
        &= U_\beta\left[r(s,a,s') + \mathcal{Q}_{h+1,\beta}^\pi(s',\pi_{h+1}(s')) \right] \\
        &= \frac{1}{\beta}\log \mathbb{E}_{s'\sim P(\cdot | s, a)}\left[\exp{\left\lbrace\beta\left( r(s,a,s') + \mathcal{Q}_{h+1,\beta}^\pi(s',\pi_{h+1}(s')) \right) \right\rbrace} \right],
    \end{aligned}
\end{equation}
starting with $\mathcal{Q}_{T+1,\beta}^\pi(s,a) = 0$. In fact, (\ref{eqn:entropic}) is the \emph{only} utility function that has this equivalence and other key properties (Lemma \ref{lem:entropic_properties}) while also satisfying \emph{time consistency} that ensures the learned risk-aware behaviors remain consistent across time \citep{kupper2009representation}. In this paper, we use (\ref{eqn:entropic_bellman}) to establish a general GPI framework for risk-aware transfer learning with provable guarantees, and leverage approximations of (\ref{eqn:entropic_return}) to learn portable policy representations. 

In {reinforcement learning}, the Bellman equation is not applied directly since it suffers from the curse of dimensionality when $\mathcal{S}$ is high-dimensional or continuous, and since neither the dynamics nor the reward function are often known. Instead, the agent interacts with the environment using a stochastic exploration policy ${\pi}^e$, collects trajectories $\lbrace (s_t, a_t, s_{t+1}, r_{t+1}) \rbrace_{t=0}^{T-1}$, and updates $\mathcal{Q}_{h,\beta}(s_t, a_t)$ via sample approximations $\hat{U}_\beta \approx {U}_{\beta}$ \citep{shen2014risk}. Our goal is to ameliorate the relative sample-inefficiency of RL through transfer learning that we aim to generalize to the risk-aware setting.

\subsection{Transfer Learning}

We now formalize the general \emph{transfer learning} problem. Let $\mathcal{M}$ be the set of all MDPs with shared transition function $P$ but different (bounded) reward functions. A fixed set of source tasks $M_1, \dots M_n \in \mathcal{M}$ is instantiated, and their corresponding optimal policies $\pi_1, \dots \pi_n$ are estimated. Our main goal is to transfer these resulting source policies to a new target task $M_{n+1} \in \mathcal{M}$, to obtain a policy $\pi_{n+1}^*$ whose utility is better than one learned from scratch using only a fixed number of samples from $M_{n+1}$. We refer to this outcome as \emph{positive transfer}.

As discussed earlier, a standard way to implement transfer learning is the {GPI/GPE} framework of \citet{barreto2017successor}. 
The core mechanism that enables positive transfer in the risk-neutral setting is called \emph{generalized policy improvement} (GPI). Specifically, the set of source policies $\pi_1, \dots \pi_n$ are evaluated on the target task $M_{n+1}$ to obtain corresponding values $Q_{n+1}^{\pi_1}, \dots Q_{n+1}^{\pi_n}$. Given a mechanism that can perform this policy evaluation step efficiently with some small error $\varepsilon$ --- namely successor features discussed and extended in Section \ref{subsec:gpe} --- an agent then selects actions in a greedy manner by following policy $\pi(s) \in \argmax_a \max_{j=1\dots n} Q_{n+1}^{\pi_j}(s,a)$ in state $s$. The policy $\pi$ corresponds to a strict policy improvement operator, and thus fulfills our requirements for positive transfer.

\section{Risk-Aware Transfer Learning}

An obvious challenge of applying GPI in the risk-aware setting is that transferring optimal risk-neutral source policies does not guarantee risk-aware optimality in the target task. A much stronger observation is that, even if the source policies $\pi_j$ are risk-aware, performing the policy evaluation step in a risk-neutral way can still break the risk-awareness of GPI. This makes the extension of GPI to the risk-aware setting a non-trivial problem.

\subsection{A Motivating Example}
\label{subsec:counterexample}

\begin{wrapfigure}{r}{0.38196601125\linewidth}
    \centering
    \includegraphics[width=0.333\linewidth]{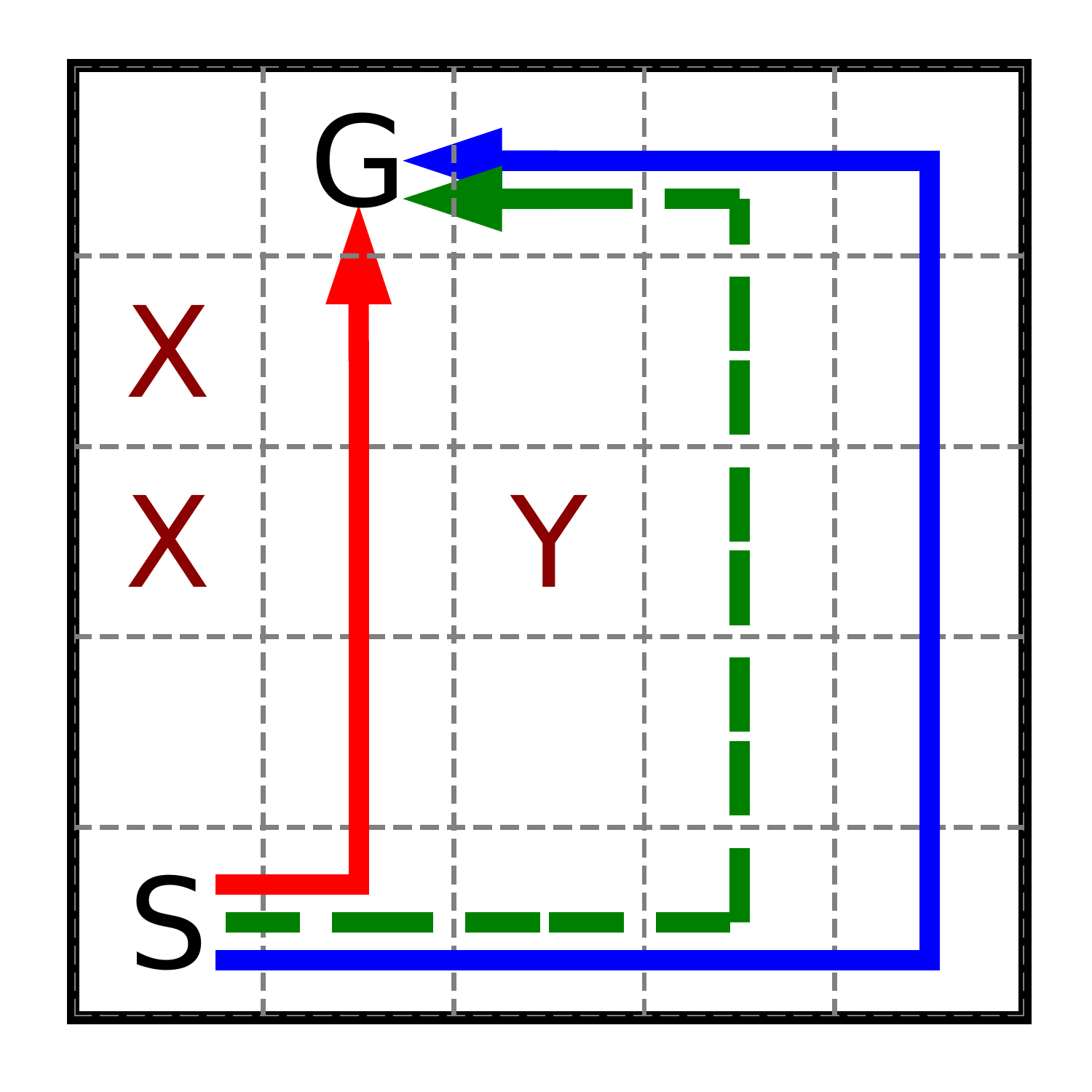}%
    \includegraphics[width=0.333\linewidth]{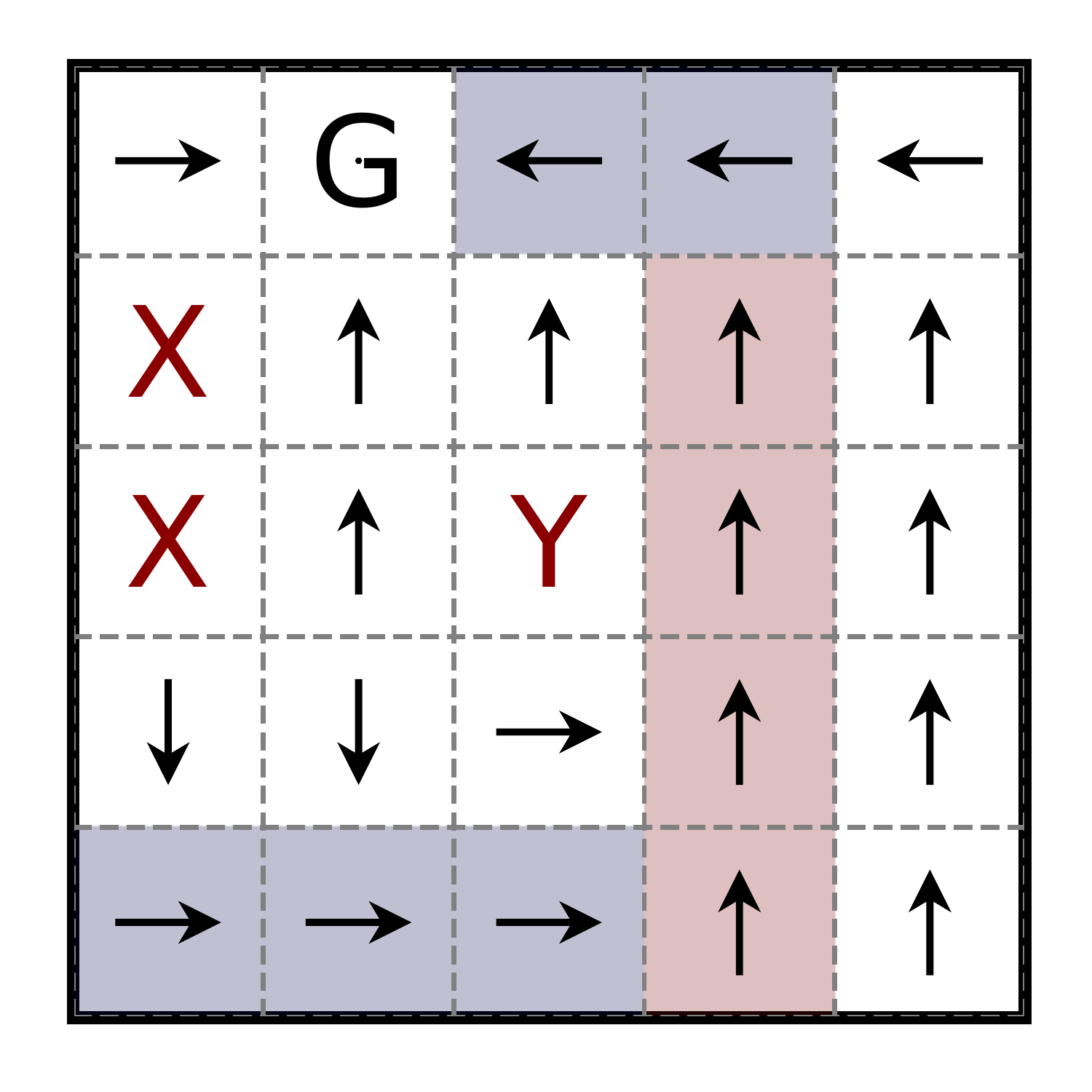}%
    \includegraphics[width=0.333\linewidth]{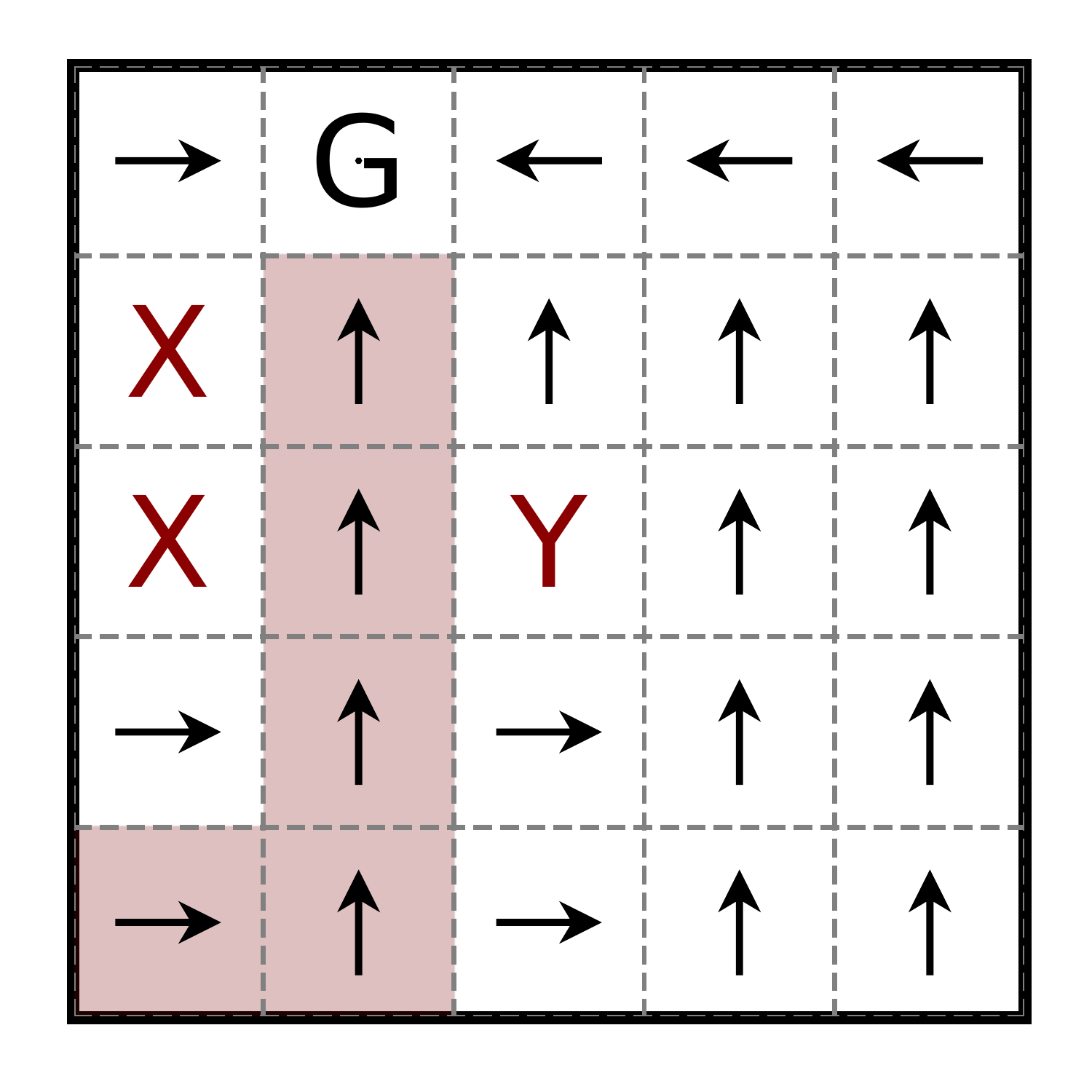}
    \includegraphics[width=0.999\linewidth]{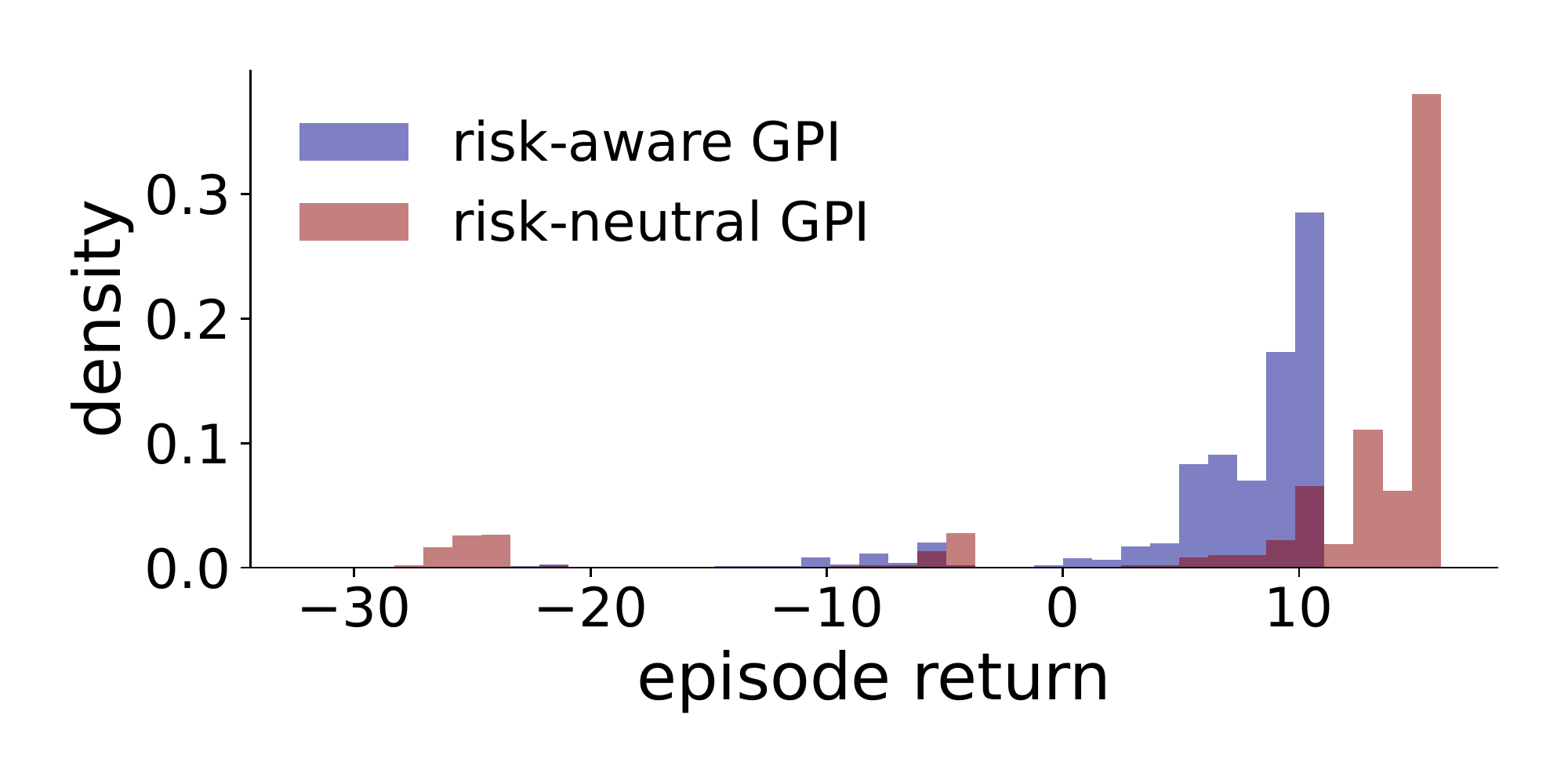}
    \caption{Comparing risk-aware and risk-neutral GPI. The risk-aware (for $\beta=-0.1$) GPI policy is shown in the top-middle plot, while the risk-neutral (for $\beta=0$) GPI policy is shown in the top-right plot.}
    \label{fig:counterexample}
\end{wrapfigure}
To see this, consider the MDP shown in Figure \ref{fig:counterexample}, which involves navigating from an initial state `S' to a goal state `G' in a grid with stochastic transitions. Traps of two types (\textcolor{red}{X, Y}) are placed in fixed cells, which upon entry terminate the episode with fixed costs, summarized compactly as pairs $(c_1, c_2)$. We define two source tasks with costs $(20,20)$ and $(0, 0)$ and a target task with costs $(20,0)$. The optimal policies for $\beta = -0.1$ induce three different trajectories: safe (\textcolor{blue}{blue}) and hazardous routes (\textcolor{red}{red}) for the source tasks, and a relatively safe route (\textcolor{green}{green}) for the target task.

We compute the GPI policies $\pi$ with risk-averse ($\beta=-0.1$) and risk-neutral ($\beta = 0$) policy evaluation, as shown in the top row on the two rightmost plots in Figure \ref{fig:counterexample}. The bottom plot shows the distribution of returns collected over $5,\!000$ test runs by acting according to the two GPI policies. Risk-averse policy evaluation results in an optimal risk-averse GPI policy corresponding to the green trajectory, whereas risk-neutral policy evaluation does not, even though both source policies are optimal on their corresponding tasks. Interestingly, the risk-averse GPI policy is a non-trivial stitching of the two source policies.

\subsection{Risk-Aware Generalized Policy Improvement}
\label{subsec:gpi}

Motivated by this example, we conjecture that the risk-awareness of the policy $\pi$ is primarily dependent on the way in which the source policies are evaluated in target task instances. In this section, we describe theoretical results that generalize the concept of generalized policy iteration to the problem of maximizing the entropic utility of returns. 

We begin by summarizing key properties necessary for establishing convergence of risk-aware GPI in the following lemma. 

\begin{lemma}
\label{lem:entropic_properties}
    Let $\beta \in \mathbb{R}$ and $X, Y$ be arbitrary random variables on $\Omega$. Then:
    \begin{enumerate}[label=(\arabic*),noitemsep]
        \item (monotonicity) if $\mathbb{P}(X \geq Y) = 1$ then $U_\beta[X] \geq U_\beta[Y]$
        \item (cash invariance) $U_\beta[X + c] = U_\beta[X] + c$ for every $c \in \mathbb{R}$
        \item (convexity) if $\beta < 0$ ($\beta > 0$) then $U_\beta$ is a concave (convex) function
        \item (non-expansion) for $f, g : \Omega \to \Omega$, it follows that 
        \begin{equation*}
            \left|U_\beta[f(X)] - U_\beta[g(X)]\right| \leq \sup_{P \in \mathscr{P}_X(\Omega)}\mathbb{E}_{P}|f(X)-g(X)|,
        \end{equation*}
        where $\mathscr{P}_X(\Omega)$ is the set of all probability distributions on $\Omega$ that are absolutely continuous w.r.t. the true distribution of $X$.
    \end{enumerate}
\end{lemma}

A proof is provided in Appendix
\appendixref{subsec:proofs_episodic}{B.2}. Properties (1)-(3) characterize \emph{concave utilities} \citep{follmer2002convex}, which intuitively can be seen as minimal requirements for rational decision-making: (1) a lottery that pays off more than another in every possible state of the world should always be preferred; (2) adding cash to a position should not increase its risk; and (3) the overall utility can be improved by diversifying across different risks. Property (4) is a derivative of the first three, and thus the theoretical results in this work could be extended further to the broad class of iterated concave utilities \citep{ruszczynski2010risk}.

We can now state the first main result of the paper.

\begin{theorem}[\bfseries GPI for Entropic Utility]
\label{thm:gpi_aux}
    Let $\pi_1, \dots \pi_n$ be arbitrary deterministic Markov policies with utilities $\tilde{\mathcal{Q}}_{h,\beta}^{\pi_1}, \dots \tilde{\mathcal{Q}}_{h,\beta}^{\pi_n}$ evaluated in an arbitrary task $M$, such that $|\tilde{\mathcal{Q}}_{h,\beta}^{\pi_i}(s,a) - \mathcal{Q}_{h,\beta}^{\pi_i}(s,a)| \leq \varepsilon$ for all $s \in \mathcal{S}$, $a \in \mathcal{A}$, $i = 1 \dots n$ and $h \in \mathcal{T}$. Define
    \begin{equation}
    \label{eqn:gpi_policy_aux}
        \pi_h(s) \in \argmax_{a \in \mathcal{A}} \max_{i=1\dots n} \tilde{\mathcal{Q}}_{h,\beta}^{\pi_i}(s,a), \quad \forall s \in \mathcal{S}.
    \end{equation}
    Then,
    \begin{equation*}
        \mathcal{Q}_{h,\beta}^{\pi}(s, a) \geq  \max_{i}\mathcal{Q}_{h,\beta}^{\pi_i}(s, a) - 2 ({T - h + 1}) \varepsilon, \quad h \leq T.
    \end{equation*}
\end{theorem}

Thus, evaluating the risk of source policies using $U_\beta$ provides monotone improvement guarantees for GPI, and thus satisfies our definition of positive transfer. Another significant property of the bound in Theorem \ref{thm:gpi_aux}, and the one in Theorem \ref{thm:gpi} below, is the linear separation between the optimal utility and the approximation error $\varepsilon$. Knowing how the optimality of $\pi$ explicitly depends on $\varepsilon$ and how errors are propagated throughout the transfer learning process is critical for developing reliable RL with predictable behavior, and highlights a key advantage of making GPI risk-aware. The additive dependence of the optimality on the approximation error in Theorem \ref{thm:gpi_aux} further explains why SFs are robust to approximation errors. This becomes particularly advantageous in our setting, since estimating utilities $U_\beta$ accurately with GPE is more complicated than the risk-neutral setting.

A stronger result for GPI can be derived when the source policies $\pi_1, \pi_2 \dots \pi_n$ are $\varepsilon$-optimal, and policy evaluation is once again performed using $U_\beta$. In this case, the optimality of GPI is determined by the similarity $\delta_r$ between the source and target task instances. 

\begin{theorem}
\label{thm:gpi}
    Let $\mathcal{Q}_{h,\beta}^{\pi_i^*}$ be the utilities of optimal Markov policies $\pi_i^*$ from task $M_i$ but evaluated in task $M$ with reward function $r(s,a,s')$. Furthermore, let $\tilde{\mathcal{Q}}_{h,\beta}^{\pi_i^*}$ be such that $|\tilde{\mathcal{Q}}_{h,\beta}^{\pi_i^*}(s, a) - \mathcal{Q}_{h,\beta}^{\pi_i^*}(s,a) | < \varepsilon$ for all $s \in \mathcal{S}, \, a \in \mathcal{A}$, $h \in \mathcal{T}$ and $i = 1 \dots n$, and let $\pi$ be the corresponding policy in (\ref{eqn:gpi_policy_aux}). Finally, let $\delta_r = \min_{i=1\dots n} \sup_{s,a,s'}| r(s,a,s') - r_i(s,a,s')|$. Then,
    \begin{equation*}
    \begin{aligned}
        \left|\mathcal{Q}_{h,\beta}^{\pi}(s,a) - \mathcal{Q}_{h,\beta}^*(s,a) \right| \leq 2 (T - h + 1) (\delta_r + \varepsilon), \quad h \leq T.
    \end{aligned}
    \end{equation*}
\end{theorem}

These results are proved in Appendix \appendixref{subsec:proofs_episodic}{B.2} for the episodic setting and in Appendix \appendixref{subsec:proofs_discounted}{B.3} for the discounted setting. Finally, while not required in this work, the above results could be extended to more general settings in risk-averse control \citep{ruszczynski2010risk}, though practical implementation of GPE in these settings remains an open problem.

\subsection{Risk-Aware Generalized Policy Evaluation}
\label{subsec:gpe}

Following \citet{barreto2017successor}, let $\bm{\phi} : \mathcal{S} \times \mathcal{A} \times \mathcal{S}' \to \mathbb{R}^d$ be a bounded and task-independent feature map, and consider the following linear representation of rewards,
\begin{equation*}
    r(s,a,s') = \tr{\bm{\phi}(s,a,s')} \mathbf{w}, \quad \forall s, a, s',    
\end{equation*}
where $\mathbf{w} \in \mathbb{R}^d$ is a task-dependent reward parameter. The \emph{risk-neutral} return becomes:
\begin{align}
    Q_h^\pi(s,a) 
    &= \mathbb{E}_P\left[\sum_{t=h}^T \tr{\bm{\phi}(s_t, a_t, s_{t+1})} \mathbf{w} \,\Big|\, s_h = s, \, a_h = a,\, a_t\sim\pi_t(s_t)\right] \nonumber \\
    \label{eqn:gpe_linear}
    &= \mathbb{E}_P\tr{\left[\sum_{t=h}^T \bm{\phi}(s_t, a_t, s_{t+1}) \,\Big|\, s_h = s,\, a_h = a,\, a_t \sim \pi_t(s_t)\right]} \mathbf{w} = \tr{\bm{\psi}_{h}^\pi(s,a)} \mathbf{w},
\end{align}
where $\bm{\psi}_{h}^\pi(s,a)$ are the \emph{successor features} (SFs) associated with the policy ${\pi}$. The linear dependence of the return on $\mathbf{w}$ allows for instantaneous policy evaluation in novel tasks with arbitrary reward preferences $\mathbf{w}$, making it a particular --- and perhaps the canonical --- instantiation of {GPE}. More critically, $\bm{\psi}_{h}^\pi$ can be seen as a \emph{task-independent} and highly portable linear feature representation of policies, and is the key to the generalization ability of SFs on novel task instances.

The concept of GPE can be generalized to incorporate distributions of the return. Repeating the above derivation for the entropic utility (\ref{eqn:entropic_return}), we have:
\begin{equation}
\label{eqn:utility_linear}
    \mathcal{Q}_{h,\beta}^\pi(s,a) = U_\beta\left[\sum_{t=h}^T r(s_t, \pi_t(s_t), s_{t+1})\right] = U_\beta\left[\tr{\Psi_{h}^\pi(s,a)}\mathbf{w} \right],
\end{equation}
corresponding to the random vector $\Psi_h^\pi(s,a) = \sum_{t=h}^T \bm{\phi}_t$ of unrealized feature returns. Thus, we have transformed the problem of estimating the utility of returns into the problem of estimating the distribution of $\tr{\Psi_h^\pi(s,a)}\mathbf{w}$. The key question now is how to estimate this distribution for fast GPE.

A natural way to do this is by applying a second-order Taylor expansion for $U_\beta$, since it allows us to precompute and cache the necessary moments of the return distribution:
\begin{align}
\label{eqn:mvsf}
    U_\beta\left[\tr{\Psi_{h}^\pi(s,a)}\mathbf{w} \right] 
    &= \mathbb{E}_P[\tr{\Psi_{h}^\pi(s,a)}\mathbf{w} ] + \frac{\beta}{2}\mathrm{Var}_P[\tr{\Psi_{h}^\pi(s,a)}\mathbf{w}] + O(\beta^2) \nonumber \\
    &\approx \tr{\bm{\psi}_{h}^\pi(s,a)}\mathbf{w} + \frac{\beta}{2} \tr{\mathbf{w}}\mathrm{Var}_P[\Psi_{h}^\pi(s,a)]\mathbf{w} = \mathcal{\tilde{Q}}_{h,\beta}^\pi(s,a),
\end{align}
in which $\mathrm{Var}_P[\Psi_{h}^\pi(s,a)] = \Sigma_h^\pi(s,a)$ is interpreted as a covariance matrix for SFs. In the context of Theorems \ref{thm:gpi_aux} and \ref{thm:gpi}, the term $\varepsilon$ encapsulates the errors in the approximation of $\bm{\psi}_h^\pi$ and $\Sigma_h^\pi$, plus the terms contained in $O(\beta^2)$ above. However, the main advantage of (\ref{eqn:mvsf}) is that, like (\ref{eqn:gpe_linear}), it is also analytic in $\mathbf{w}$ and allows for \emph{instantaneous policy evaluation} with arbitrary reward preferences $\mathbf{w}$. Now, $\bm{\psi}_h^\pi$ and $\Sigma_h^\pi$ together provide task-independent and portable representations of policies, while also accounting for exogenous risk. This is the key to preserving the task generalization ability of SFs in the risk-aware setting, and (\ref{eqn:mvsf}) can now be seen as a particular instantiation of GPE. We call this overall approach \emph{Risk-aware Successor Features} (RaSF). 

The simplest approaches for estimating $\Sigma_h^\pi$ in the exact (e.g. tabular) Q-learning setting are based on dynamic programming \citep{sherstan2018comparing,tamar2016learning}, which would allow the overall approach to be easily integrated into existing SF implementations. In particular, the covariance satisfies the Bellman equation
\begin{equation}
    \label{eqn:covariance_Bellman}
    \Sigma_h^\pi(s,a) = \mathbb{E}_{s' \sim P(\cdot | s, a)}\left[\bm{\delta}_h \tr{\bm{\delta}_h} +  \Sigma_{h+1}^\pi(s', \pi_{h+1}(s')) \,|\, s_h = s,\, a_h = a \right],
\end{equation}
where $\bm{\delta}_h$ are the Bellman residuals of $\bm{\psi}_h^\pi(s,a)$. The approximation $\tilde{\bm{\psi}}_h^\pi$ is known to converge to the true value ${\bm{\psi}}_h^\pi$, and a similar result also holds for updating the covariance based on (\ref{eqn:covariance_Bellman}).

\begin{theorem}[\bfseries Convergence of Covariance]
\label{thm:sigma_convergence}
    Let $\|\cdot\|$ be a matrix-compatible norm, and suppose there exists $\varepsilon : \mathcal{S} \times \mathcal{A} \times \mathcal{T} \to [0, \infty)$ such that $\|\tilde{\bm{\psi}}_h^\pi(s,a) - {\bm{\psi}}_h^\pi(s,a) \|^2 \leq \varepsilon_h(s,a)$ and $\|\mathbb{E}_{s'\sim P(\cdot | s, a)}[\tilde{\bm{\delta}}_h \tr{(\tilde{\bm{\psi}}_h^\pi(s',\pi_{h+1}(s')) - {\bm{\psi}}_h^\pi(s',\pi_{h+1}(s')))}]\| \leq \varepsilon_h(s,a)$. Then,
    \begin{equation*}
        \left\| {\Sigma}_h^\pi(s,a) - \mathbb{E}_{s' \sim P(\cdot |s,a)}\left[\tilde{\bm{\delta}}_h \tr{\tilde{\bm{\delta}}_h} + \tilde{\Sigma}_{h+1}^\pi(s',\pi_{h+1}(s')) \right] \right\| \leq 3 \varepsilon_h(s,a).
    \end{equation*}
\end{theorem}

A proof can be found in Appendix \appendixref{subsec:proofs_covariance}{B.4}. Appendix \appendixref{subsec:algorithmic_total_reward}{A.1} describes how $\tilde{\bm{\psi}}_h^\pi$ and $\tilde{\Sigma}_h^\pi$ can be learned online from environment interactions, while Appendix \appendixref{subsec:parametric}{A.4} discusses further generalizations of (\ref{eqn:mvsf}). There are, however, several limitations of estimating $\Sigma_h^\pi$ in this way. First, obtaining accurate estimates of $\Sigma_h^\pi$ requires accurate estimates of $\bm{\psi}_h^\pi$ (thus estimating one quantity on top of another), making this approach difficult to apply with deep function approximation. This claim is further substantiated by Theorem \ref{thm:sigma_convergence} and preliminary experiments. A second issue that occurs is \emph{double sampling}, when the same transitions are used to update the mean and covariance, resulting in accumulation of bias in the latter \citep{baird1995residual,zhu2020borrowing}. Our experiments on the reacher domain avoid these issues by leveraging distributional RL to approximate the mean and variance in (\ref{eqn:mvsf}), while maintaining computational efficiency.

\section{Experiments}
\label{sec:experiments}

To evaluate the performance of RaSF, we revisit the benchmark domains in \citet{barreto2017successor}, adapted for learning and evaluating risk-aware behaviors. Further details can be found in Appendix \appendixref{sec:experimental_detail}{C}. 

\paragraph{Four-Room.}
\begin{wrapfigure}{r}{0.38196601125\linewidth}
    \centering
    \includegraphics[width=0.499\linewidth]{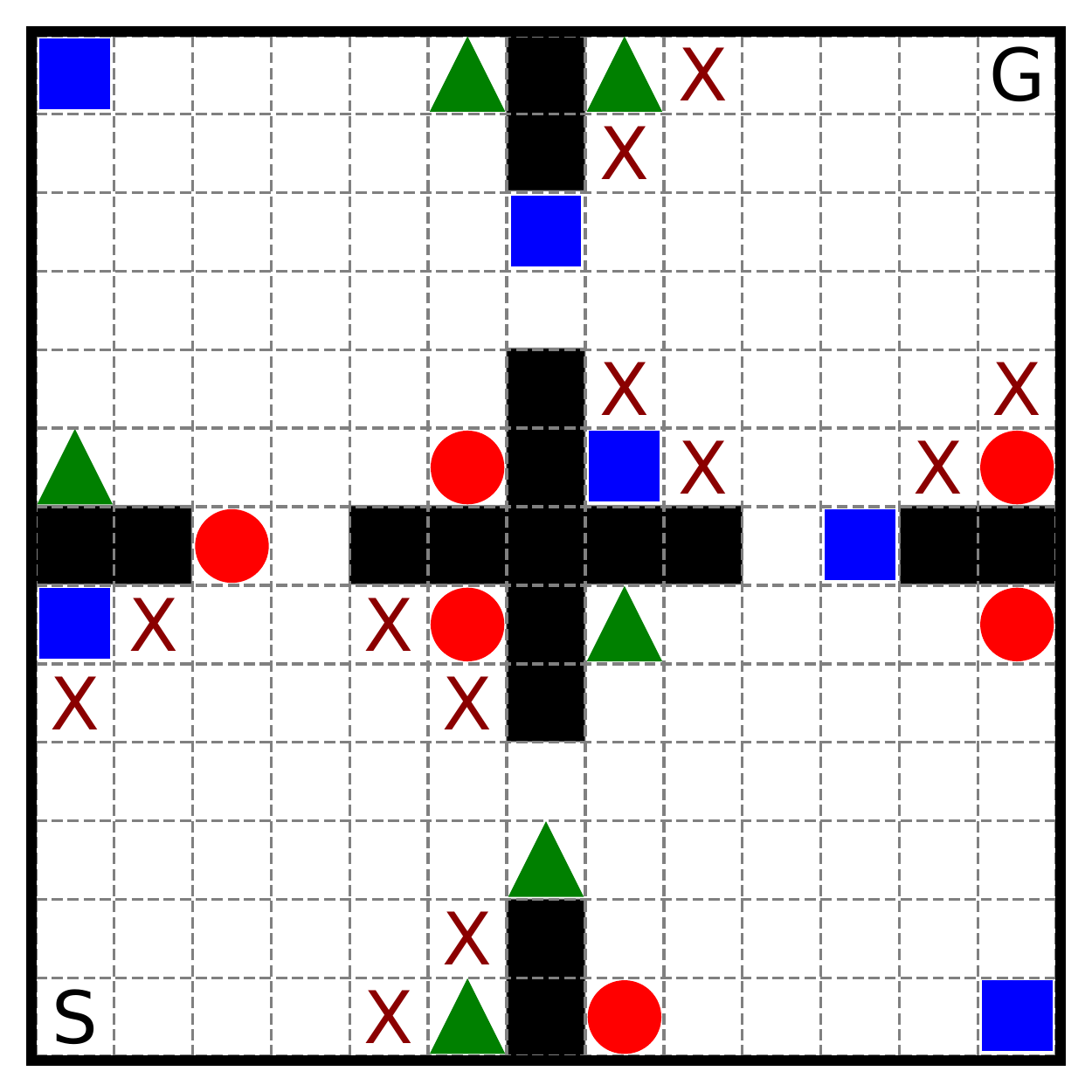}%
    \includegraphics[width=0.499\linewidth]{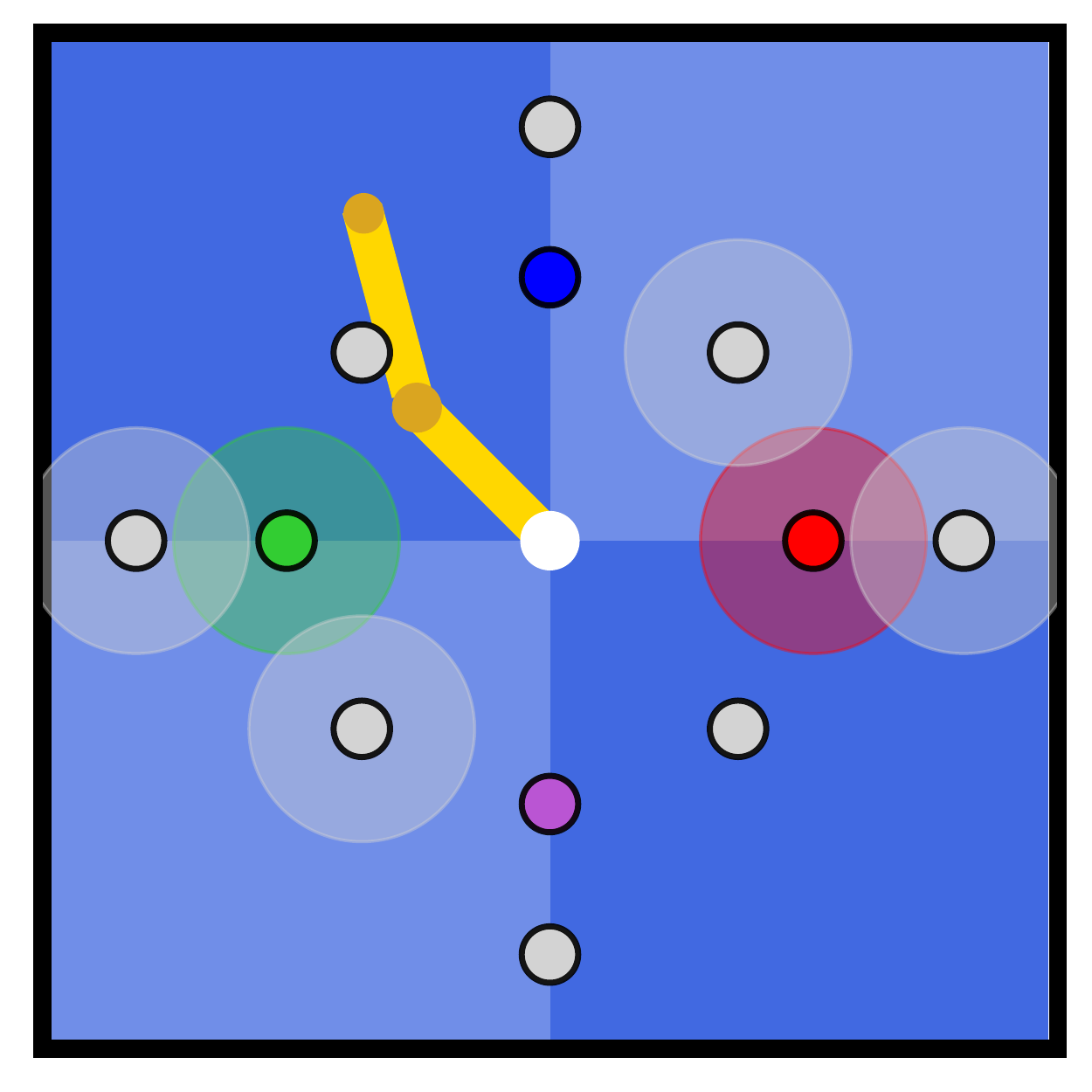}
    \caption{The four-room and reacher domains. {\bf{Four-Room:}} the shapes of the objects represent their classes, `S' is the start state and `G' is the final goal state, and (\textcolor{red}{X}) mark the states with high risk. {\bf{Reacher:}} colored and gray circles represent training and test targets, respectively, while shaded regions represent areas of high risk.}
    \label{fig:maps}
\end{wrapfigure}
The first domain consists of a set of navigation tasks defined on a discrete 2-D space divided into four rooms, as illustrated on the left in Figure \ref{fig:maps}. The environment has additional objects that can be picked up by the agent by occupying their cells. These objects belong to one of three possible classes, drawn as different shapes in Figure \ref{fig:maps}, which determine their reward. The position of the objects remains fixed, but the rewards of their classes are reset every $20,\!000$ transitions to random values sampled uniformly in $[-1, +1]$. To incorporate risk, traps are placed in fixed cells marked with \textcolor{red}{X}. For every time instant during which the agent occupies a trap cell, the trap activates spontaneously with a small probability, resulting in an immediate penalty and termination of the episode (we refer to this event as a failure). The goal is to maximize the total reward accumulated over 128 random task instances while minimizing the number of failures. 

\begin{figure}[b]
    \centering
    \includegraphics[width=0.999\linewidth]{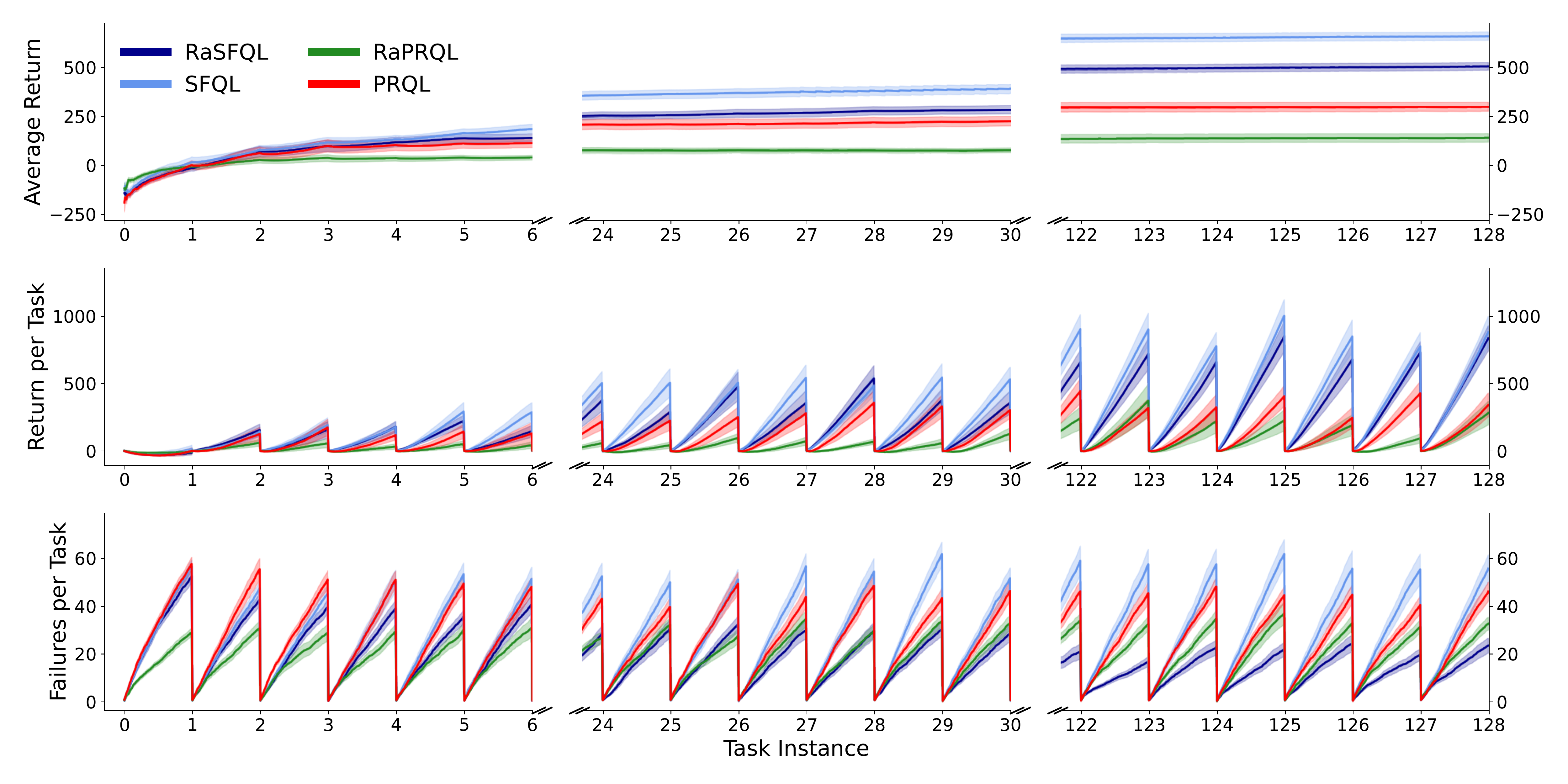}
    \caption{Average return, cumulative return and number of failures per task in the four-room domain, for $\beta = \omega = -2$. Shaded bands show one standard error over 30 independent runs.}
    \label{fig:shapes_main}
\end{figure}

In order to demonstrate the power of our approach in the absence of approximation errors, we define a simple instance of RaSFs in which $\bm{\psi}^{\pi_i}$ and $\Sigma^{\pi_i}$ are learned exactly using lookup tables and dynamic programming (equation (\ref{eqn:covariance_Bellman})). We also apply modest discounting of rewards to ensure that the Q-function converges, as is standard in RL and discussed further in Appendix \appendixref{subsec:discounted_setting}{A.2}. The $\mathbf{w}$ are also learned using immediate reward feedback and exact, sparse state features $\bm{\phi}$ provided to the agent. Due to its similarity to standard Q-learning, we call this approach RaSFQL. To provide a challenging baseline for comparison, we implemented another policy reuse algorithm (PRQL) \citep{fernandez2006probabilistic}. Further replacing the risk-neutral action selection mechanism of PRQL with \emph{smart exploration} \citep{gehring2013smart} allows PRQL to be sensitive to reward volatility, which we refer to as RaPRQL.

The performance of these algorithms is shown in Figure \ref{fig:shapes_main}. The cumulative reward obtained by RaSFQL is generally lower than SFQL, as expected since a risk-averse agent should avoid the objects in the bottom-left and top-right rooms and forgo their associated rewards. Interestingly, the performance of RaSFQL far exceeds that of RaPRQL and even PRQL, suggesting that the benefits of task generalization provided by GPI/GPE are quite strong. Furthermore, the number of failures observed by RaSFQL gradually decreases over the task instances, while the number of failures of SFQL slightly increases. This is consistent with Theorem \ref{thm:gpi_aux} that guarantees monotone improvement in the \emph{risk-adjusted} return of RaSFQL. On the other hand, while RaPRQL also learns to avoid risk, it fails slightly more often than RaSFQL. This suggests that the benefits of task generalization promised by GPI/GPE even allow risk-aware behaviors to emerge sooner than by using generic policy reuse methods that are unable to exploit the task structure, namely PRQL. This aspect becomes critical for minimizing failures when deploying a trained policy library on novel task instances in a real-world setting. Further analysis and ablation studies are provided in Appendix \appendixref{subsec:shapes_ablation}{D.1}. 

\begin{figure}[!tb]
    \centering
    \includegraphics[width=0.333\linewidth]{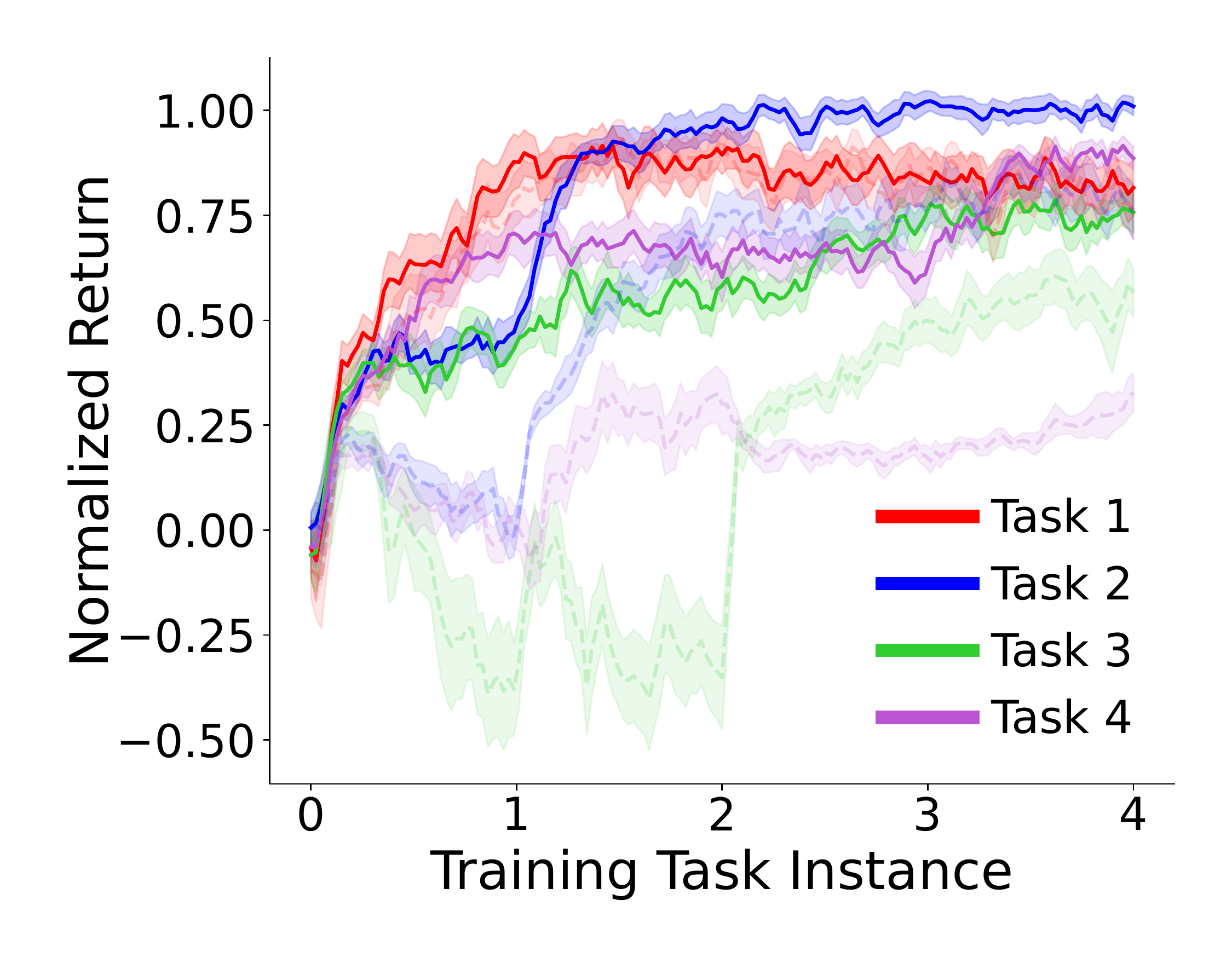}%
    \includegraphics[width=0.333\linewidth]{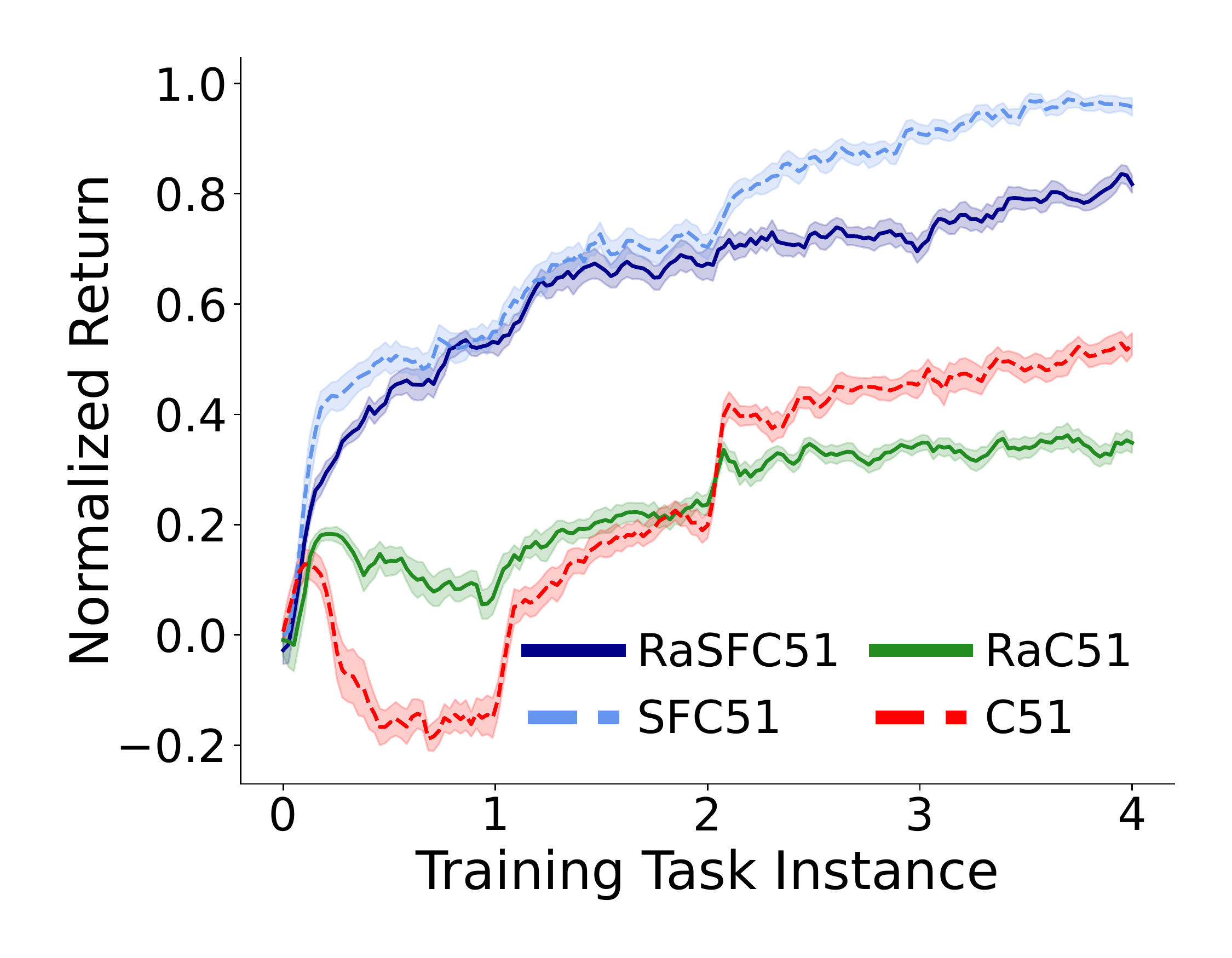}%
    \includegraphics[width=0.333\linewidth]{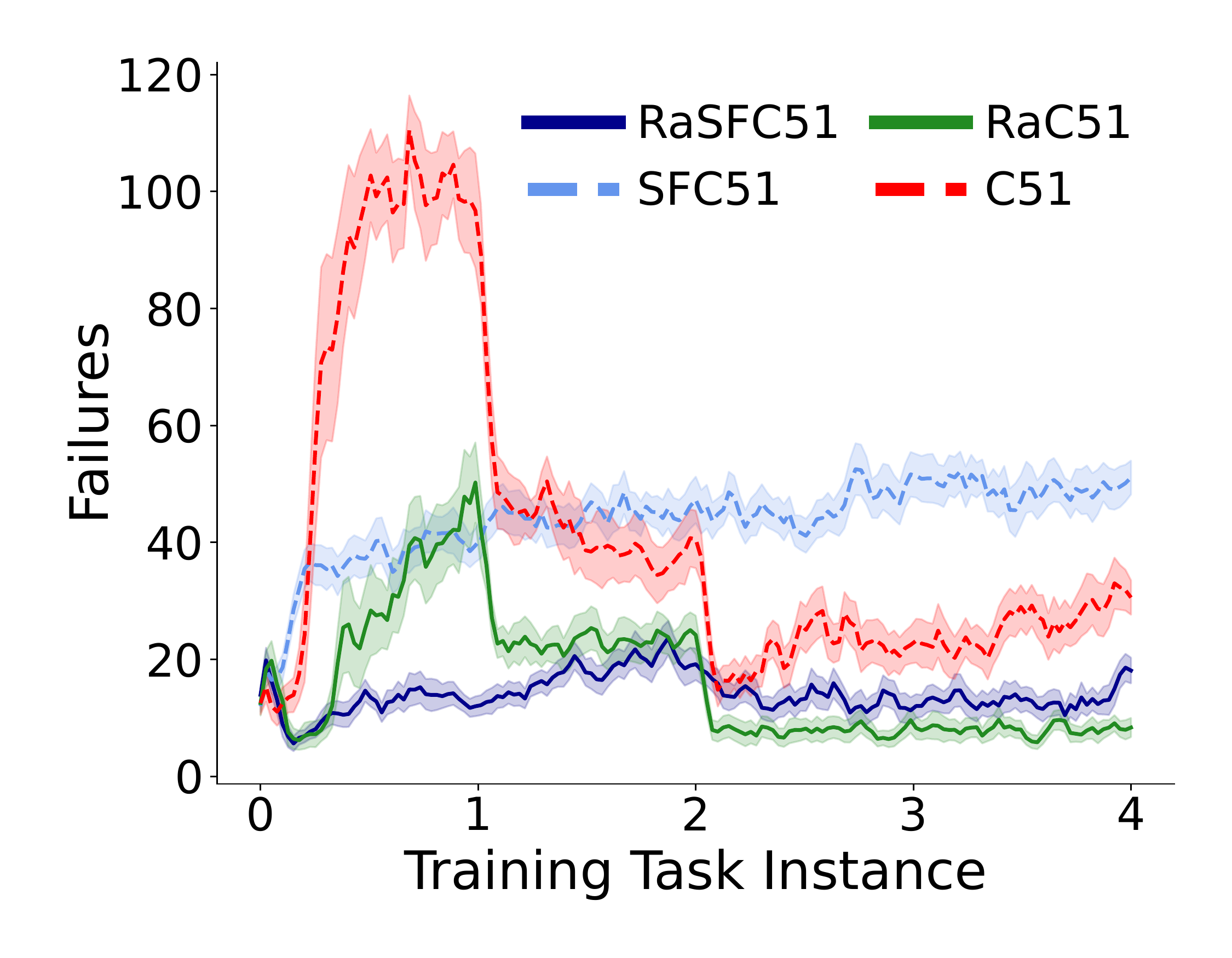}
    \caption{Performance on training and test tasks for the reacher domain. {\textbf{Left:}} Normalized return on training tasks. Faded curves correspond to C51 performance. {\textbf{Middle:}} Normalized return averaged across all test tasks. {\textbf{Right:}} Total failures across all test tasks. Shaded bands show one standard error over 10 independent runs with different seeds.}
    \label{fig:reacher_main}
\end{figure}

\paragraph{Reacher.}
The second domain consists of a set of tasks based on the MuJoCo physics engine \citep{todorov2012mujoco} that involve the maneuver of a robotic arm toward a fixed target location. As illustrated in the rightmost plot in Figure \ref{fig:maps}, the agent is only allowed to train on 4 tasks, whose target locations are indicated by colored circles, and must be able to perform well on 8 test tasks whose target locations are indicated by the grey circles. Furthermore, we incorporate two sources of reward volatility: (1) actions are perturbed by additive Gaussian noise; and (2) fixed regions around some of the target locations randomly incur negative rewards (failures), illustrated by faded circles in Figure \ref{fig:maps}. Please note that most of these high-volatility regions are centered on target locations of test tasks from which the agent never learns directly. This stresses the agent's ability to avoid unforeseen dangers in the environment, in additional to performing well on previously unseen task instances. 

\begin{wrapfigure}{r}{0.38196601125\linewidth}
    \centering
    \includegraphics[width=0.999 \linewidth]{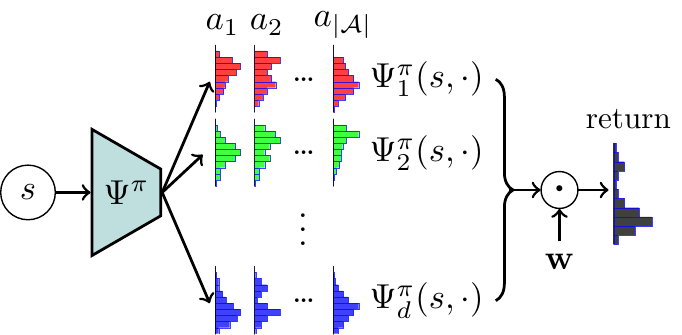}
    \caption{Architecture for $\Psi^\pi(s,a)$.}
    \label{fig:sfnetwork}
\end{wrapfigure}
As discussed earlier, it is difficult to compute $\Sigma^{\pi_i}$ using (\ref{eqn:covariance_Bellman}) in the deep RL setting. A computationally tractable way to avoid these issues is to first approximate the density of $\Psi^{\pi_i}(s,a)$ using the distributional RL framework, and then extract the moment information needed for the calculation of (\ref{eqn:covariance_Bellman}). Specifically, we apply C51 \citep{bellemare2017distributional} by modeling $\Psi_1^{\pi_i}(s,a), \dots \Psi_d^{\pi_i}(s,a)$ using histograms for each $(s,a)$. However, $\Psi^{\pi_i}$ are high-dimensional, so we avoid the curse of dimensionality by modeling $\Psi_j^{\pi_i}$ without their interaction effects. This still turns out to be an effective way of detecting high-variance scenarios in the environment. The final architecture is illustrated in Figure \ref{fig:sfnetwork}, where the marginal distributions of $\Psi^{\pi_i}$ are modeled as separate output heads with a shared state encoder. The rest of the training protocol is identical to the SFDQN of \citet{barreto2017successor}, except that DQN is replaced by the C51 architecture above, as further detailed in Appendix \appendixref{subsec:distributional_sf}{A.3}. The risk-averse and risk-neutral instances of this approach for modeling successor features are referred to as RaSFC51 and SFC51, respectively, while RaC51 and C51 replace successor features with \emph{universal value functions} \citep{schaul2015universal} for generalization across target locations.

\begin{figure}[!tb]
    \centering
    \begin{subfigure}[c]{0.32\textwidth}
        \centering
         \includegraphics[width=\linewidth]{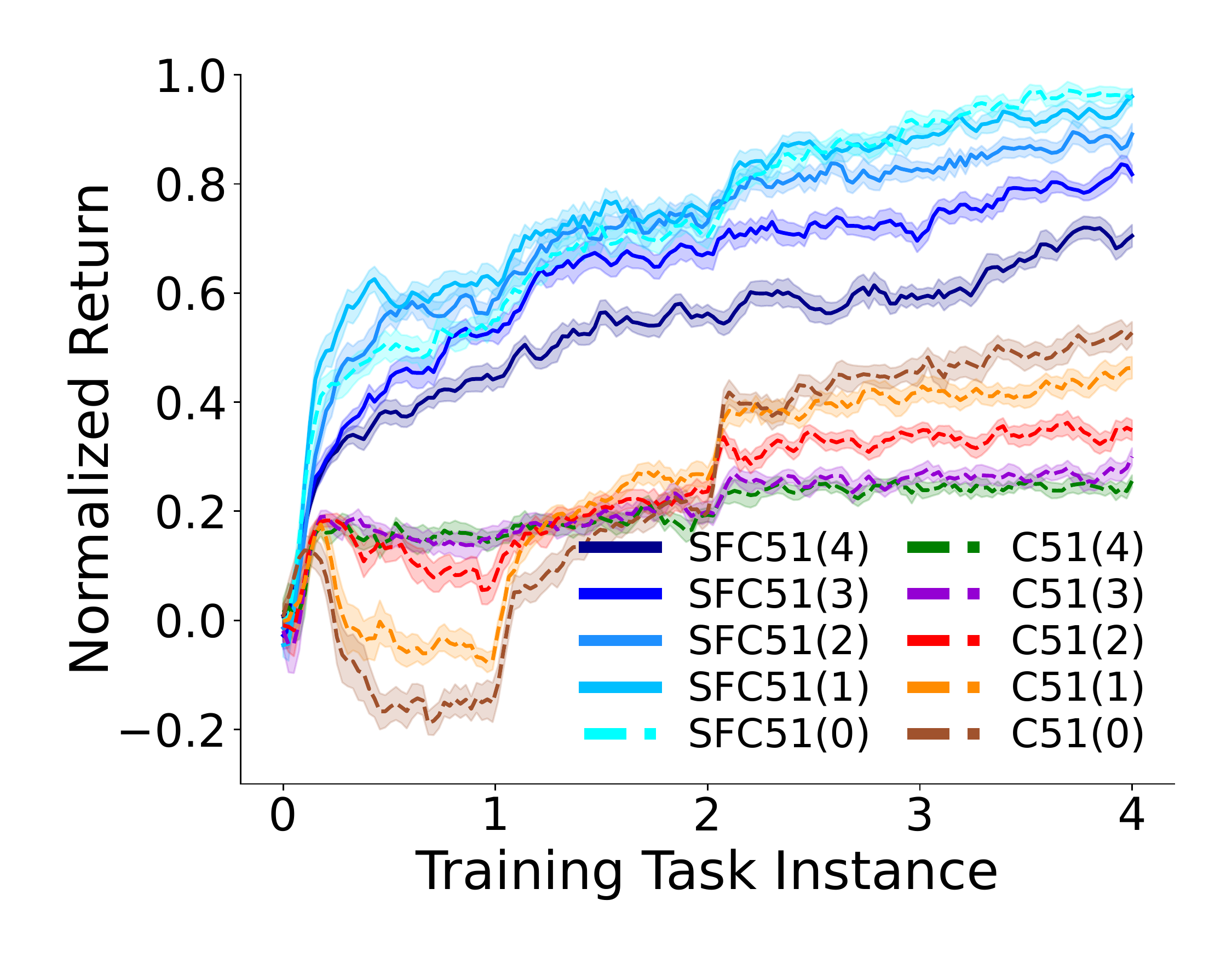} 
    \end{subfigure}%
    \begin{subfigure}[c]{0.32\textwidth}
        \centering
        \includegraphics[width=\linewidth]{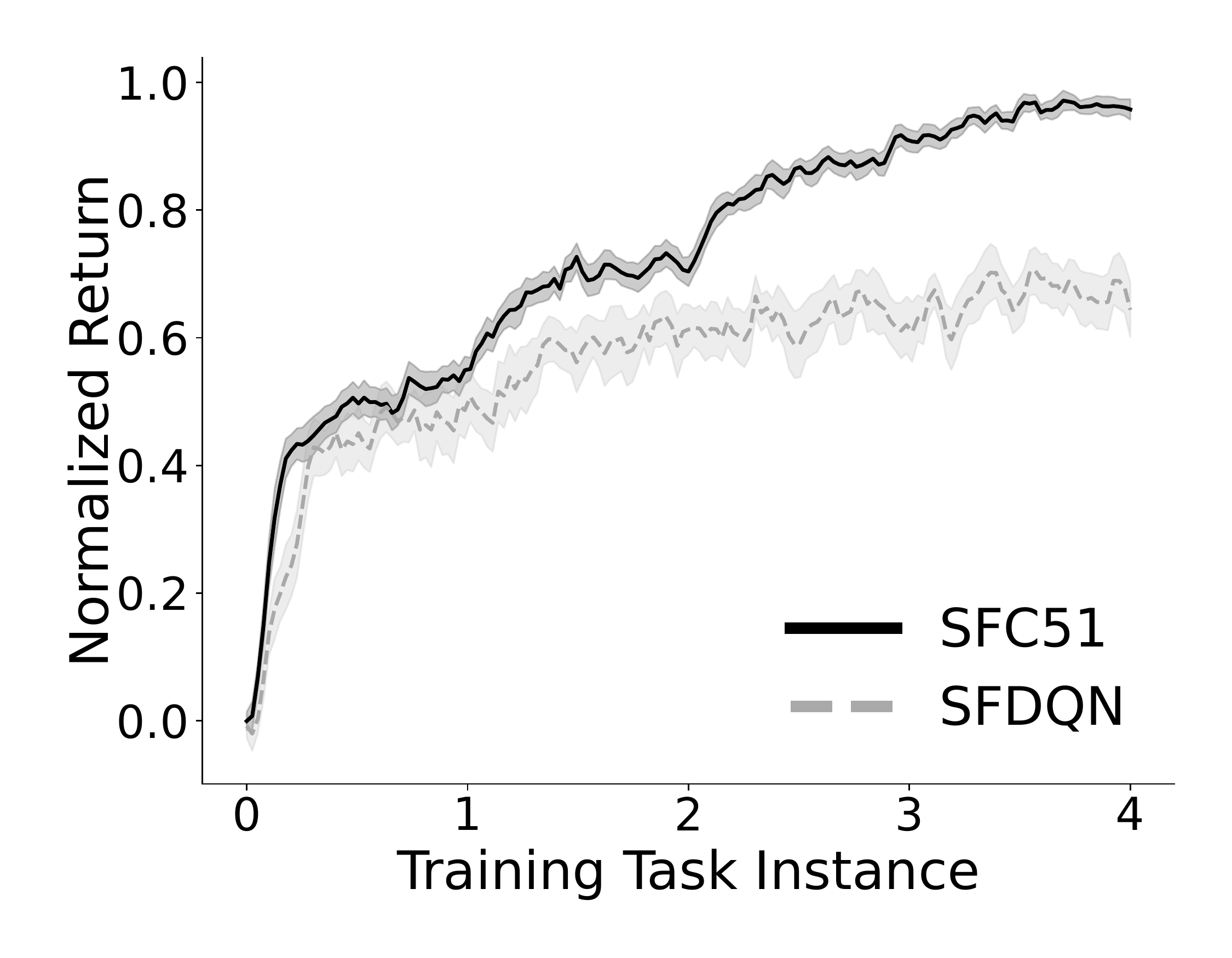}
    \end{subfigure}%
    \begin{subfigure}[c]{0.359\textwidth}
        \centering
        \includegraphics[width=0.33\linewidth]{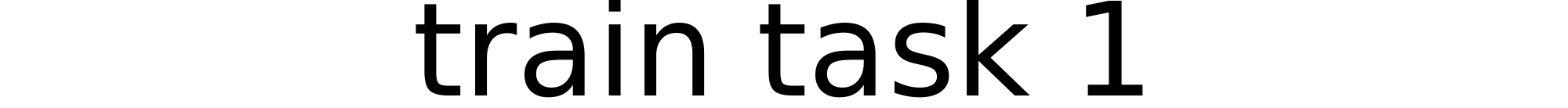}%
        \includegraphics[width=0.33\linewidth]{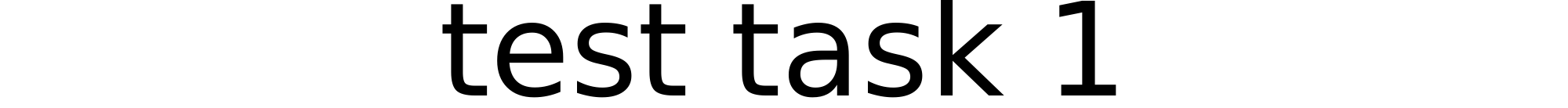}%
        \includegraphics[width=0.33\linewidth]{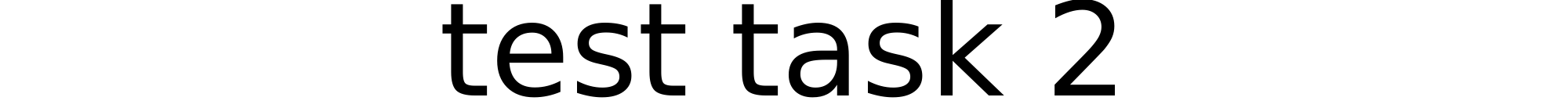}
        
        \includegraphics[width=0.33\linewidth]{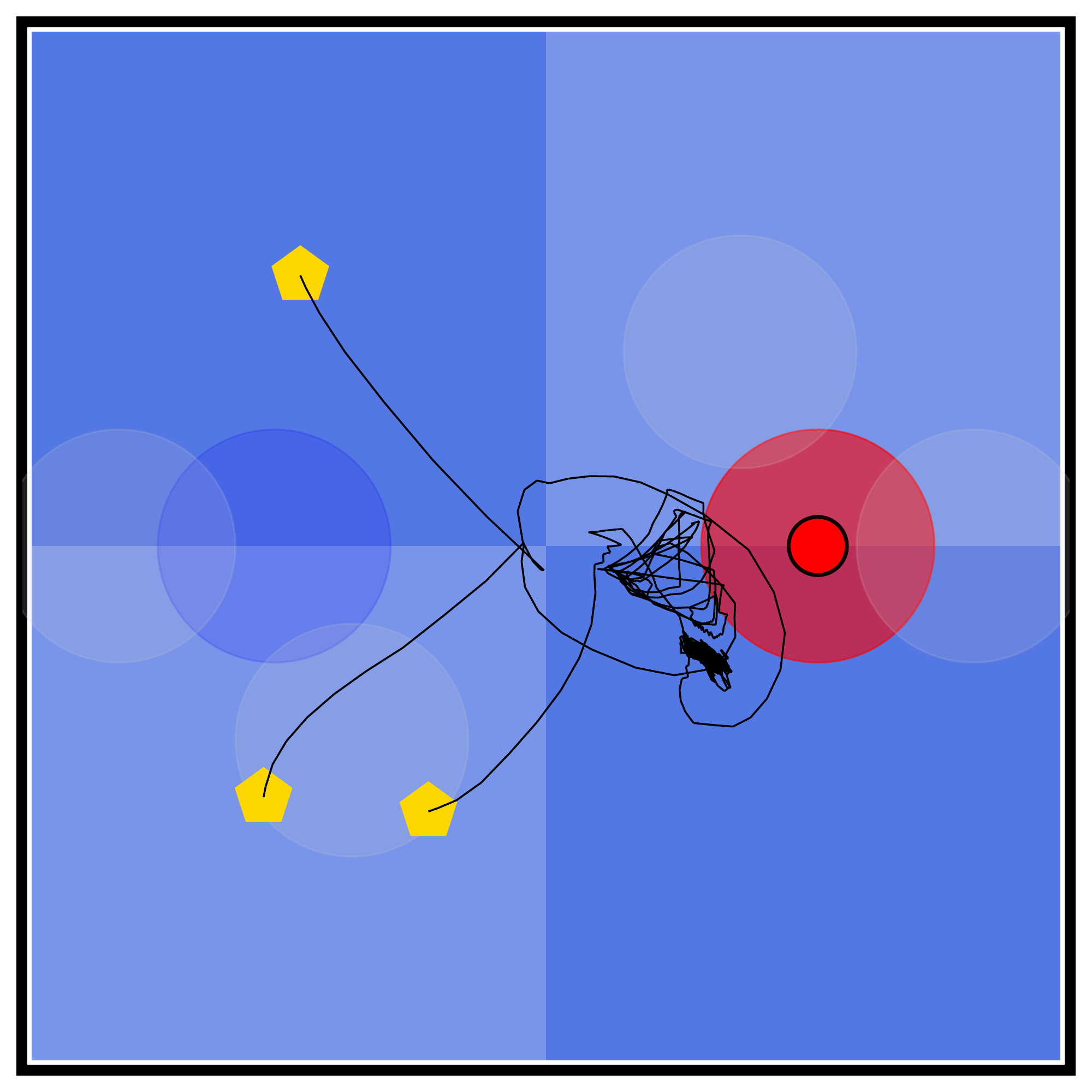}%
        \includegraphics[width=0.33\linewidth]{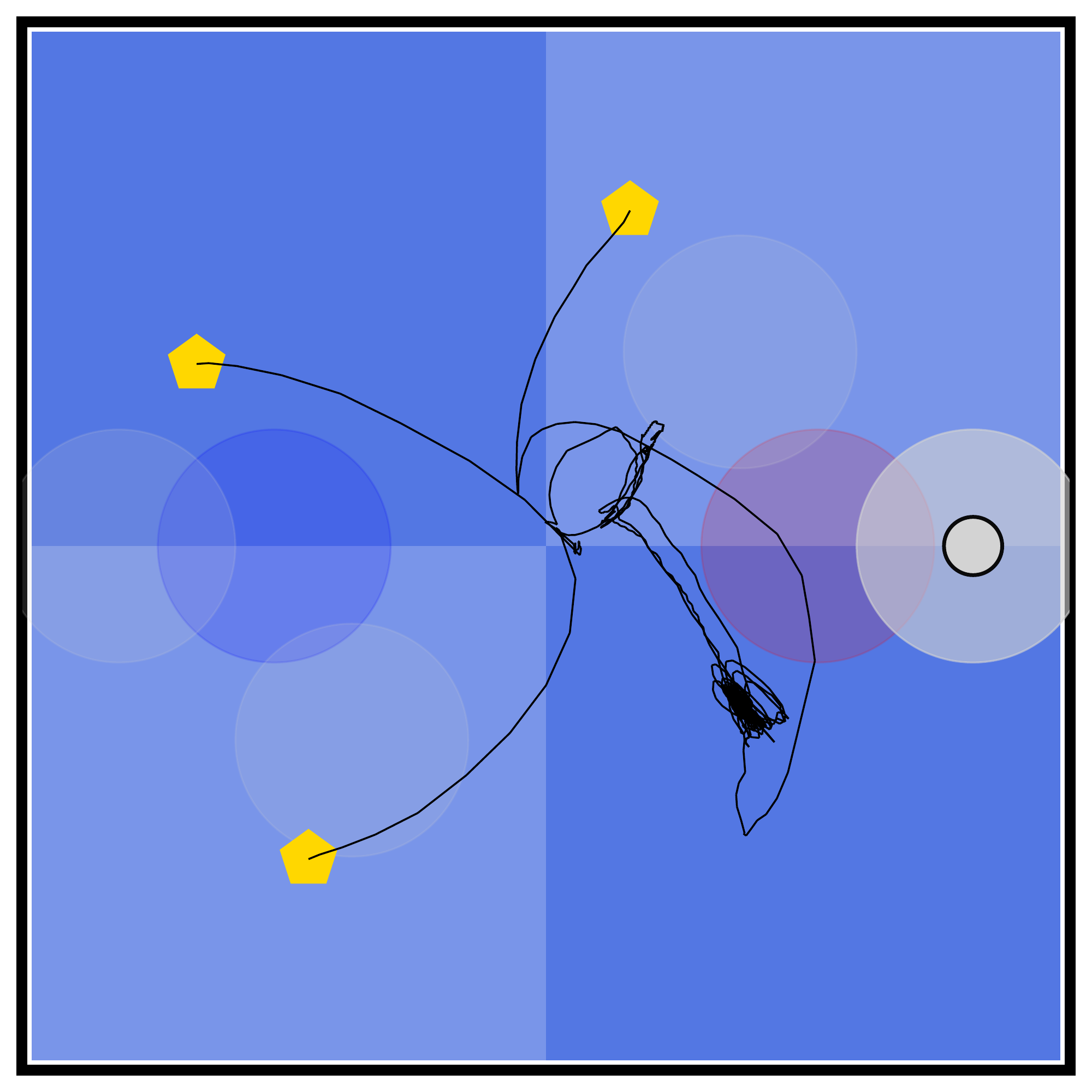}%
        \includegraphics[width=0.33\linewidth]{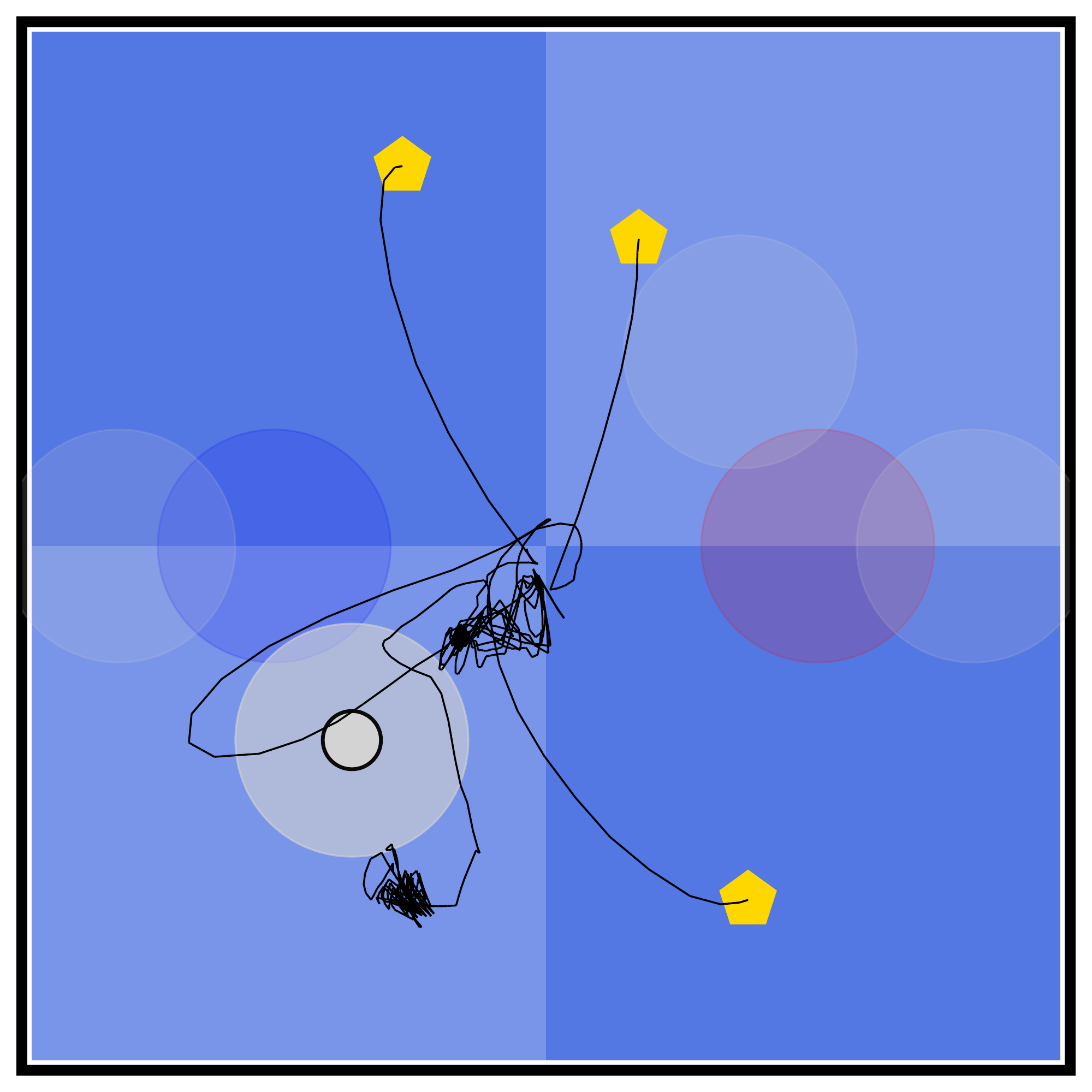}%
        \includegraphics[height=0.07\textheight]{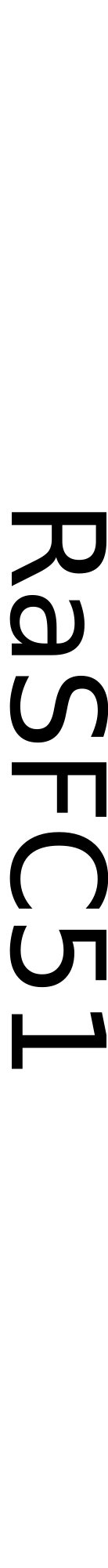}
        
        \includegraphics[width=0.33\linewidth]{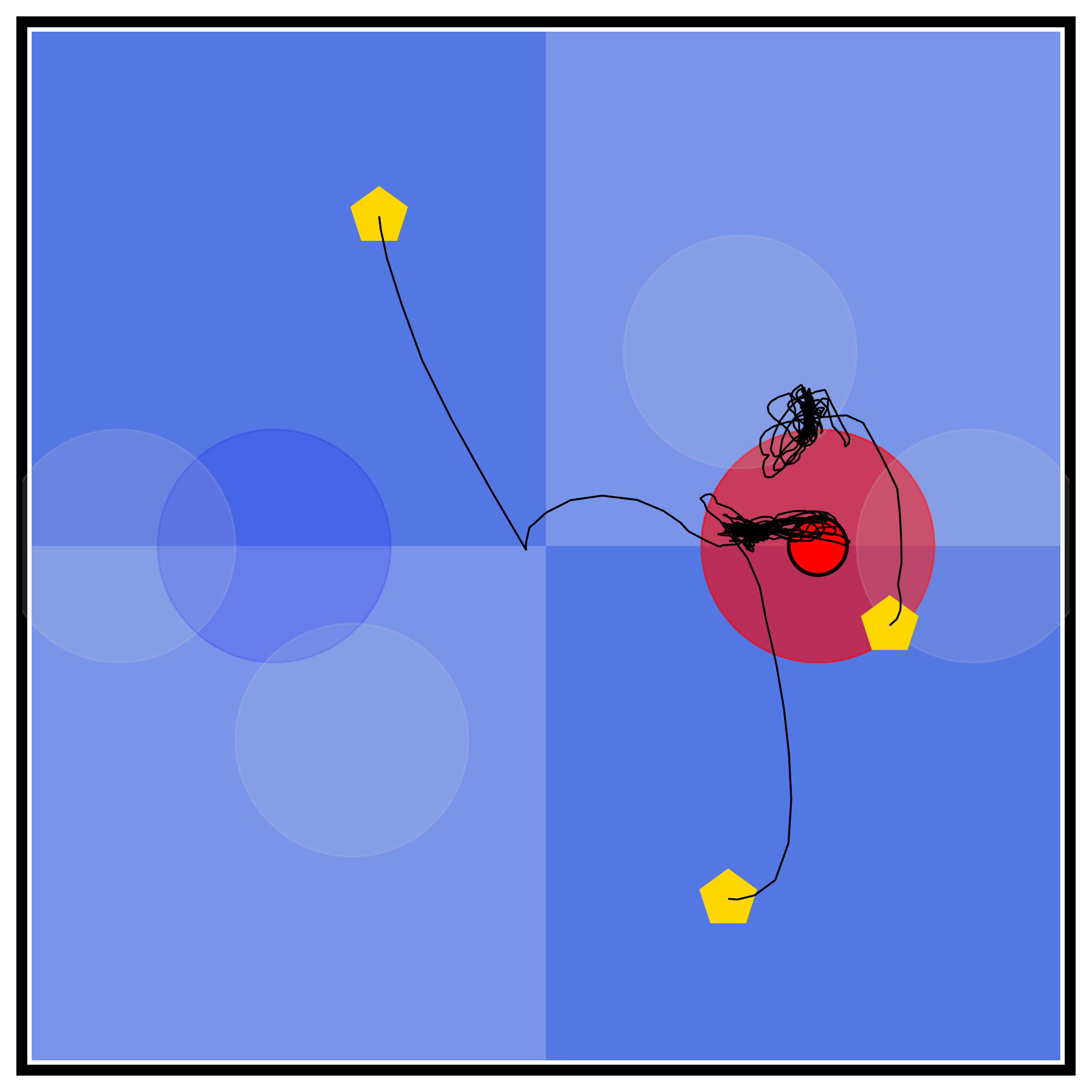}%
        \includegraphics[width=0.33\linewidth]{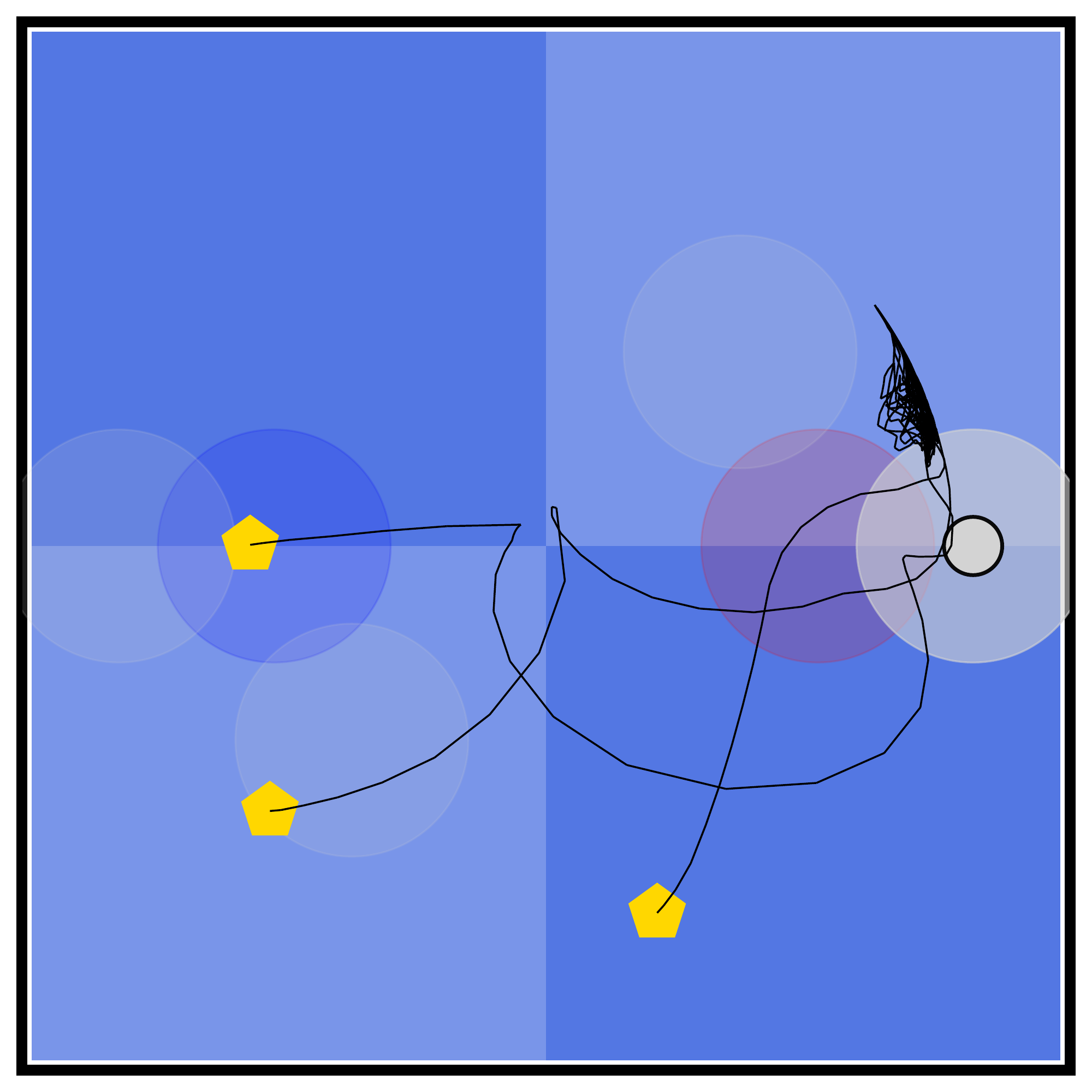}%
        \includegraphics[width=0.33\linewidth]{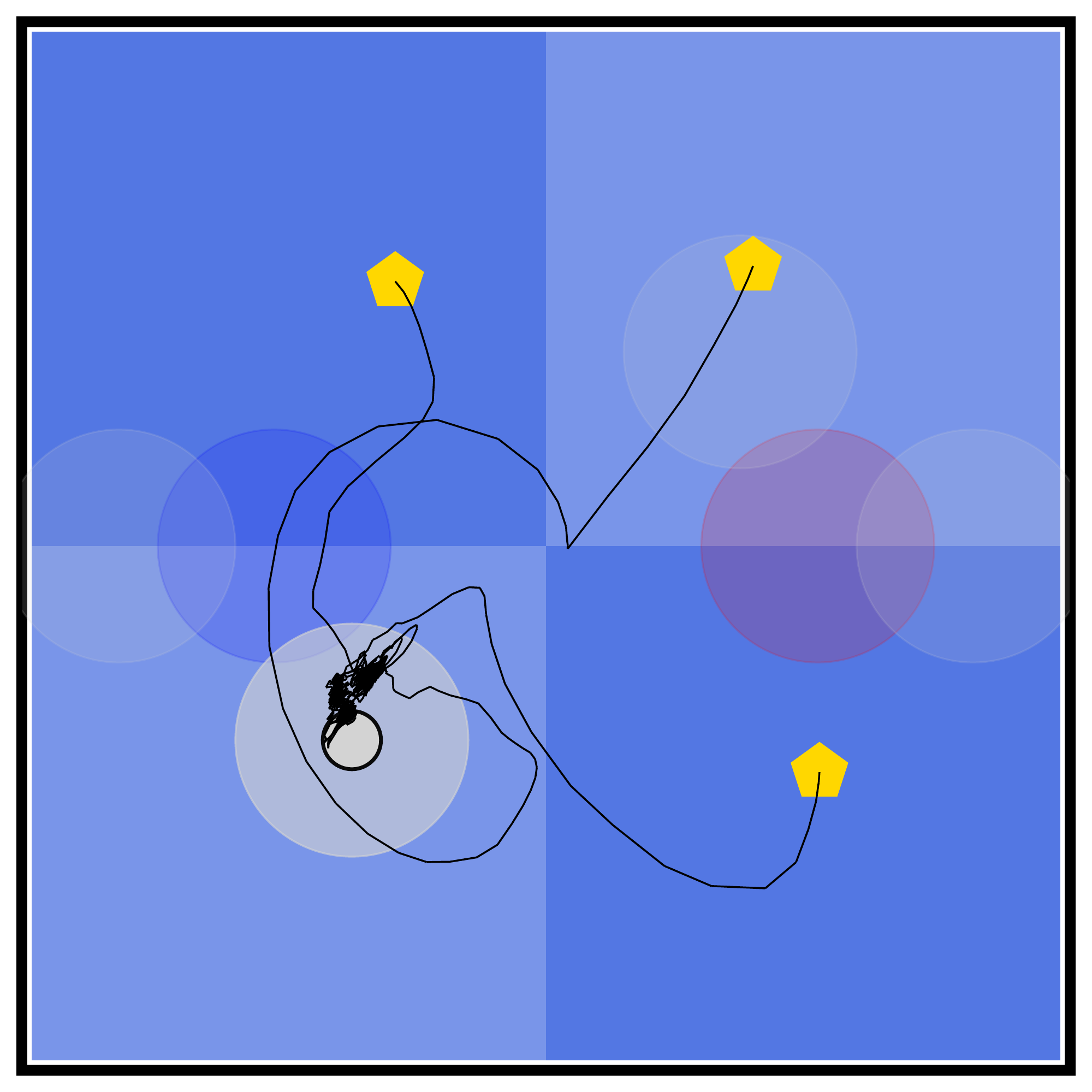}%
         \includegraphics[height=0.07\textheight]{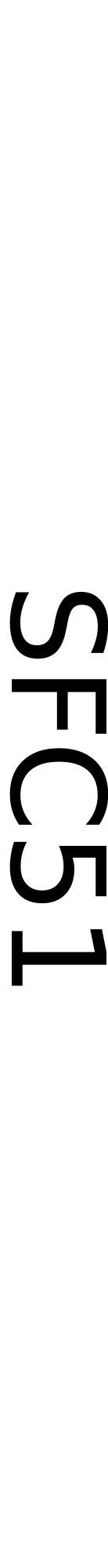}
    \end{subfigure}
    \caption{\textbf{Left:} Normalized average test return for the reacher domain, for different values of $\beta$ (legend values indicate the negative value of $\beta$). \textbf{Middle:} Normalized average test return for the reacher domain, showing the improvement obtained by replacing DQN by C51 as a function approximator for SFs. \textbf{Right:} Evolutions of the arm tip position in three rollouts of the reacher domain according to the GPI policy obtained after training on all 4 tasks (yellow pentagons indicate the initial states in each rollout). Only one training task and two test tasks are shown. The risk-averse agent learns to hover close to the goal while avoiding the high-volatility shaded regions.}
    \label{fig:reacher_ablation}
\end{figure}

The performance of these algorithms on the reacher domain is illustrated in Figure \ref{fig:reacher_main}. We first observe that the performance of all four training policies for RaSFC51 improves almost immediately, obtaining returns that exceed $75\%$ of the performance of a \emph{fully trained} C51 agent, an observation that correlates highly with the original SFDQN.  Interestingly, RaC51 is unable to achieve satisfactory performance on two out of the four training tasks, even after fully trained on all 4 tasks. Furthermore, the performance of RaSFC51 on the testing tasks far exceeds that of both RaC51 and C51, demonstrating the superior generalization ability of SFs in the risk-aware setting. Finally, the total number of failures across the test task instances is also considerably lower for RaSFC51 than it is for SFC51, and remains low during the entire training horizon. RaC51 and C51 fail frequently at the beginning of training, but less often later in training. While this may suggest that C51 is learning adequate risk-sensitive policies, it is mainly due to the fact that C51 is not able to generalize nearly as well as SFC51 for some of the test target locations, as elaborated in Appendix \appendixref{subsec:reacher_ablation}{D.2}.

We also conducted several ablation studies, summarized in Figure \ref{fig:reacher_ablation}. The performance of RaSFC51 and RaC51 decays gracefully as $\beta$ is increased in magnitude, but the performance of RaSFC51 uniformly outperforms RaC51 for \emph{every} value of $\beta$ tested. We also compared the benefit of replacing DQN by C51 as the architecture for learning SFs, and found that this simple modification can significantly improve SFs. It is likely that learning the full distribution of SF returns provides additional stability of the Bellman backups in stochastic domains, and thus allows SFs to inherit the advantages of distributional RL previously reported by \citet{bellemare2017distributional}. The final plot on the right shows that RaSFC51 learns safer policies than SFC51 on both training and test instances. Specifically, RaSFC51 learns to hover as close to the goal as possible in most cases, while still avoiding the high-volatility shaded regions. Further analysis of this domain can be found in Appendix \appendixref{subsec:reacher_ablation}{D.2}.

\section{Conclusion}

We presented Risk-aware Successor Features (RaSFs) for realizing policy transfer between tasks with shared dynamics and different goals, where the overall objective is to optimize the trade-off between expected return and risk, measured by the variance of return. We extended generalized policy improvement to the risk-aware setting, providing monotone guarantees and optimality of GPI, provided that source policies can be evaluated in a risk-aware manner. To facilitate policy evaluation, we extended the notion of generalized policy evaluation to the mean-variance objective. Together, risk-aware GPI and GPE inherit the superior task generalization abilities of successor features, while also learning to avoid dangerous and unpredictable situations in the environment, as analysis on discrete navigation and robotic control domains has showed. 

Our implementation relied heavily on modern deep reinforcement learning, thus requiring discounting, limitation to stationary policies, and the independence assumption in the SFC51 architecture for tractability. One way to mitigate some of these concerns is to pursue a model-based direction for better sample-efficiency and better policy search through planning. More generally, incorporating risk and safety in sequential decision-making is a complex problem. The entropic utility objective studied in this work does not capture other well-known utility functions such as CVaR, which would be difficult to incorporate into GPE due to their reliance on percentiles and lack of asymptotic expansions that are exploited in this work. Incorporating other notions of risk in the GPI/GPE framework could form interesting and challenging future extensions of our work.

\section*{Acknowledgements}

The authors would like to thank Daniel Mankowitz for suggesting relevant research in the area of robust and risk-aware reinforcement learning and for providing insightful comments throughout the development of the paper.

\small
\bibliographystyle{plainnat}
\bibliography{references}
\normalsize

\ifappendix

\clearpage
\appendix

\vspace{7mm} 
\begin{center}
\vspace{7mm} 
\noindent\makebox[\textwidth]{\rule{\textwidth}{4.0pt}} \\
\vspace{3mm}    
{\bf {\LARGE Risk-Aware Transfer in Reinforcement Learning \\ \vspace{1.5mm} using Successor Features} \\
\vspace{5mm} {\Large Supplementary Material} }
\vspace{5mm} 
\noindent\makebox[\textwidth]{\rule{\textwidth}{1.0pt}}
  {\bf Michael Gimelfarb, Andr\'e Barreto, Scott Sanner} and {\bf Chi-Guhn Lee}
  \vskip 0.2in
\end{center}

\begin{abstract}
This part of the paper discusses algorithmic details for the total reward episodic setting and the discounted setting that were not included in the main paper due to space limitations. It includes proofs of all main theoretical claims in the paper, as well as additional experimental details and parameter settings used to reproduce the experiments.
\end{abstract}

\newenvironment{narrow}[2]{%
  \begin{list}{}{%
    \setlength{\topsep}{0pt}%
    \setlength{\leftmargin}{#1}%
    \setlength{\rightmargin}{#2}%
    \setlength{\listparindent}{\parindent}%
    \setlength{\itemindent}{\parindent}%
    \setlength{\parsep}{\parskip}
  }%
  \item[]
}{\end{list}}

\renewcommand{\contentsname}{\centering Contents}
\addtocontents{toc}{\setcounter{tocdepth}{3}} 
\begin{narrow}{1.25cm}{1.25cm}
    \tableofcontents
 \end{narrow}

\renewcommand{\thesection}{\Alph{section}}

\section{Mathematical and Algorithmic Details}
\label{sec:algorithmic_detail}

In this section, we outline the ways in which successor features and their covariance matrices can be learned in practical RL settings. Specifically, we introduce Bellman updates and the distributional RL framework, and also discuss the mean-variance approximation under more general assumptions in parameteric density estimation.

\subsection{Mean-Variance Approximation for Episodic MDPs}
\label{subsec:algorithmic_total_reward}

\paragraph{Directly Computing Successor Features and Covariances.}
In the tabular setting, we compute estimates of the mean $\tilde{\bm{\psi}}_{h}^{\pi}(s,a)$ and covariance $\tilde{\Sigma}_h^\pi(s,a)$ by extending the analysis in \citet{sherstan2018comparing} to the $d$-dimensional setting. For a transition $(s, a, \bm{\phi}, s', a')$, where $\bm{\phi} = \bm{\phi}(s, a, s')$, and learning rate $\alpha > 0$, the update for successor features is:
\begin{equation}
\label{eqn:mvsf_mean}
    \begin{aligned}
        \tilde{\bm{\delta}}_h &= \bm{\phi} +  \tilde{\bm{\psi}}_{h+1}^\pi(s',a') - \tilde{\bm{\psi}}_h^\pi(s,a), \\
        \tilde{\bm{\psi}}_{h}^\pi(s,a) &= \tilde{\bm{\psi}}_{h}^\pi(s,a) + \alpha \tilde{\bm{\delta}}_h.
    \end{aligned}
\end{equation}
By virtue of the covariance Bellman equation (\ref{eqn:covariance_Bellman}), a per-sample update of the covariance matrix can be computed by using the Bellman residuals $\tilde{\bm{\delta}}_h$ as a pseudo-reward:
\begin{equation}
    \label{eqn:mvsf_covariance}
    \begin{aligned}
        \Delta_h &= \tilde{\bm{\delta}}_h \tr{\tilde{\bm{\delta}}_h} +  \tilde{\Sigma}_{h+1}^\pi(s',a') - \tilde{\Sigma}_h^\pi(s,a), \\
        \tilde{\Sigma}_{h}^\pi(s,a) &= \tilde{\Sigma}_h^\pi(s,a) + \bar{\alpha} \Delta_h.
    \end{aligned}
\end{equation}
In practice, $\bar{\alpha}$ is usually set much smaller than $\alpha$, e.g. $\bar{\alpha} = \rho \alpha$ for some positive $\rho \ll 1$.

\paragraph{Computing Reward Parameters.}
When the reward parameters $\mathbf{w}$ are unknown, they can be learned by solving a regression problem. Specifically, for known or estimated features $\bm{\phi}(s,a,s')$ and an observed reward $r$, the objective function to minimize is the \emph{mean-squared error},
\begin{equation*}
    \mathcal{L}(\mathbf{w}) = \frac{1}{2} \left(r - \tr{\bm{\phi}(s,a,s')} \mathbf{w} \right)^2,
\end{equation*}
that can be minimized using \emph{stochastic gradient descent} (SGD). Introducing a learning rate $\alpha_w > 0$, an update of SGD on a single transition $(s, a) \to s'$ is
\begin{equation*}
    \mathbf{w} \gets \mathbf{w} + \alpha_w (r - \tr{\bm{\phi}(s,a,s')} \mathbf{w}) \bm{\phi}(s,a,s').
\end{equation*}

\paragraph{Pseudocode.}
The general routine for performing risk-aware transfer learning in an online setting, which we call \emph{Risk-Aware Successor Feature Q-Learning} (RaSFQL), can now be fully described. Pseudocode adapted for the total-reward episodic MDP setting is given as Algorithm \ref{alg:mvsfql}. Please note that our approach closely follows the risk-neutral SFQL in \citet{barreto2017successor}.

\begin{algorithm}[!tb]
  \caption{RaSFQL with Mean-Variance Approximation}
  \label{alg:mvsfql}
\begin{algorithmic}[1]
    \State {\bfseries Requires} $m, T, N_e \in \mathbb{N}, \,  \varepsilon \in [0, 1], \, \alpha, \bar{\alpha}, \alpha_w > 0, \, \beta \in \mathbb{R}, \, \bm{\phi} \in \mathbb{R}^d$, $M_1, \dots M_m \in \mathcal{M}$
  \For{$t=1, 2\dots m$}
        \IndentLineComment{Initialize successor features and covariance for the current task}
        \State {\bf{if}} $t = 1$ {\bf{then}} Initialize $\tilde{\bm{\psi}}^t, \tilde{{\Sigma}}^t$ to small random values {\bf{else}} Initialize $\tilde{\bm{\psi}}^t, \tilde{{\Sigma}}^t$ to $\tilde{\bm{\psi}}^{t-1}, \tilde{{\Sigma}}^{t-1}$
        \State Initialize $\tilde{\mathbf{w}}_t$ to small random values
        \LineComment{Commence training on task $M_t$}
        \For{$n_e = 1, 2 \dots N_e$}
            \State Initialize task $M_t$ with initial state $s$
        \For{$h=0,1\dots T$} 
            \IndentLineComment{Select the source task $M_c$ using GPI}
            \State $c \gets \argmax_j \max_b \lbrace \tr{\tilde{\bm{\psi}}_h^j(s,b)} \tilde{\mathbf{w}}_t - \beta {\tr{\tilde{\mathbf{w}}_t} \tilde{{\Sigma}}_h^j(s,b) \tilde{\mathbf{w}}_t} \rbrace$
            \LineComment{Sample action from the epsilon-greedy policy based on $\pi_c$}
            \State $\mathrm{random\_a} \sim \mathrm{Bernoulli}(\varepsilon)$
            \State \parbox[t]{\dimexpr\textwidth-\leftmargin-\labelsep-\labelwidth}{%
            \bf{if} $\mathrm{random\_a}$ \bf{then} $a \sim \mathrm{Uniform}(\mathcal{A})$ \bf{else} \\ $a \gets \argmax_b \lbrace \tr{\tilde{\bm{\psi}}_h^c(s,b)} \tilde{\mathbf{w}}_t - \beta {\tr{\tilde{\mathbf{w}}_t} \tilde{{\Sigma}}_h^c(s,b) \tilde{\mathbf{w}}_t} \rbrace$ \strut}
            \State Take action $a$ in $M_t$ and observe $r$ and $s'$
            \LineComment{Update reward parameters for the current task}
            \State $\tilde{\mathbf{w}}_t \gets \tilde{\mathbf{w}}_t + \alpha_w (r - \tr{\bm{\phi}(s,a,s')} \tilde{\mathbf{w}}_t) \bm{\phi}(s,a,s')$
            \LineComment{Update the successor features and covariance for the current task}
            \State $a' \gets \argmax_b \max_j \lbrace \tr{\tilde{\bm{\psi}}_h^j(s',b)} \tilde{\mathbf{w}}_t - \beta {\tr{\tilde{\mathbf{w}}_t} \tilde{\Sigma}_h^j(s',b) \tilde{\mathbf{w}}_t} \rbrace$
            \State Update $\tilde{\bm{\psi}}_{h}^t,  \tilde{\Sigma}_{h}^t$ on $(s,a,\bm{\phi}, s',a')$ using (\ref{eqn:mvsf_mean}) and (\ref{eqn:mvsf_covariance})
            \LineComment{Update the successor features and covariance for task $M_c$}
            \If{$c \not= t$}
                \State $a' \gets \argmax_b \lbrace \tr{\tilde{\bm{\psi}}_h^c(s',b)} \tilde{\mathbf{w}}_c - \beta {\tr{\tilde{\mathbf{w}}_c} \tilde{\Sigma}_h^c(s',b) \tilde{\mathbf{w}}_c} \rbrace$
                \State Update $\tilde{\bm{\psi}}_{h}^c, \tilde{\Sigma}_{h}^c$ on $(s,a,\bm{\phi}, s',a')$ using (\ref{eqn:mvsf_mean}) and (\ref{eqn:mvsf_covariance})
            \EndIf
            \State $s \gets s'$
        \EndFor
        \EndFor
  \EndFor
\end{algorithmic}
\end{algorithm}

Both the discounted and total reward episodic settings are amenable to function approximation. However, as discussed in the main text, this ``residual" method is usually not advisable as the approximation errors in the residuals $\bm{\tilde{\delta}}_h$ can dominate the environment uncertainty. While this could be useful for handling \emph{epistemic} or model uncertainty in successor features \citep{janz2018successor}, the intent of our work is to learn the \emph{aleatory} or environment uncertainty. Therefore, a more precise method --- based on the projected Bellman equation --- will be introduced in Appendix \appendixref{subsec:distributional_sf}{A.3}.

\subsection{Mean-Variance Approximation for Discounted MDPs}
\label{subsec:discounted_setting}

\paragraph{Bellman Principle and Augmented MDP.}
The utility objective in the discounted infinite-horizon setting becomes
\begin{equation*}
    \mathcal{Q}_{\beta,\gamma}^\pi(s,a) = U_\beta\left[\sum_{t=0}^\infty \gamma^{t} r(s_t, a_t, s_{t+1})\right],
\end{equation*}
where $\gamma \in (0, 1)$ is a discount factor for future rewards.

In the discounted setting, it is necessary to accumulate and keep track of the discounting over time. This can be implemented by augmenting the state space of the original MDP \citep{bauerle2014more}. Specifically, define $\mathcal{Z} = [0, 1]$ and let $z \in \mathcal{Z}$ denote the state of discounting. For a given MDP $\langle \mathcal{S}, \mathcal{A}, r, P \rangle$, we define the augmented MDP $\langle \mathcal{S}', \mathcal{A}, r', P' \rangle$, with state space $\mathcal{S}' = \mathcal{S} \times \mathcal{Z}$, action space $\mathcal{A}$, reward function 
\begin{equation*}
    r'((s,z),a,(s',z')) = z r(s,a,s'),    
\end{equation*}
and dynamics 
\begin{equation*}
    P'((s',z') | (s, z), a) = P(s' | s, a) \delta_{\gamma z}(z'),
\end{equation*}
where $\delta$ is the Dirac delta function. Applying this augmentation transformation to a set of MDPs with common transition function and common discount factor implies that the set of augmented MDPs will also have the same transition functions. 

Moreover, the following Bellman equation can be derived for the augmented MDP \citep{bauerle2014more}:
\begin{equation}
\label{eqn:entropic_bellman_discounted}
    \begin{aligned}
        \mathcal{J}_{\beta}^\pi(s,a,z) 
        &= U_\beta\left[z r(s,a,s') + \mathcal{J}_\beta^\pi(s',\pi(s',\gamma z), \gamma z) \right] \\
        &= \frac{1}{\beta}\log \mathbb{E}_{s'\sim P(\cdot | s, a)}\left[\exp{\left\lbrace \beta\left( z r(s,a,s') + \mathcal{J}_{\beta}^\pi(s',\pi(s',\gamma z),\gamma z) \right)\right\rbrace} \right].
    \end{aligned}
\end{equation}
Then, we can recover the original utility with $\mathcal{Q}_{\beta,\gamma}^\pi(s,a) = \mathcal{J}_{\beta}^\pi(s,a,1)$. Furthermore, the Bellman equation above converges to a unique fixed point, and so the search for optimal policies can be restricted to stationary Markov policies $\pi : \mathcal{S} \times \mathcal{Z} \to \mathcal{A}$.

However, learning general policies $\pi : \mathcal{S} \times \mathcal{Z} \to \mathcal{A}$ introduces additional difficulties in the function approximation setting. In this case, successor features and their covariance matrices would have to be functions of $z$. For a single transition, their corresponding updates would also require a sweep over all possible values of $z$, e.g. $z = 1, \gamma, \gamma^2, \dots$, and would be computationally demanding. On the other hand, restricting the search to stationary policies $\pi : \mathcal{S} \to \mathcal{A}$ alleviates this computational burden, making the overall time and space complexity per update comparable to the risk-neutral SF representation, and also allows off-the-shelf RL algorithms to be used to learn successor features. This also facilitates more precise estimation of risk using the distributional framework discussed in the next section. 

Fortunately, the restriction to $z$-independent source policies does not affect the validity of Theorem \ref{thm:gpi_aux}, since policy improvement was shown for \emph{arbitrary} admissible policies. This implies that monotone policy improvement is guaranteed even for $z$-independent policies, provided that their utilities can be estimated. In the case of Theorem \ref{thm:gpi}, the approximation error $\varepsilon$ generally arises from two sources of additive error, namely that of restricting optimal policies $\pi_i^*$ to $z$-independent optimal policies $\bar{\pi}_i^*$, and that of approximating utilities using function approximation, e.g.
\begin{align*}
    \varepsilon 
    &= \left|\tilde{\mathcal{J}}_\beta^{\bar{\pi}_i^*}(s,a,1) - {\mathcal{J}}_\beta^{\pi_i^*}(s,a,1) \right| \\
    &= \left|\tilde{\mathcal{J}}_\beta^{\bar{\pi}_i^*}(s,a,1) - \mathcal{J}_\beta^{\bar{\pi}_i^*}(s,a,1) + \mathcal{J}_\beta^{\bar{\pi}_i^*}(s,a,1) - {\mathcal{J}}_\beta^{\pi_i^*}(s,a,1) \right| \\
    &\leq \left|\tilde{\mathcal{J}}_\beta^{\bar{\pi}_i^*}(s,a,1) - \mathcal{J}_\beta^{\bar{\pi}_i^*}(s,a,1) \right| + \left| \mathcal{J}_\beta^{\bar{\pi}_i^*}(s,a,1) - {\mathcal{J}}_\beta^{\pi_i^*}(s,a,1) \right| \\
    &= \left\lbrace\textrm{approximation error of } \mathcal{J}_\beta^{\bar{\pi}_i^*}\right\rbrace + \left\lbrace \textrm{absolute difference between utilities of } \bar{\pi}_i^* \textrm{ and } \pi_i^* \right\rbrace.
\end{align*}
The first source of error arises solely due to the method of function approximation, and can be reduced by using architectures whose training parameters and capacity are well-calibrated for each problem. The second source of error is in general irreducible, but whether it can be tolerated should be traded-off against the difficulty of learning $z$-dependent policies. In general, the learning of $z$-dependent policies tractably is a challenging problem, which we leave for future investigation.

\paragraph{Incorporating Moment Information into GPE in Discounted MDPs.}
We now apply the idea of generalized policy evaluation to discounted objectives. First, observe that, for fixed $\pi : \mathcal{S} \to \mathcal{A}$:
\begin{equation}
\label{eqn:utility_linear_discounted}
    \mathcal{J}_{\beta}^\pi(s,a,1) = U_\beta\left[\sum_{t=0}^\infty \gamma^{t} r(s_t, \pi(s_t), s_{t+1})\right] = U_\beta\left[ \tr{\Psi^\pi(s,a)}\mathbf{w} \right],
\end{equation}
corresponding to the random vector $\Psi^\pi(s,a) = \sum_{t=0}^\infty \gamma^{t} \bm{\phi}_t$ of unrealized feature returns. Thus, we have again transformed the problem of estimating the utility of rewards into the problem of estimating the moments of the random variable $\tr{\Psi^\pi(s,a)}\mathbf{w}$. 

Next, computing the Taylor expansion of $U_\beta$:
\begin{align}
\label{eqn:mvsf_discounted}
    \mathcal{J}_{\beta}^\pi(s,a,1) 
    &= \mathbb{E}_P[ \tr{\Psi^\pi(s,a)}\mathbf{w} ] + \frac{\beta}{2}\mathrm{Var}_P[ \tr{\Psi^\pi(s,a)}\mathbf{w}] + O(\beta^2) \nonumber \\
    &\approx  \tr{\bm{\psi}^\pi(s,a)}\mathbf{w} + \frac{\beta}{2} \tr{\mathbf{w}}\mathrm{Var}_P[\Psi^\pi(s,a)]\mathbf{w}= \mathcal{\tilde{J}}_{\beta}^\pi(s,a,1).
\end{align}
From a practical point of view, the mean-variance approximation in the discounted setting is identical to the episodic total-reward setting, with the exception that the successor features and covariance are discounted (and also time-independent). As in the undiscounted case, (\ref{eqn:mvsf_discounted}) induces an error of $O(\beta^2)$, but is now another instantiation of GPE. However, restricting the search to $z$-independent policies introduces additional approximation error that can also be absorbed into $\varepsilon$, as discussed previously. Crucially, the theoretical results proved for the discounted setting (Appendix \appendixref{subsec:proofs_discounted}{B.3}) will now also hold for $z$-independent stationary policies.

\paragraph{Bellman Updates for Covariance in Discounted MDPs.}
The covariance matrix satisfies the covariance Bellman equation
\begin{equation}
    \label{eqn:covariance_Bellman_discounted}
    \Sigma_h^\pi(s,a) = \mathbb{E}_{s' \sim P(\cdot | s, a)}\left[\bm{\delta}_h \tr{\bm{\delta}_h} +  \gamma^2 \Sigma_{h+1}^\pi(s', \pi_{h+1}(s')) \,|\, s_h = s,\, a_h = a \right],
\end{equation}
Similar to (\ref{eqn:mvsf_mean}) and (\ref{eqn:mvsf_covariance}), in the discounted setting the successor features can be computed as \citep{barreto2017successor}:
\begin{equation}
\label{eqn:mvsf_mean_discounted}
    \begin{aligned}
        \tilde{\bm{\delta}}_h &= \bm{\phi} +  \gamma \tilde{\bm{\psi}}_{h+1}^\pi(s',a') - \tilde{\bm{\psi}}_h^\pi(s,a), \\
        \tilde{\bm{\psi}}_{h}^\pi(s,a) &= \tilde{\bm{\psi}}_{h}^\pi(s,a) + \alpha \tilde{\bm{\delta}}_h.
    \end{aligned}
\end{equation}
Once again, the covariance matrix can be updated per sample following (\ref{eqn:covariance_Bellman_discounted}):
\begin{equation}
    \label{eqn:mvsf_covariance_discounted}
    \begin{aligned}
        \Delta_h &= \tilde{\bm{\delta}}_h \tr{\tilde{\bm{\delta}}_h} +  \gamma^2 \tilde{\Sigma}_{h+1}^\pi(s',a') - \tilde{\Sigma}_h^\pi(s,a), \\
        \tilde{\Sigma}_{h}^\pi(s,a) &= \tilde{\Sigma}_h^\pi(s,a) + \bar{\alpha} \Delta_h.
    \end{aligned}
\end{equation}
In the context of Algorithm \ref{alg:mvsfql}, all calls to (\ref{eqn:mvsf_mean}) and (\ref{eqn:mvsf_covariance}) would be replaced with (\ref{eqn:mvsf_mean_discounted}) and (\ref{eqn:mvsf_covariance_discounted}), respectively.

The convergence of the covariance matrix in the discounted setting (\ref{eqn:covariance_Bellman_discounted}) is established in the following result that can be easily proved using the techniques in Appendix \appendixref{subsec:proofs_covariance}{B.4} for the episodic setting.

\begin{theorem}[\bfseries Convergence of Covariance]
\label{thm:sigma_convergence_discounted}
    Let $\|\cdot\|$ be a matrix-compatible norm, and suppose there exists $\varepsilon : \mathcal{S} \times \mathcal{A} \times \mathcal{T} \to [0, \infty)$ such that $\|\tilde{\bm{\psi}}_h^\pi(s,a) - {\bm{\psi}}_h^\pi(s,a) \|^2 \leq \varepsilon_h(s,a)$ and $\|\mathbb{E}_{s'\sim P(\cdot | s, a)}[\gamma \tilde{\bm{\delta}}_h \tr{(\tilde{\bm{\psi}}_h^\pi(s',\pi_{h+1}(s')) - {\bm{\psi}}_h^\pi(s',\pi_{h+1}(s')))}]\| \leq \varepsilon_h(s,a)$. Then,
    \begin{equation*}
        \left\| {\Sigma}_h^\pi(s,a) - \mathbb{E}_{s' \sim P(\cdot |s,a)}\left[\tilde{\bm{\delta}}_h \tr{\tilde{\bm{\delta}}_h} + \gamma^2 \tilde{\Sigma}_{h+1}^\pi(s',\pi_{h+1}(s')) \right] \right\| \leq 3 \varepsilon_h(s,a).
    \end{equation*}
\end{theorem}
Please note that this result is identical to Theorem \ref{thm:sigma_convergence}, with the exception of the discount factor. 

\subsection{Histogram Representations for Successor Features}
\label{subsec:distributional_sf}

The theoretical framework for distributional RL is discussed in details in the relevant literature \citep{bellemare2017distributional}. In this appendix, we discuss how this framework can be applied to learn distributions over successor features, and how to use these distributions to select actions in a risk-aware manner.

\paragraph{Learning Distributions over Successor Features.}
As discussed in the main text, the goal is to estimate the distribution of each component in the discounted infinite horizon setting
\begin{equation*}
    \Psi_i^\pi(s,a) = \sum_{t=0}^\infty \gamma^t \phi_i(s_t, a_t, s_{t+1}), 
\end{equation*}
starting from $s_0 = s,\, a_0 = a$, where $a_t = \pi(s_t)$ is selected according to a policy $\pi$. Treating $\Psi_1^\pi, \dots \Psi_d^\pi$ as value functions, we are now able to apply distributional RL.

Specifically, suppose that each state feature component is bounded in a compact interval, e.g. $\phi_i(s,a,s') \in [\phi_i^{min}, \phi_i^{max}]$. Then, we may define corresponding bounds on $\Psi_i^\pi(s,a)$ by bounding the terms of its geometric series representation above:
\begin{equation*}
   \Psi_i^{min} = \frac{\phi_i^{min}}{1-\gamma} \leq \Psi_i^\pi(s,a) \leq \frac{\phi_i^{max}}{1-\gamma} = \Psi_i^{max}.
\end{equation*}
Now, we may model each $\Psi_h^\pi(s,a)$ by using a discrete distribution parameterized by $N \in \mathbb{N}$ and $[\Psi_i^{min}, \Psi_i^{max}]$, whose support is defined by a set of atoms 
\begin{equation*}
    \mathcal{Z}_i = \left\lbrace z_{j,i} = \Psi_i^{min} + j \Delta z_i : 0 \leq j < N \right\rbrace, \, \Delta z_i = \frac{\Psi_i^{max} - \Psi_i^{min}}{N -1},\, \forall i = 1, \dots d.
\end{equation*}
Finally, the atom probabilities for $z_{j,i}$ are given by a parameteric model $\theta_{j,i} : \mathcal{S} \times \mathcal{A} \to \mathbb{R}^N$, e.g.
\begin{equation}
\label{eqn:c51_sf_distribution}
    Z_{\theta,i}(s,a) = z_{j,i} \quad \mathrm{w.p.} \,\, p_{j,i}(s,a) = \frac{e^{\theta_{j,i}(s,a)}}{\sum_j e^{\theta_{j,i}(s,a)}},
\end{equation}
where the softmax layer ensures that probabilities are non-negative and sum to one.

In order to update $p_{j,i}$ on environment transitions $(s,a,\phi_i,s')$, we project the Bellman updates for each $i$ onto the support of $\mathcal{Z}_i$. To do this, given a sample $(s,a,\phi_i,s')$, we compute the projected Bellman update, clipped to the interval $[\Psi_i^{min}, \Psi_i^{max}]$
\begin{equation*}
    \hat{\mathcal{T}}_i z_{j,i} = \mathrm{clip}\left(\phi_i + \gamma z_{j,i}; [\Psi_i^{min}, \Psi_i^{max}] \right),
\end{equation*}
and then distribute its probability $p_{j,i}(s',\pi(s'))$ to the immediate neighbors of $\hat{\mathcal{T}}_i z_{j,i}$. Here, we again follow \citet{bellemare2017distributional} and define the projected operator $\Phi$ with $j$-th component equal to
\begin{equation*}
    (\Phi \hat{\mathcal{T}}_i Z_{\theta,i}(s,a))_j = \sum_{k=0}^{N-1} \mathrm{clip}\left(1 - \frac{\left|\hat{\mathcal{T}}_i z_{k,i} - z_{j,i} \right|}{\Delta z_i}; [0,1] \right) p_{k,i}(s',\pi(s')).
\end{equation*}
As standard in deep RL, we view the target distribution $p_{k,i}(s',\pi(s'))$ as parameterized by a set of frozen parameters $\theta'$. Then, the loss function to optimize for the sample $(s,a,\phi_i,s')$ is given as the cross-entropy term
\begin{equation*}
    \mathcal{L}_i(\theta) = D_{KL}\left(\Phi \hat{\mathcal{T}}_i Z_{\theta',i}(s,a) \,\middle\|\, Z_{\theta,i}(s,a) \right),
\end{equation*}
that can be easily optimized using gradient descent.

\paragraph{Calculating Utilities.}
The calculation of (\ref{eqn:mvsf}) is a trivial matter given the distribution (\ref{eqn:c51_sf_distribution}). In particular, we have:
\begin{equation}
\label{eqn:c51_compute_moments}
\begin{aligned}
    \mathbb{E}[Z_{\theta,i}(s,a)^p] &= \sum_{j=0}^{N-1} {(z_{j,i})}^p p_{j,i}(s,a), \quad p \in \mathbb{N},
\end{aligned}
\end{equation}
from which we can easily compute the variance
\begin{equation}
\label{eqn:c51_compute_var}
    \mathrm{Var}[\Psi_i^\pi(s,a)] = \mathrm{Var}[Z_{\theta,i}(s,a)] = \mathrm{E}[Z_{\theta,i}(s,a)^2] - \mathbb{E}[Z_{\theta,i}(s,a)]^2.
\end{equation}
Recall that by the independence assumption, the cross-covariance terms are ignored in these calculations, and thus $\Sigma_i^\pi(s,a)$ is represented as a diagonal matrix with entries on the $i$-th diagonal term equal to $\mathrm{Var}[\Psi_i^\pi(s,a)]$. 

Another possibility is to compute the entropic utility $U_\beta$ exactly. In particular, using the independence assumption of $\Psi_i^\pi(s,a)$ again, we have:
\begin{align}
    U_\beta[\tr{\Psi^\pi(s,a)} \mathbf{w}] 
    &= \frac{1}{\beta} \log \mathbb{E}\left[e^{\beta \tr{\Psi^\pi(s,a)} \mathbf{w}}\right]
    \approx \sum_{i=1}^d \frac{1}{\beta} \log \mathbb{E}\left[e^{\beta \Psi_i^\pi(s,a) w_i}\right] \nonumber \\
    \label{eqn:c51_utility}
    &= \sum_{i=1,\, w_i \not= 0}^d w_i \frac{1}{\beta w_i} \log \mathbb{E}\left[e^{(\beta w_i) \Psi_i^\pi(s,a) }\right] = \sum_{i=1}^d w_i U_{\beta w_i}[\Psi_i^\pi(s,a)],
\end{align}
and can be seen as another risk-sensitive instantiation of GPE. Crucially, the utility terms in (\ref{eqn:c51_utility}) can be calculated efficiently in the C51 framework using (\ref{eqn:c51_sf_distribution})
\begin{equation*}
    U_{\beta}[\Psi_i^\pi(s,a)] = \frac{1}{\beta}\log  \mathbb{E}\left[e^{\beta Z_{\theta,i}(s,a)}\right] 
    = \frac{1}{\beta} \log \sum_{j=0}^{N-1} e^{\beta z_{j,i}} p_{j,i}(s,a).
\end{equation*}
However, this quantity is difficult to compute numerically, since for negative $\beta$, the terms $e^{\beta z_{j,i}}$ often suffer from overflow at $z_{j,i}$ close to $\Psi_i^{min}$, and underflow for $z_{j,i}$ close to $\Psi_i^{max}$. This becomes considerably more problematic for $\beta$ of larger magnitude, such as when risk-awareness is a priority, or for rewards $\mathbf{w}$ of larger magnitude. We also find that the log-sum-exp trick, a standard computational device used for calculations of this form, offers relatively little improvement. A similar issue has also been previously pointed out in other work using the entropic utility \citepsupp{gosavi2014beyond}. For this reason, we use the mean-variance approximation, which provides an excellent approximation to the entropic utility for various values of $\beta$, as we demonstrated experimentally, and without suffering from the aforementioned issues above.

\paragraph{Pseudocode.}
The approach described above can be applied to compute the distribution of successor features for every component $i = 1, \dots d$ across all training task instances. This results in a new algorithm that we call SFC51. Generally, the training procedure of SFC51 is identical in structure to SFDQN in \citet{barreto2017successor}, except the deterministic DQN update of successor features \citepsupp{mnih2015human} is replaced by the distributional C51 update described above. Therefore, the overall training procedure is similar to Algorithm \ref{alg:mvsfql}, but with a few subtle differences. First, instead of learning $\mathbf{w}$, it is provided to the agent as done in SFDQN. Second, every sample $(s,a,\bm{\phi},s')$ collected from any training is used to update all successor feature distributions simultaneously, as also done in SFDQN. Finally, the utility of returns can be used to select actions, rather than the expected return as done in DQN. Applying this last modification to SFC51 leads our proposed algorithm, which we call \emph{Risk-aware SFC51} (RaSFC51). Of course, SFC51 can be recovered by simply setting $\beta = 0$. A complete description of RaSFC51 with the mean-variance approximation is provided in Algorithm \ref{alg:sfc51}\footnote{In practice, the double for loop starting in lines 18 and 19 can be implemented efficiently by vectoring the computation of $m_{j,i}$, in languages that support vectorized arithmetic operations}.

\begin{algorithm}[!tb]
  \caption{RaSFC51 with Mean-Variance Approximation}
  \label{alg:sfc51}
\begin{algorithmic}[1]
    \State {\bfseries Requires} $m, T, N, N_e \in \mathbb{N}, \,  \varepsilon \in [0, 1], \, \beta, \phi_1^{min}, \phi_1^{max}, \dots \phi_d^{min}, \phi_d^{max} \in \mathbb{R}, \, \bm{\phi} \in \mathbb{R}^d$, $\gamma \in (0, 1)$, $M_1, \dots M_m \in \mathcal{M}$ with $\mathbf{w}_1, \dots \mathbf{w}_m \in \mathbb{R}^d$
    \LineComment{Initialize atoms and their probability distributions}
    \State Initialize ${\bm{\theta}}^1(s,a), \dots {\bm{\theta}}^m(s,a)$ to random values
    \State {\bf{for}} $i = 1, 2\dots d$ {\bf{do}} $\Psi_i^{min} \gets \frac{\phi_i^{min}}{1 - \gamma},\, \Psi_i^{max} \gets \frac{\phi_i^{max}}{1 - \gamma}$, $\Delta z_i \gets \frac{\Psi_i^{max} - \Psi_i^{min}}{N - 1}$
    \State {\bf{for}} $i = 1, 2 \dots d$ {\bf{do}} {\bf{for}} $j = 0, 1 \dots N - 1$ {\bf{do}} $z_{j,i}\gets \Psi_i^{min} + j \Delta z_i$
    \LineComment{Main training loop}
    \For{$t=1, 2\dots m$}
        \IndentLineComment{Commence training on task $M_t$}
        \For{$n_e = 1, 2 \dots N_e$}
        \State Initialize task $M_t$ with initial state $s$
        \For{$h=0,1\dots T$} 
            \IndentLineComment{Extract sufficient statistics from $\bm{\theta}^t(s,\cdot)$ and select the source task $M_c$ using GPI}
            \State {\bf{for}} $j = 1, 2 \dots m$ {\bf{do}} Compute $\tilde{\bm{\psi}}^j(s,\cdot), \, \tilde{\Sigma}^j(s,\cdot)$ using $\bm{\theta}^j(s,\cdot)$ and (\ref{eqn:c51_compute_moments}) and (\ref{eqn:c51_compute_var})
            \State $c \gets \argmax_j \max_b \lbrace \tr{\tilde{\bm{\psi}}^j(s,b)} \mathbf{w}_t - \beta {\tr{\mathbf{w}_t} \tilde{{\Sigma}}^j(s,b) \mathbf{w}_t} \rbrace$
            \LineComment{Sample action from the epsilon-greedy policy based on $\pi_c$}
            \State $\mathrm{random\_a} \sim \mathrm{Bernoulli}(\varepsilon)$
            \State \parbox[t]{\dimexpr\textwidth-\leftmargin-\labelsep-\labelwidth}{%
            \bf{if} $\mathrm{random\_a}$ \bf{then} $a \sim \mathrm{Uniform}(\mathcal{A})$ \bf{else} \\ $a \gets \argmax_b \lbrace \tr{\tilde{\bm{\psi}}^c(s,b)} \mathbf{w}_t - \beta {\tr{\mathbf{w}_t} \tilde{{\Sigma}}^c(s,b) \mathbf{w}_t} \rbrace$ \strut}
            \State Take action $a$ in $M_t$ and observe $r$ and $s'$
            \LineComment{Update ${\bm{\theta}}^1(s,a), \dots {\bm{\theta}}^m(s,a)$}
            \For{$c = 1,2\dots m$}
                \IndentLineComment{Extract sufficient statistics from $\bm{\theta}^c(s',\cdot)$ and select action $a' = \pi_c(s')$}
                \State Compute $\tilde{\bm{\psi}}^c(s',\cdot), \, \tilde{\Sigma}^c(s',\cdot)$ using $\bm{\theta}^c(s',\cdot)$ and (\ref{eqn:c51_compute_moments}) and (\ref{eqn:c51_compute_var})
                \State $a' \gets \argmax_b \lbrace \tr{\tilde{\bm{\psi}}^c(s',b)} \mathbf{w}_c - \beta {\tr{\mathbf{w}_c} \tilde{\Sigma}^c(s',b) \mathbf{w}_c} \rbrace$
                \LineComment{Apply Categorical Algorithm to update $\bm{\theta}^c(s,a)$}
                \State {\bf{for}} $j = 0, 1\dots N - 1$ {\bf{do}} {\bf{for}} $i = 1, \dots d$ {\bf{do}} $m_{j,i} \gets 0$
                \For{$i=1,2\dots d$}
                \For{$j=0,1\dots N-1$}
                    \IndentLineComment{Compute the projection of $\hat{\mathcal{T}}_i z_{j,i}$ onto the support $\mathcal{Z}_i$}
                    \State $\hat{\mathcal{T}}_i z_{j,i} \gets \mathrm{clip}\left(\phi_i(s,a,s') + \gamma z_{j,i}; [\Psi_i^{min}, \Psi_i^{max}] \right)$
                    \State $b_{j,i} \gets (\hat{\mathcal{T}}_i z_{j,i} - \Psi_i^{min})/ \Delta z_i$
                    \State $l \gets \lfloor b_{j,i} \rfloor, \, u \gets \lceil b_{j,i} \rceil$
                    \LineComment{Distribute probability of $\hat{\mathcal{T}}_i z_{j,i}$}
                    \State $m_{l,i} \gets m_{l,i} + p_{j,i}^c(s',a')(u-b_{j,i})$
                    \State $m_{u,i} \gets m_{u,i} + p_{j,i}^c(s',a')(b_{j,i} - l)$
                \EndFor
                \EndFor
                \State Backpropagate through $-\sum_{j,i} m_{j,i} \log p_{j,i}^c(s,a)$ to update $\bm{\theta}^c(s,a)$
            \EndFor
            \State $s \gets s'$
        \EndFor
        \EndFor
  \EndFor
\end{algorithmic}
\end{algorithm}

\subsection{Possible Generalizations of the Mean-Variance Approximation}
\label{subsec:parametric}

\paragraph{Cumulant-Generating Functions.}
The quantity $K_R(\beta) = \log \mathbb{E}[e^{\beta R}]$ in (\ref{eqn:entropic}) is often referred to as the \emph{cumulant-generating function}. The cumulant generating function admits the well-known Taylor expansion:
\begin{equation}
\label{eqn:entropic_utility_expansion}
    U_\beta[R] = \frac{1}{\beta} K_R(\beta) = \frac{1}{\beta} \sum_{n=1}^\infty \kappa_R(n) \frac{\beta^n}{n!} = \sum_{n=1}^\infty \kappa_R(n) \frac{\beta^{n-1}}{n!},
\end{equation}
where $\kappa_R(n)$ is the $n$-th \emph{cumulant} of the random variable $R$ \citepsupp{weisstein}. The mean-variance approximation (\ref{eqn:mvsf}) then follows directly from (\ref{eqn:entropic_utility_expansion}) by ignoring all terms of order $n \geq 3$. Another way to look at the mean-variance approximation is that it is the result of applying a \emph{Laplace approximation} to the return distribution prior to calculating its utility \citepsupp{bishop2006pattern}. While it is also possible to approximate (\ref{eqn:entropic_utility_expansion}) using orders of $n$ greater than 2, such approximations would no longer provide ``instantaneous" GPE. In particular, cumulants are much harder to compute as functions of $\mathbf{w}$ for $n = 3$ \citepsupp{boudt2020coskewness}, and no closed formulas are even known to us for $n \geq 4$. 

\paragraph{Elliptical Distributions.} 
The mean-variance approximation (\ref{eqn:mvsf}) results from making the distributional assumption $\Psi_h^\pi(s,a) \sim \mathcal{N}(\bm{\psi}_h^\pi(s,a), \Sigma_h^\pi(s,a))$. Since the normal distribution is a member of the class of {elliptical distributions}, a natural question to ask is whether GPE can apply to other members of this class of distributions as well. 

Formally, a random variable $X$ has an \emph{elliptical distribution} on $\mathbb{R}^d$ if there exists $\bm{\mu} \in \mathbb{R}^d$, positive definite $\Sigma \in \mathbb{R}^{d\times d}$ and a positive-valued function $\xi : \mathbb{R} \to \mathbb{R}$, and the characteristic function of $X$ has the form
    \begin{equation}
    \label{eqn:elliptical_characteristic}
        \mathbb{E}[e^{i \tr{\mathbf{t}} X}] = e^{i \tr{\mathbf{t}} \bm{\mu}} \xi(\tr{\mathbf{t}} \Sigma \mathbf{t}), \quad \forall \mathbf{t} \in \mathbb{R}^d.
    \end{equation}
Equivalently, for any random variable with characteristic function (\ref{eqn:elliptical_characteristic}), there exists a positive function $g_d : \mathbb{R} \to \mathbb{R}$ such that the density of $X$ is 
\begin{equation*}
    f_X(\mathbf{x}) \propto | \Sigma |^{-1/2} g_d\left(\tr{(\mathbf{x} - \bm{\mu})} \Sigma^{-1}(\mathbf{x} - \bm{\mu}) \right).
\end{equation*}
In either case, we write $X \sim \mathcal{E}_d(\bm{\mu}, \Sigma, \xi)$. One advantage of this parameterization is that $\bm{\mu}$ corresponds exactly to the mean of $X$, e.g. $\mathbb{E}[X] = \bm{\mu}$. Furthermore, if the covariance of $X$ exists, then it is equal to $\Sigma$ up to a positive multiplicative constant, e.g. $\mathrm{Var}[X] = c\Sigma$ for some $c > 0$\footnote{This implies that the Bellman updates (\ref{eqn:mvsf_covariance}) or (\ref{eqn:mvsf_covariance_discounted}) can still be used to learn $\Sigma_h^\pi$, but now the resulting estimates must be scaled by $c$ when computing the utilities, if $c$ is not one.}.

In order to connect this to the SF framework, we parameterize $\Psi_h^\pi(s,a) \sim \mathcal{E}_d(\bm{\psi}_h^\pi(s,a), \Sigma_h^\pi(s,a), \xi)$. Then, using the linearity property (\ref{eqn:utility_linear}, \ref{eqn:utility_linear_discounted}), GPE evaluates the entropic utilities of the random variables $\tr{\Psi_h^\pi(s,a)} \mathbf{w}$. Fortunately, affine transforms of elliptically distributed random variables are univariate elliptically distributed \citepsupp{owen1983class}.

\begin{lemma}
\label{lem:elliptical_properties}
    Let $X \sim \mathcal{E}_d(\bm{\mu}, \Sigma, \xi)$ and $\mathbf{w} \in \mathbb{R}^d$. Then, $\tr{X}\mathbf{w} \sim \mathcal{E}_1(\tr{\bm{\mu}} \mathbf{w}, \tr{\mathbf{w}}\Sigma \mathbf{w}, \xi)$.
\end{lemma}

\begin{wrapfigure}{r}{0.3\linewidth}
\centering
    \includegraphics[width=0.999\linewidth]{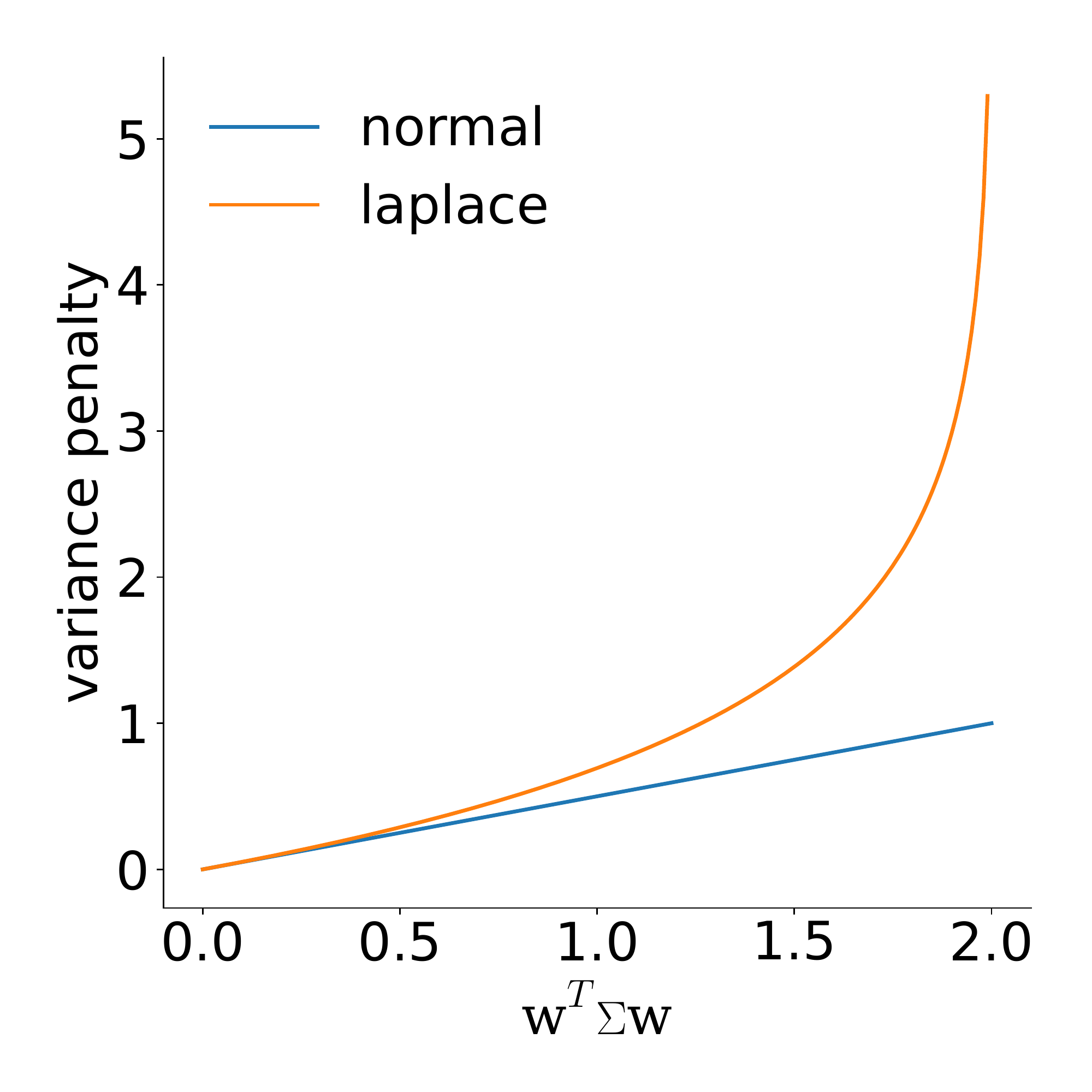}
    \caption{Comparing $\frac{1}{\beta}\log \circ \xi$ in (\ref{eqn:elliptical_mvsf}) for normal and Laplace distributions, for $\beta = 1$.}
    \label{fig:elliptical_utilities}
    \vspace{-2.5em}
\end{wrapfigure}

Applying Lemma \ref{lem:elliptical_properties} and then substituting $t = -i \beta$, the entropic utility becomes
\begin{equation}
\label{eqn:elliptical_mvsf}
    U_\beta[\tr{\Psi_h^\pi(s,a)}\mathbf{w}] = \tr{\bm{\psi}_h^\pi(s,a)} \mathbf{w} + \frac{1}{\beta}\log \xi\left(-\beta^2 \tr{\mathbf{w}} \Sigma_h^\pi(s,a) \mathbf{w} \right),
\end{equation}
and is also a mean-variance approximation. However, unlike (\ref{eqn:mvsf}) in which $\frac{1}{\beta}\log \circ \xi$ is the identity mapping, (\ref{eqn:elliptical_mvsf}) is allowed to depend \emph{non-linearly} on the variance of returns. For heavy-tailed distributions, $\xi$ should increase super-linearly for sufficiently large return variances, and thus (\ref{eqn:elliptical_mvsf}) will often be more sensitive to variance than (\ref{eqn:mvsf}). This phenomenon is clearly illustrated in Figure \ref{fig:elliptical_utilities} by comparing the variance penalties of the normal and Laplace distributions. Another advantage of this generalization is that the methodologies for estimating successor features and their covariances (Appendix \appendixref{subsec:algorithmic_total_reward}{A.1} and \appendixref{subsec:distributional_sf}{A.3}) can be directly applied to this more general setting.

\begin{table}[!tb]
\centering
    \begin{tabular}{ccc} \toprule
        Name & Parameters & $\frac{1}{\beta}\log \circ \xi$ \\ \midrule
        multivariate normal & $\bm{\mu}, \bm{\Sigma}$ & $\frac{\beta}{2} \tr{\mathbf{w}} \Sigma \mathbf{w}$ \\
        multivariate Student & $\bm{\mu}, \bm{\Sigma}, \nu$ & does not exist \\
        multivariate Laplace & $\bm{\mu}, \bm{\Sigma}$ & $-\frac{1}{\beta} \log(1 - \frac{\beta^2}{2}\tr{\mathbf{w}} \Sigma \mathbf{w})$ \\
        multivariate logistic & $\bm{\mu}, \bm{\Sigma}$ & $\frac{1}{\beta} \log \mathrm{B}(1-\beta \sqrt{\tr{\mathbf{w}} \Sigma \mathbf{w}}, 1 + \beta \sqrt{\tr{\mathbf{w}} \Sigma \mathbf{w}})$ \\
        \bottomrule
    \end{tabular}
\caption{Table of common elliptical distributions with corresponding variance penalties. Here $\mathrm{B}$ denotes the Beta function.}
\label{table:elliptical}
\end{table}

\paragraph{Skew-Elliptical Distributions.}
The elliptical distributions represent a well-known class of symmetric probability distributions, containing both heavy-tailed and light-tailed members as special cases. However, they cannot capture skew in the return distribution that often arises in strictly discounted MDPs. The class of \emph{generalized skew-elliptical distributions} (GSE) can model skew by extending the characteristic function (\ref{eqn:elliptical_characteristic}) to 
\begin{equation}
    \label{eqn:skew_elliptical_characteristic}
        \mathbb{E}[e^{i \tr{\mathbf{t}} X}] = e^{i \tr{\mathbf{t}} \bm{\mu}} \xi(\tr{\mathbf{t}} \Sigma \mathbf{t}) k_d(\mathbf{t}), \quad \forall \mathbf{t} \in \mathbb{R}^d,
\end{equation}
where $k_d : \mathbb{R}^d \to \mathbb{R}$ is some positive-valued function. In this case, we write $X \sim \mathcal{SE}_d(\bm{\mu}, \Sigma, \xi)$.

As for the elliptical distributions (Lemma \ref{lem:elliptical_properties}), it is possible to show that GSE distributions are also closed under affine transforms \citepsupp{shushi2018generalized}.

\begin{lemma}
\label{lem:skew_elliptical_properties}
    Let $X \sim \mathcal{SE}_d(\bm{\mu}, \Sigma, \xi)$ and $\mathbf{w} \in \mathbb{R}^d$. Then, $\tr{X}\mathbf{w}$ is univariate GSE with characteristic function $\mathbb{E}[e^{it\tr{X}\mathbf{w}}] = e^{it\tr{X}\bm{\mu}} \xi\left(t^2 \tr{\mathbf{w}}\Sigma \mathbf{w}\right) k(t;\mathbf{w},\Sigma)$ for some real-valued function $k$.
\end{lemma}

By using Lemma \ref{lem:skew_elliptical_properties} and the substitution $t = -i \beta$,
\begin{equation*}
    U_\beta[\tr{\Psi_h^\pi(s,a)}\mathbf{w}] = \tr{\bm{\psi}_h^\pi(s,a)} \mathbf{w} + \frac{1}{\beta}\log \xi\left(-\beta^2 \tr{\mathbf{w}} \Sigma_h^\pi(s,a) \mathbf{w} \right) + \frac{1}{\beta}\log k(-i \beta; \mathbf{w}, \Sigma_h^\pi(s,a)).
\end{equation*}
This new approximation generalizes (\ref{eqn:elliptical_mvsf}) through the introduction of the term $\log k$, which intuitively captures the skew of the return distribution. 

\paragraph{Mixtures Densities.}
One significant limitation of elliptical (and skew-elliptical) distributions to model returns is that they are unimodal, and can fail to capture multimodal risks in the environment. Consider the following \emph{mixture of elliptical distributions} on $\mathbb{R}^d$:
\begin{equation*}
    \begin{aligned}
        I &\sim \mathrm{Categorical}_K(\bm{\pi}), \\
        X \,|\, I = k &\sim \mathcal{E}_d(\bm{\mu}_k, \Sigma_k, \xi_k),
    \end{aligned}
\end{equation*}
where $\bm{\pi} \in \mathbb{R}^K$ satisfies $\pi_k \geq 0$ and $\sum_k \pi_k = 1$ and $k = 1, \dots K$ define the possible modes of the distribution. In other words, each component of the mixture is a member of an elliptical distribution. This model extends the standard Gaussian mixture model, and can approximate any continuous distribution to arbitrary accuracy provided $K$ is chosen sufficiently large \citepsupp{scott2015multivariate,titterington1985statistical}. In the context of risk-aware transfer (Theorem \ref{thm:gpi_aux}, Theorem \ref{thm:gpi}) this means that the approximation error terms in $\varepsilon$ associated with approximating $U_\beta$ could in principle be driven to zero.

Applying Lemma \ref{lem:elliptical_properties} to each component, $\tr{X}\mathbf{w} \,|\, I = k \sim \mathcal{E}_1(\tr{\bm{\mu}_k}\mathbf{w}, \tr{\mathbf{w}}\Sigma_k\mathbf{w}, \xi_k)$, and so $\tr{X}\mathbf{w}$ is a mixture of univariate elliptical distributions. Now, (\ref{eqn:utility_linear}, \ref{eqn:utility_linear_discounted}) can be computed using the law of total expectation,
\begin{align*}
    U_\beta[\tr{\Psi_h^\pi(s,a)}\mathbf{w}] 
    = \frac{1}{\beta} \log \sum_{k=1}^K \pi_{k}^\pi(s,a) e^{\beta \tr{\psi_{k}^\pi(s,a)}\mathbf{w}} \xi_k(-\beta^2 \tr{\mathbf{w}} \Sigma_{k}^\pi(s,a)\mathbf{w}).
\end{align*}
This expression does not simplify further unless $K = 1$ and is thus \emph{not} a mean-variance approximation, although it \emph{is} a generalization of (\ref{eqn:elliptical_mvsf}). It can be computed numerically by using the log-sum-exp trick, and the parameters of its associated mixture density could be learned in the Bellman framework using expectation propagation \citepsupp{morimura2010parametric}.

\section{Proofs of Theoretical Results}
\label{sec:proofs}

In this section, we verify all the theoretical claims stated in the main paper for the episodic MDPs. For ease of exposition and due to space limitations, we also state and prove similar results for discounted MDPs in this section\footnote{While the analysis of GPI in the risk-neutral setting \citep{barreto2017successor} follows \citetsupp{strehl2005theoretical}, our analysis of risk-aware GPI is also inspired by \citetsupp{huang2020stochastic}.}.

\subsection{Proof of Lemma \ref{lem:entropic_properties}}
\label{subsec:proof_entropic_properties}

\begin{customlemma}{1}
    Let $\beta \in \mathbb{R}$ and $X, Y$ be arbitrary random variables on $\Omega$. Then:
    \begin{enumerate}[label=\textbf{A\arabic*},noitemsep]
        \item (monotonicity) if $\mathbb{P}(X \geq Y) = 1$ then $U_\beta[X] \geq U_\beta[Y]$
        \item (cash invariance) $U_\beta[X + c] = U_\beta[X] + c$ for every $c \in \mathbb{R}$
        \item (convexity) if $\beta < 0$ ($\beta > 0$) then $U_\beta$ is a concave (convex) function
        \item (non-expansion) for $f, g : \Omega \to \Omega$, it follows that 
        \begin{equation*}
            \left|U_\beta[f(X)] - U_\beta[g(X)]\right| \leq \sup_{P \in \mathscr{P}_X(\Omega)}\mathbb{E}_{P}|f(X)-g(X)|,
        \end{equation*}
        where $\mathscr{P}_X(\Omega)$ is the set of all probability distributions on $\Omega$ that are absolutely continuous w.r.t. the true distribution of $X$.
    \end{enumerate}
\end{customlemma}

\begin{proof}
    The first three properties are derived in \citet{follmer2002convex}. As for the fourth property, we use {\bfseries A3} and convex duality \citep{follmer2002convex}, \citepsupp{maccheroni2006ambiguity} to write
    \begin{equation}
    \label{eqn:convex_duality}
        U_\beta[R] = \sup_{P \in \mathscr{P}_R(\Omega)} \left\lbrace \mathbb{E}_{P}[R] - \frac{1}{\beta}D(P || P^*) \right\rbrace,
    \end{equation}
    where $\mathscr{P}_R(\Omega)$ is the set of all probability distributions on $\Omega$ absolutely continuous w.r.t. the true distribution $P^*$ of $R$, $D$ is the KL-divergence between $P$ and $P^*$. Now, for $f, g : \Omega \to \Omega$, $f$ and $g$ are bounded and hence $P$-integrable for any $P \in \mathscr{P}_R(\Omega)$, and using (\ref{eqn:convex_duality}):
    \begin{align*}
        &\left|U_\beta[f(X)] - U_\beta[g(X)] \right|\\
        &= \left|\sup_{P \in \mathscr{P}_X(\Omega)} \left\lbrace \mathbb{E}_{P}[f(X)] - \frac{1}{\beta}D(P || P^*) \right\rbrace - \sup_{P \in \mathscr{P}_X(\Omega)} \left\lbrace \mathbb{E}_{P}[g(X)] - \frac{1}{\beta}D(P || P^*) \right\rbrace \right| \\
        &\leq \sup_{P \in \mathscr{P}_X(\Omega)} \left|\mathbb{E}_{P}[f(X)] -  \mathbb{E}_{P}[g(X)] \right| \\
        &\leq \sup_{P \in \mathscr{P}_X(\Omega)} \mathbb{E}_{P}| f(X) - g(X)|.
    \end{align*}
    This completes the proof.
\end{proof}

\subsection{Proofs of Theorem \ref{thm:gpi_aux} and \ref{thm:gpi} for Episodic MDPs}
\label{subsec:proofs_episodic}

\begin{customtheorem}{1}
        Let $\pi_1, \dots \pi_n$ be arbitrary deterministic Markov policies with utilities $\tilde{\mathcal{Q}}_{h,\beta}^{\pi_1}, \dots \tilde{\mathcal{Q}}_{h,\beta}^{\pi_n}$ evaluated in an arbitrary task $M$, such that $|\tilde{\mathcal{Q}}_{h,\beta}^{\pi_i}(s,a) - \mathcal{Q}_{h,\beta}^{\pi_i}(s,a)| \leq \varepsilon$ for all $s \in \mathcal{S}$, $a \in \mathcal{A}$, $i = 1 \dots n$ and $h \in \mathcal{T}$. Define
    \begin{equation*}
        \pi_h(s) \in \argmax_{a \in \mathcal{A}} \max_{i=1\dots n} \tilde{\mathcal{Q}}_{h,\beta}^{\pi_i}(s,a), \quad \forall s \in \mathcal{S}.
    \end{equation*}
    Then,
    \begin{equation*}
        \mathcal{Q}_{h,\beta}^{\pi}(s, a) \geq  \max_{i}\mathcal{Q}_{h,\beta}^{\pi_i}(s, a) - 2 ({T - h + 1}) \varepsilon, \quad h \leq T.
    \end{equation*}
\end{customtheorem}

\begin{proof}
    We have for all $h$ that
    \begin{align*}
        |\max_i \mathcal{Q}_{h}^{\pi_i}(s, a) - \max_i \tilde{\mathcal{Q}}_{h}^{\pi_i}(s, a)|\leq \max_i | \mathcal{Q}_{h}^{\pi_i}(s, a) - \tilde{\mathcal{Q}}_{h}^{\pi_i}(s, a) | \leq \varepsilon
    \end{align*}
    We proceed by induction on $h$. Clearly, the desired result holds for $h = T + 1$ since $\mathcal{Q}_{T+1,\beta}^{\pi_i}(s,a) = \tilde{\mathcal{Q}}_{T+1,\beta}^{\pi_i}(s,a) = 0$ uniformly. Next, suppose that $\mathcal{Q}_{h+1,\beta}^\pi(s,a) \geq \max_i \mathcal{Q}_{h+1,\beta}^{\pi_i}(s,a) - 2 \varepsilon (T - h)$ holds uniformly at time $h + 1$. Using {\bfseries A1} and {\bfseries A2} of Lemma \ref{lem:entropic_properties}:
    \begin{align*}
        \mathcal{Q}_{h,\beta}^\pi(s,a)
        &= U_\beta[r(s,a,s') + \mathcal{Q}_{h +1,\beta}^\pi(s',\pi_{h+1}(s'))] \\
        &\geq U_\beta[r(s,a,s') + \max_i \mathcal{Q}_{h +1,\beta}^{\pi_i}(s',\pi_{h+1}(s')) - 2\varepsilon(T-h)] \\
        &\geq U_\beta[r(s,a,s') + \max_i \tilde{\mathcal{Q}}_{h +1,\beta}^{\pi_i}(s',\pi_{h+1}(s')) - 2\varepsilon(T-h) - \varepsilon] \\
        &= U_\beta[r(s,a,s') + \max_{a'}\max_i \tilde{\mathcal{Q}}_{h +1,\beta}^{\pi_i}(s',a') - 2\varepsilon(T-h) - \varepsilon] \\
        &\geq U_\beta[r(s,a,s') + \max_{a'}\max_{i} {\mathcal{Q}}_{h+1,\beta}^{\pi_i}(s',a') - 2\varepsilon(T-h + 1)] \\
        &\geq U_\beta[r(s,a,s') + \max_{i}\max_{a'} {\mathcal{Q}}_{h+1,\beta}^{\pi_i}(s',a')] - 2\varepsilon(T-h + 1) \\
        &\geq U_\beta[r(s,a,s') + \max_{i} {\mathcal{Q}}_{h+1,\beta}^{\pi_i}(s',\pi_{i,h+1}(s'))] - 2\varepsilon(T-h + 1) \\
        &\geq U_\beta[r(s,a,s') + {\mathcal{Q}}_{h+1,\beta}^{\pi_i}(s',\pi_{i,h+1}(s'))] - 2\varepsilon(T-h + 1) \\
        &= \mathcal{Q}_{h,\beta}^{\pi_i}(s,a) - 2\varepsilon(T-h + 1).
    \end{align*}
    Since $i$ is arbitrary, the proof is complete.
\end{proof}

\begin{lemma}
\label{lem:aux1}
    Let $\mathcal{Q}_{h}^{ij}$ be the utility of policy $\pi_i^*$ evaluated in task $j$ at time $h$. Then for all $i, j$,
    \begin{equation*}
        \sup_{s,a} \left|\mathcal{Q}_{h}^{ii}(s,a) - \mathcal{Q}_{h}^{jj}(s,a) \right| \leq (T - h + 1)\delta_{ij}.
    \end{equation*}
\end{lemma}

\begin{proof}
    Let $\Delta_{ij}(h) = \sup_{s,a} | \mathcal{Q}_{h}^{ii}(s,a) - \mathcal{Q}_{h}^{jj}(s,a)|$. Define $\mathscr{P}_{s,a}$ to be the set of probability distributions that are absolutely continuous with respect to $P(\cdot | s, a)$. Then, using {\bfseries A4} from Lemma \ref{lem:entropic_properties}:
    \begin{align*}
        &\Delta_{ij}(h) \\
        &= \sup_{s,a}\left| \mathcal{Q}_{h}^{ii}(s,a) - \mathcal{Q}_{h}^{jj}(s,a) \right| \\
        &= \sup_{s,a}\left| U_\beta[r_i(s,a,s') + \max_{a'} \mathcal{Q}_{h+1}^{ii}(s', a')] -  U_\beta[r_j(s,a,s') + \max_{a'} \mathcal{Q}_{h+1}^{jj}(s',a')] \right| \\
        &\leq \sup_{s,a} \sup_{P' \in \mathscr{P}_{s,a}} \mathbb{E}_{s'\sim P'(\cdot | s, a)}\left|r_i(s,a,s') -r_j(s,a,s') + \max_{a'} \mathcal{Q}_{h+1}^{ii}(s', a') - \max_{a'} \mathcal{Q}_{h+1}^{jj}(s', a')\right| \\
        &\leq \sup_{s,a} \sup_{s'} \left|r_i(s,a,s') -r_j(s,a,s') + \max_{a'} \mathcal{Q}_{h+1}^{ii}(s', a') - \max_{a'} \mathcal{Q}_{h+1}^{jj}(s', a')\right| \\
        &\leq \sup_{s,a,s'} \left|r_i(s,a,s') - r_j(s,a,s')\right| + \sup_{s,a,s'} \left|\max_{a'} \mathcal{Q}_{h+1}^{ii}(s', a') - \max_{a'} \mathcal{Q}_{h+1}^{jj}(s', a')\right| \\
        &= \delta_{ij} + \sup_{s,a} \left|\mathcal{Q}_{h+1}^{ii}(s,a) - \mathcal{Q}_{h+1}^{jj}(s,a)\right|\\
        &= \delta_{ij} + \Delta_{ij}(h+1).
    \end{align*}
    Starting with $\Delta_{ij}(T+1) = 0$ and proceeding by backward induction, we have $\Delta_{ij}(h) \leq \delta_{ij} (T - h + 1)$ for all $h$. 
\end{proof}

\begin{lemma}
    \label{lem:aux2}
    \begin{equation*}
        \sup_{s,a} \left|\mathcal{Q}_{h}^{jj}(s,a) - \mathcal{Q}_{h}^{ji}(s,a) \right| \leq (T - h + 1)\delta_{ij}.
    \end{equation*}
\end{lemma}

\begin{proof}
    Define $\Gamma_{ij}(h) = \sup_{s,a} | \mathcal{Q}_{h}^{jj}(s,a) - \mathcal{Q}_{h}^{ji}(s,a) |$. Then using {\bfseries A4} from Lemma \ref{lem:entropic_properties}:
    \begin{align*}
        &\Gamma_{ij}(h) \\
        &= \sup_{s,a}\left| \mathcal{Q}_{h}^{jj}(s,a) - \mathcal{Q}_{h}^{ji}(s,a) \right| \\
        &= \sup_{s,a}\left|U_\beta[r_j(s,a,s') +  \mathcal{Q}_{h+1}^{jj}(s',\pi_{j,h+1}^*(s'))] - U_\beta[r_i(s,a,s') + \mathcal{Q}_{h+1}^{ji}(s', \pi_{j,h+1}^*(s'))] \right| \\
        &\leq \sup_{s,a}\sup_{P' \in \mathscr{P}_{s,a}} \mathbb{E}_{s'\sim P'(\cdot | s, a)} \Big| r_i(s,a,s') -r_j(s,a,s') + \mathcal{Q}_{h+1}^{jj}(s', \pi_{j,h+1}^*(s')) - \mathcal{Q}_{h+1}^{ji}(s', \pi_{j,h+1}^*(s'))\Big| \\
        &\leq \sup_{s,a}\sup_{s'} \Big| r_i(s,a,s') -r_j(s,a,s') + \mathcal{Q}_{h+1}^{jj}(s', \pi_{j,h+1}^*(s')) - \mathcal{Q}_{h+1}^{ji}(s', \pi_{j,h+1}^*(s'))\Big| \\
        &\leq \sup_{s,a,s'} \left|r_i(s,a,s') - r_j(s,a,s')\right| + \sup_{s,a,s'}\left| \mathcal{Q}_{h+1}^{jj}(s', \pi_{j,h+1}^*(s')) - \mathcal{Q}_{h+1}^{ji}(s', \pi_{j,h+1}^*(s'))\right| \\
        &\leq \delta_{ij} + \sup_{s,a} \left| \mathcal{Q}_{h+1}^{jj}(s,a) - \mathcal{Q}_{h+1}^{ji}(s,a) \right| \\
        &= \delta_{ij} + \Gamma_{ij}(h+1).
    \end{align*}
    Thus, $\Gamma_{ij}(h) \leq \delta_{ij}(T - h + 1)$ as claimed.
\end{proof}

\begin{customtheorem}{2}
     Let $\mathcal{Q}_{h,\beta}^{\pi_i^*}$ be the utilities of optimal Markov policies $\pi_i^*$ evaluated in task $M$. Furthermore, let $\tilde{\mathcal{Q}}_{h,\beta}^{\pi_i^*}$ be such that $|\tilde{\mathcal{Q}}_{h,\beta}^{\pi_i^*}(s, a) - \mathcal{Q}_{h,\beta}^{\pi_i^*}(s,a) | < \varepsilon$ for all $s \in \mathcal{S}, \, a \in \mathcal{A}$, $h \in \mathcal{T}$ and $i = 1 \dots n$. Similarly, let $\pi$ be the corresponding policy in (\ref{eqn:gpi_policy_aux}). Finally, let $\delta_r = \min_{i=1\dots n} \sup_{s,a,s'}| r(s,a,s') - r_i(s,a,s')|$. Then,
    \begin{equation*}
    \begin{aligned}
        \left|\mathcal{Q}_{h,\beta}^{\pi}(s,a) - \mathcal{Q}_{h,\beta}^*(s,a) \right| \leq 2(T - h + 1) (\delta_r + \varepsilon) , \quad h \leq T.
    \end{aligned}
    \end{equation*}
\end{customtheorem}

\begin{proof}
Using Theorem \ref{thm:gpi_aux} and the triangle inequality:
\begin{align*}
    \left|\mathcal{Q}_{h,\beta}^{{\pi}}(s,a) - \mathcal{Q}_{h,\beta}^{*}(s,a) \right| \leq \left|\mathcal{Q}_{h,\beta}^{{\pi}_j^*}(s,a) - \mathcal{Q}_{h,\beta}^{*}(s,a) \right| + 2 (T - h + 1) \varepsilon.
\end{align*}

The goal now is to bound the first term. By the triangle inequality and Lemma \ref{lem:aux1} and Lemma \ref{lem:aux2}, $| \mathcal{Q}_{h}^{ii}(s,a) - \mathcal{Q}_{h}^{ji}(s,a)| \leq |\mathcal{Q}_{h}^{ii}(s,a) - \mathcal{Q}_{h}^{jj}(s,a)| + | \mathcal{Q}_{h}^{jj}(s,a) - \mathcal{Q}_{h}^{ji}(s,a)| = 2 (T - h + 1) \delta_{ij}$. Finally, designating $j$ as source task $j$ and $i$ as target task, and substituting this bound into the first inequality above yields the desired result.
\end{proof}

\subsection{Proofs of Theorem \ref{thm:gpi_aux} and \ref{thm:gpi} for Discounted MDPs}
\label{subsec:proofs_discounted}

\begin{theorem}
\label{thm:gpi_aux_discounted}
    Let $\pi_1, \dots \pi_n$ be arbitrary deterministic Markov policies with utilities $\tilde{\mathcal{J}}_{\beta}^{\pi_1}, \dots \tilde{\mathcal{J}}_{\beta}^{\pi_n}$ evaluated in an arbitrary task $M$, such that $|\tilde{\mathcal{J}}_{\beta}^{\pi_i}(s,a,z) - \mathcal{J}_{\beta}^{\pi_i}(s,a,z)| \leq \varepsilon z$ for all $s, a, z$ and $i$. Define
    \begin{equation}
    \label{eqn:gpi_policy_discounted}
        \pi(s,z) \in \argmax_{a \in \mathcal{A}} \max_{i=1\dots n} \tilde{\mathcal{J}}_{\beta}^{\pi_i}(s,a,z), \quad \forall s \in \mathcal{S},\, z \in \mathcal{Z}.
    \end{equation}
    Then,
    \begin{equation*}
        \mathcal{J}_{\beta}^{\pi}(s, a,z) \geq \max_{i}\mathcal{J}_{\beta}^{\pi_i}(s, a,z) - \frac{2 \varepsilon}{1 - \gamma} z.
    \end{equation*}
\end{theorem}

\begin{proof}
    Define $\mathcal{J}_{\beta}^{max}(s,a,z) = \max_i \mathcal{J}_{\beta}^{\pi_i}(s,a,z)$ and $\tilde{\mathcal{J}}_{\beta}^{max}(s,a,z) = \max_i \tilde{\mathcal{J}}_{\beta}^{\pi_i}(s,a,z)$. We have:
    \begin{align*}
        |\mathcal{J}_{\beta}^{max}(s, a, z) - \tilde{\mathcal{J}}_{\beta}^{max}(s, a,z) | \leq \max_i | \mathcal{J}_{\beta}^{\pi_i}(s, a,z) - \tilde{\mathcal{J}}_{\beta}^{\pi_i}(s, a,z) | \leq \varepsilon z.
    \end{align*}
    Let $T_\beta^\pi$ be the operator corresponding to (\ref{eqn:entropic_bellman_discounted}). Then using {\bfseries A1} and {\bfseries A2} of Lemma \ref{lem:entropic_properties} leads to:
    \begin{align*}
        T_\beta^\pi \tilde{\mathcal{J}}_{\beta}^{max}(s,a,z) 
        &= U_\beta\left[z r(s,a,s') + \tilde{\mathcal{J}}_{\beta}^{max}(s',\pi(s',\gamma z), \gamma z)\right] \\
        &= U_\beta\left[z r(s,a,s') + \max_{a'} \tilde{\mathcal{J}}_{\beta}^{max}(s',a', \gamma z)\right] \\
        &\geq U_\beta\left[z r(s,a,s') + \max_{a'} {\mathcal{J}}_{\beta}^{max}(s',a', \gamma z)\right] - \gamma \varepsilon z  \\
        &\geq U_\beta\left[z r(s,a,s') + {\mathcal{J}}_{\beta}^{max}(s',\pi_i(s',\gamma z), \gamma z)\right] - \gamma \varepsilon z  \\
        &\geq U_\beta\left[z r(s,a,s') + {\mathcal{J}}_{\beta}^{\pi_i}(s',\pi_i(s',\gamma z), \gamma z)\right] - \gamma \varepsilon z \\
        &= T_{\beta}^{\pi_i} {\mathcal{J}}_{\beta}^{\pi_i}(s, a, z) - \gamma \varepsilon z  \\
        &= {\mathcal{J}}_{\beta}^{\pi_i}(s, a, z) - \gamma \varepsilon z  \\
        &\geq \max_i{\mathcal{J}}_{\beta}^{\pi_i}(s, a, z) - \gamma \varepsilon z  \\
        &\geq \tilde{\mathcal{J}}_{\beta}^{max}(s,a,z) - \varepsilon z - \gamma \varepsilon z 
    \end{align*}
    Finally, using {\bfseries A1} of Lemma \ref{lem:entropic_properties} and the fact that $T_\beta^\pi$ has a unique fixed point \citep{bauerle2014more}:
    \begin{align*}
        \mathcal{J}_\beta^\pi(s,a,z) 
        &= \lim_{k \to \infty} (T_\beta^\pi)^k \tilde{\mathcal{J}}_{\beta}^{max}(s,a,z) 
        \geq \tilde{\mathcal{J}}_{\beta}^{max}(s,a,z) - (1 + \gamma) \frac{ \varepsilon}{1 - \gamma} z \\
        &\geq {\mathcal{J}}_{\beta}^{max}(s,a,z) - \frac{ 2 \varepsilon}{1 - \gamma} z,
    \end{align*}
    and is the desired result.
\end{proof}

\begin{lemma}
\label{lem:aux1_discounted}
    Define $\mathcal{J}_j^i(s,a,z)$ be the utility of optimal policy $\pi_i^*$ on the augmented MDP $i$ when evaluated in the augmented MDP $j$. Furthermore, let $\delta_{ij} = \sup_{s,a,s'} |r_i(s,a,s') - r_j(s,a,s')|$. Then,
    \begin{equation*}
        \sup_{s,a} \left|\mathcal{J}_i^i(s,a,z) - \mathcal{J}_j^j(s,a,z) \right| \leq \frac{ \delta_{ij}}{1 - \gamma} z.
    \end{equation*}
\end{lemma}

\begin{proof}
    Define $\Delta_{ij}(z) = \sup_{s,a}\left|\mathcal{J}_i^i(s,a,z) - \mathcal{J}_j^j(s,a,z) \right|$. Let $\mathscr{P}_{s,a}$ be the set of probability distributions for the one-step transitions of the augmented MDP, e.g. $P((s',z') | (s,z), a)$ that are absolutely continuous w.r.t. the true distribution. Since $P((s',z') | (s,z), a) = P(s'|s,a) \delta_{z\gamma}(z')$, and $\delta_{z\gamma}(z')$ is absolutely continuous only w.r.t. itself, the set $\mathscr{P}_{s,a}$ consists of all products $P(s'|s,a) \delta_{z\gamma}(z')$, where $P(s'|s,a)$ is absolutely continuous w.r.t. the true dynamics of the original MDP. 
    
    Now, using {\bfseries A4} from Lemma \ref{lem:entropic_properties}:
    \begin{align*}
        &\Delta_{ij}(z) \\
        &= \sup_{s,a}\left| \mathcal{J}_{i}^{i}(s,a,z) - \mathcal{J}_{j}^{j}(s,a,z) \right| \\
        &= \sup_{s,a}\left| U_\beta[z r_i(s,a,s') + \max_{a'} \mathcal{J}_{i}^{i}(s', a',\gamma z)] -  U_\beta[z r_j(s,a,s') + \max_{a'} \mathcal{J}_{j}^{j}(s',a',\gamma z)] \right| \\
        &\leq \sup_{s,a} \sup_{P \in \mathscr{P}_{s,a}} \mathbb{E}_{s'\sim P(\cdot | s, a)}\left|z r_i(s,a,s') - z r_j(s,a,s') + \max_{a'} \mathcal{J}_{i}^{i}(s', a', \gamma z) - \max_{a'} \mathcal{J}_{j}^{j}(s', a', \gamma z)\right| \\
        &\leq z \sup_{s,a,s'} \left|r_i(s,a,s') - r_j(s,a,s')\right| + \sup_{s,a,s'}\left|\max_{a'} \mathcal{J}_{i}^{i}(s', a', \gamma z) - \max_{a'} \mathcal{J}_{j}^{j}(s', a', \gamma z)\right| \\
        &= z \delta_{ij} + \sup_{s,a} \left|\mathcal{J}_{i}^{i}(s,a, \gamma z) - \mathcal{J}_{j}^{j}(s,a, \gamma z)\right|\\
        &= z \delta_{ij} + \Delta_{ij}(\gamma z).
    \end{align*}
    Repeating the above bounding procedure leads to:
    \begin{align*}
        \Delta_{ij}(z) 
        &\leq z \delta_{ij} + \Delta_{ij}(\gamma z) \\
        &\leq z \delta_{ij} + \gamma z \delta_{ij} +  \Delta_{ij}(\gamma^2 z) \\
        &\phantom{=} \vdots \\
        &\leq z \delta_{ij} + \gamma z \delta_{ij} + \gamma^2 z \delta_{ij} + \dots = \frac{\delta_{ij}}{1-\gamma} z,
    \end{align*}
    and completes the proof.
\end{proof}

\begin{lemma}
\label{lem:aux2_discounted}
    \begin{equation*}
        \sup_{s,a} \left| \mathcal{J}_j^j(s,a,z) - \mathcal{J}_i^j(s,a,z) \right| \leq \frac{\delta_{ij}}{1-\gamma} z.
    \end{equation*}
\end{lemma}

\begin{proof}
    Define $\Gamma_{ij}(z) = \sup_{s,a} | \mathcal{J}_{j}^{j}(s,a,z) - \mathcal{J}_{i}^{j}(s,a,z)|$. Then, using {\bfseries A4} from Lemma \ref{lem:entropic_properties} and the technique from Lemma \ref{lem:aux1_discounted}:
    \begin{align*}
        &\Gamma_{ij}(z) \\
        &= \sup_{s,a}\left| \mathcal{J}_{j}^{j}(s,a,z) - \mathcal{J}_{i}^{j}(s,a,z) \right| \\
        &= \sup_{s,a}\left|U_\beta[z r_j(s,a,s') +  \mathcal{J}_{j}^{j}(s',\pi_{j}^*(s',\gamma z),\gamma z)] - U_\beta[z r_i(s,a,s') + \mathcal{J}_{i}^{j}(s', \pi_{j}^*(s',\gamma z), \gamma z)] \right| \\
        &\leq \sup_{s,a}\sup_{P \in \mathscr{P}_{s,a}} \mathbb{E}_{s'\sim P(\cdot | s, a)} \Big| z r_i(s,a,s') - z r_j(s,a,s') + \mathcal{J}_{j}^{j}(s', \pi_{j}^*(s',\gamma z), \gamma z) - \mathcal{J}_{i}^{j}(s', \pi_{j}^*(s',\gamma z), \gamma z)\Big| \\
        &\leq z \sup_{s,a,s'} \left|r_i(s,a,s') - r_j(s,a,s')\right| + \sup_{s,a,s'}\left| \mathcal{J}_{j}^{j}(s', \pi_{j}^*(s',\gamma z), \gamma z) - \mathcal{J}_{i}^{j}(s', \pi_{j}^*(s', \gamma z), \gamma z)\right| \\
        &\leq z \delta_{ij} + \sup_{s,a} \left| \mathcal{J}_{j}^{j}(s,a,\gamma z) - \mathcal{J}_{i}^{j}(s,a, \gamma z) \right| \\
        &= z \delta_{ij} + \Gamma_{ij}(\gamma z) \\
        &\leq \frac{\delta_{ij}}{1 - \gamma} z.
    \end{align*}
    The proof is complete.
\end{proof}

\begin{theorem}
\label{thm:gpi_discounted}
    Let $\mathcal{J}_{\beta}^{\pi_i^*}$ be the utilities of optimal Markov policies $\pi_i^*$ evaluated in some task $M$. Furthermore, let $\tilde{\mathcal{J}}_{\beta}^{\pi_i^*}$ be such such that $|\tilde{\mathcal{J}}_{\beta}^{\pi_i^*}(s, a,z) - \mathcal{J}_{\beta}^{\pi_i^*}(s,a,z) | < \varepsilon z$ for all $s \in \mathcal{S}, \, a \in \mathcal{A}$, $z \in \mathcal{Z}$ and $i = 1 \dots n$, and $\pi$ be the corresponding policy in (\ref{eqn:gpi_policy_discounted}). Finally, let $\delta_r = \min_{i=1\dots n} \sup_{s,a,s'}| r(s,a,s') - r_i(s,a,s')|$. Then,
    \begin{equation*}
    \begin{aligned}
        \left|\mathcal{J}_{\beta}^{\pi}(s,a,z) - \mathcal{J}_{\beta}^*(s,a,z) \right| \leq \frac{2 (\delta_r + \varepsilon)}{1 - \gamma} z.
    \end{aligned}
    \end{equation*}
\end{theorem}

\begin{proof}
Using Theorem \ref{thm:gpi_aux_discounted}:
\begin{align*}
    \mathcal{J}_{\beta}^{\pi}(s,a,z) - \mathcal{J}_{\beta}^{*}(s,a,z)
        &= \mathcal{J}_{\beta}^{\pi}(s,a,z) - \mathcal{J}_{\beta}^{\pi_j^*}(s,a,z) + \mathcal{J}_{\beta}^{\pi_j^*}(s,a,z)  - \mathcal{J}_{\beta}^{*}(s,a,z) \\
        &\geq \frac{2\varepsilon}{1 - \gamma}z + \mathcal{J}_{\beta}^{\pi_j^*}(s,a,z) - \mathcal{J}_{\beta}^{*}(s,a,z).
\end{align*}

The goal now is to bound the difference between the last two terms. Let $\mathcal{J}_{j}^{i}(s,a)$ be the entropic utility of the optimal policy $\pi_i^*$ evaluated in the augmented MDP for task $j$. Then, by the triangle inequality, $| \mathcal{J}_{i}^{i}(s,a,z) - \mathcal{J}_{i}^{j}(s,a,z)| \leq |\mathcal{J}_{i}^{i}(s,a,z) - \mathcal{J}_{j}^{j}(s,a,z)| + | \mathcal{J}_{j}^{j}(s,a,z) - \mathcal{J}_{i}^{j}(s,a,z)|$. Applying Lemma \ref{lem:aux1_discounted} and Lemma \ref{lem:aux2_discounted}, we have $| \mathcal{J}_{i}^{i}(s,a,z) - \mathcal{J}_{i}^{j}(s,a,z)| \leq \frac{2 \delta_{ij}}{1 - \gamma} z$, and the result follows.
\end{proof}

\subsection{Proof of Theorem \ref{thm:sigma_convergence}}
\label{subsec:proofs_covariance}

The proofs depend on the following result adapted from \citet{sherstan2018comparing}.

\begin{lemma}
    \label{lem:bellman_zero}
    Let $X$ be a random vector in $\mathbb{R}^d$ that depends only on $s_h, a_h, r_{h}$ and $s_{h+1}$. Then,
    \begin{equation*}
        \mathbb{E}\left[X \tr{(\Psi_{h+1}^\pi(s',\pi_{h+1}(s')) - \bm{\psi}_{h+1}^\pi(s', \pi_{h+1}(s'))} \,|\, s_h = s,\, a_h = a\right] = 0.
    \end{equation*}
\end{lemma}

We first demonstrate that the Bellman equation (\ref{eqn:covariance_Bellman}) is correct for our problem.

\begin{lemma}
\label{lem:covariance_bellman}
    \begin{equation*}
        \Sigma_h^\pi(s,a) = \mathbb{E}_{s' \sim P(\cdot | s, a)}\left[\bm{\delta}_h \tr{\bm{\delta}_h} +  \Sigma_{h+1}^\pi(s', \pi_{h+1}(s')) \,|\, s_h = s,\, a_h = a \right].
    \end{equation*}
\end{lemma}

\begin{proof}
    Let $\Psi_h^\pi(s,a) = \bm{\phi}_h + \bm{\phi}_{h+1} + \dots$ and define $\bm{\xi}_h^\pi(s,a) = \Psi_h^\pi(s,a) - \bm{\psi}_h^\pi(s,a)$. By definition of successor features, we have:
    \begin{align*}
        \bm{\xi}_h^\pi(s,a) 
        &= \Psi_h^\pi(s,a) - {\bm{\psi}}_h^\pi(s,a) \\ 
        &= \bm{\phi}_h + \bm{\psi}_{h+1}^\pi(s', \pi_{h+1}(s')) - \bm{\psi}_h^\pi(s,a) + (\Psi_{h+1}^\pi(s', \pi_{h+1}(s')) - \bm{\psi}_{h+1}^\pi(s', \pi_{h+1}(s'))) \\
        &= \bm{\delta}_h + \bm{\xi}_{h+1}^\pi(s',\pi_{h+1}(s'))
    \end{align*}
    By definition, the covariance is:
    \begin{align*}
        \Sigma_h^\pi(s,a) 
        &= \mathbb{E}\left[(\Psi_h^\pi(s,a) - \bm{\psi}_h^\pi(s,a))\tr{(\Psi_h^\pi(s,a) - \bm{\psi}_h^\pi(s,a))} \, | \, s_h = s, \, a_h = a \right] \\
        &= \mathbb{E}\left[\bm{\xi}_h^\pi(s,a)\tr{\bm{\xi}_h^\pi(s,a)} \, | \, s_h = s, \, a_h = a \right] \\
        &= \mathbb{E}\left[(\bm{\delta}_h + \bm{\xi}_{h+1}^\pi(s',\pi_{h+1}(s')))\tr{(\bm{\delta}_h + \bm{\xi}_{h+1}^\pi(s',\pi_{h+1}(s')))} \, | \, s_h = s, \, a_h = a \right] \\
        &= \mathbb{E}\left[\bm{\delta}_h \tr{\bm{\delta}_h} +  \bm{\xi}_{h+1}^\pi(s',\pi_{h+1}(s')) \tr{\bm{\xi}_{h+1}^\pi(s',\pi_{h+1}(s'))} \, | \, s_h = s, \, a_h = a \right] \\
        &\phantom{=}+ \mathbb{E}\left[\bm{\delta}_h \tr{\bm{\xi}_{h+1}^\pi(s',\pi_{h+1}(s'))} \, | \, s_h = s, \, a_h = a \right] \\
        &\phantom{=}+ \mathbb{E}\left[\bm{\xi}_{h+1}^\pi(s',\pi_{h+1}(s')) \tr{\bm{\delta}_h} \, | \, s_h = s, \, a_h = a \right] \\
        &= \mathbb{E}\left[\bm{\delta}_h \tr{\bm{\delta}_h} +  \bm{\xi}_{h+1}^\pi(s',\pi_{h+1}(s')) \tr{\bm{\xi}_{h+1}^\pi(s',\pi_{h+1}(s'))} \, | \, s_h = s, \, a_h = a \right] \\
        &= \mathbb{E}\left[\bm{\delta}_h \tr{\bm{\delta}_h} +  \Sigma_{h+1}^\pi(s', \pi_{h+1}(s')) \, | \, s_h = s, \, a_h = a \right],
    \end{align*}
    where the second-last line follows from Lemma \ref{lem:bellman_zero}.
\end{proof}

\begin{customtheorem}{3}
    Let $\|\cdot\|$ be a matrix-compatible norm, and suppose there exists $\varepsilon : \mathcal{S} \times \mathcal{A} \times \mathcal{T} \to [0, \infty)$ such that $\|\tilde{\bm{\psi}}_h^\pi(s,a) - {\bm{\psi}}_h^\pi(s,a) \|^2 \leq \varepsilon_h(s,a)$ and $\|\mathbb{E}_{s'\sim P(\cdot | s, a)}[\tilde{\bm{\delta}}_h \tr{(\tilde{\bm{\psi}}_h^\pi(s',\pi_{h+1}(s')) - {\bm{\psi}}_h^\pi(s',\pi_{h+1}(s')))}]\| \leq \varepsilon_h(s,a)$. Then,
    \begin{equation*}
        \left\| {\Sigma}_h^\pi(s,a) - \mathbb{E}_{s' \sim P(\cdot |s,a)}\left[\tilde{\bm{\delta}}_h \tr{\tilde{\bm{\delta}}_h} +  \tilde{\Sigma}_{h+1}^\pi(s',\pi_{h+1}(s')) \right] \right\| \leq 3 \varepsilon_h(s,a).
    \end{equation*}
\end{customtheorem}

\begin{proof}
    We start by decomposing the true covariance matrix:
    \begin{align*}
        \Sigma_h^\pi(s,a)
        &= \mathbb{E}\Big[(\Psi_h^\pi(s,a) - \tilde{\bm{\psi}}_h^\pi(s,a) + \bm{\psi}_h^\pi(s,a) - \bm{\psi}_h^\pi(s,a))\\
        &\phantom{ }\indent\tr{(\Psi_h^\pi(s,a) - \tilde{\bm{\psi}}_h^\pi(s,a) + \bm{\psi}_h^\pi(s,a) - \bm{\psi}_h^\pi(s,a))} \,|\, s_h = s,\, a_h = a\Big] \\
        &= \mathbb{E}\left[(\Psi_h^\pi(s,a)-\tilde{\bm{\psi}}_h^\pi(s,a))\tr{(\Psi_h^\pi(s,a)-\tilde{\bm{\psi}}_h^\pi(s,a))} \,|\, s_h = s,\, a_h = a \right] \\
        &\phantom{=} + (\tilde{\bm{\psi}}_h^\pi(s,a) - {\bm{\psi}}_h^\pi(s,a))\tr{(\tilde{\bm{\psi}}_h^\pi(s,a) - {\bm{\psi}}_h^\pi(s,a))} \\
        &\phantom{=} + 2 \mathbb{E}\left[\Psi_h^\pi(s,a) - \tilde{\bm{\psi}}_h^\pi(s,a) \,|\, s_h = s,\, a_h = a \right] \tr{(\tilde{\bm{\psi}}_h^\pi(s,a) - {\bm{\psi}}_h^\pi(s,a))} \\
        &= \mathbb{E}\left[(\Psi_h^\pi(s,a)-\tilde{\bm{\psi}}_h^\pi(s,a))\tr{(\Psi_h^\pi(s,a)-\tilde{\bm{\psi}}_h^\pi(s,a))} \,|\, s_h = s,\, a_h = a \right] \\
        &\phantom{=} - (\tilde{\bm{\psi}}_h^\pi(s,a) - {\bm{\psi}}_h^\pi(s,a))\tr{(\tilde{\bm{\psi}}_h^\pi(s,a) - {\bm{\psi}}_h^\pi(s,a))} 
    \end{align*}
    where in the last step we use the identity $\mathbb{E}\left[\Psi_h^\pi(s,a) - \tilde{\bm{\psi}}_h^\pi(s,a) \,|\, s_h = s,\, a_h = a \right] = \mathbb{E}\left[\Psi_h^\pi(s,a) - {\bm{\psi}}_h^\pi(s,a) \,|\, s_h = s,\, a_h = a \right] + \bm{\psi}_h^\pi(s,a) - \tilde{\bm{\psi}}_h^\pi(s,a) = \bm{\psi}_h^\pi(s,a) - \tilde{\bm{\psi}}_h^\pi(s,a)$. Now, we define $\tilde{\bm{\xi}}_h^\pi(s,a) = \Psi_h^\pi(s,a)-\tilde{\bm{\psi}}_h^\pi(s,a)$, then follow the derivations in Lemma \ref{lem:covariance_bellman} to write the first term above as:
    \begin{align*}
        &\mathbb{E}\left[\tilde{\bm{\xi}}_h^\pi(s,a)\tr{\tilde{\bm{\xi}}_h^\pi(s,a)} \,|\, s_h = s,\, a_h = a \right] \\
        &= \mathbb{E}\left[\tilde{\bm{\delta}}_h \tr{\tilde{\bm{\delta}}_h} +  \tilde{\bm{\xi}}_{h+1}^\pi(s',\pi_{h+1}(s')) \tr{\tilde{\bm{\xi}}_{h+1}^\pi(s',\pi_{h+1}(s'))} \, | \, s_h = s, \, a_h = a \right] \\
        &\phantom{=}+ \mathbb{E}\left[ \tilde{\bm{\delta}}_h \tr{\tilde{\bm{\xi}}_{h+1}^\pi(s',\pi_{h+1}(s'))} \, | \, s_h = s, \, a_h = a \right] + \mathbb{E}\left[ \tilde{\bm{\xi}}_{h+1}^\pi(s',\pi_{h+1}(s')) \tr{\tilde{\bm{\delta}}_h} \, | \, s_h = s, \, a_h = a \right] \\
        &= \mathbb{E}\left[\tilde{\bm{\delta}}_h \tr{\tilde{\bm{\delta}}_h} +  \tilde{\Sigma}_{h+1}^\pi(s', \pi_{h+1}(s')) \, | \, s_h = s, \, a_h = a \right] \\
        &\phantom{=}+ \mathbb{E}\left[ \tilde{\bm{\delta}}_h \tr{\tilde{\bm{\xi}}_{h+1}^\pi(s',\pi_{h+1}(s'))} \, | \, s_h = s, \, a_h = a \right] + \mathbb{E}\left[ \tilde{\bm{\xi}}_{h+1}^\pi(s',\pi_{h+1}(s')) \tr{\tilde{\bm{\delta}}_h} \, | \, s_h = s, \, a_h = a \right].
    \end{align*}
    Finally, we norm bound the desired difference as follows:
    \begin{align*}
        &\left\|\Sigma_h^\pi(s,a) - \mathbb{E}\left[\tilde{\bm{\delta}}_h \tr{\tilde{\bm{\delta}}_h} +  \tilde{\Sigma}_{h+1}^\pi(s', \pi_{h+1}(s')) \, | \, s_h = s, \, a_h = a \right]\right\| \\
        &\leq 2 \left\|\mathbb{E}\left[ \tilde{\bm{\delta}}_h \tr{\tilde{\bm{\xi}}_{h+1}^\pi(s',\pi_{h+1}(s'))} \, | \, s_h = s, \, a_h = a \right]\right\| \\
        &\phantom{=} + \left\| (\tilde{\bm{\psi}}_h^\pi(s,a) - {\bm{\psi}}_h^\pi(s,a))\tr{(\tilde{\bm{\psi}}_h^\pi(s,a) - {\bm{\psi}}_h^\pi(s,a))} \right\| \\
        &\leq 2 \left\|\mathbb{E}\left[ \tilde{\bm{\delta}}_h \tr{(\Psi_{h+1}^\pi(s', \pi_{h+1}(s')) - \bm{\psi}_{h+1}^\pi(s', \pi_{h+1}(s')))} \, | \, s_h = s, \, a_h = a \right]\right\| \\
        &\phantom{=} + 2 \left\|\mathbb{E}\left[ \tilde{\bm{\delta}}_h \tr{(\bm{\psi}_{h+1}^\pi(s', \pi_{h+1}(s')) - \tilde{\bm{\psi}}_{h+1}^\pi(s', \pi_{h+1}(s')))} \, | \, s_h = s, \, a_h = a \right]\right\| \\
        &\phantom{=} + \left\| (\tilde{\bm{\psi}}_h^\pi(s,a) - {\bm{\psi}}_h^\pi(s,a))\tr{(\tilde{\bm{\psi}}_h^\pi(s,a) - {\bm{\psi}}_h^\pi(s,a))} \right\| \\
        &\leq 2 \left\|\mathbb{E}\left[ \tilde{\bm{\delta}}_h \tr{(\bm{\psi}_{h+1}^\pi(s', \pi_{h+1}(s')) - \tilde{\bm{\psi}}_{h+1}^\pi(s', \pi_{h+1}(s')))} \, | \, s_h = s, \, a_h = a \right]\right\| \\
        &\phantom{=} + \left\| \tilde{\bm{\psi}}_h^\pi(s,a) - {\bm{\psi}}_h^\pi(s,a)) \right\|^2 \\
        &\leq 2 \varepsilon_h(s,a) + \varepsilon_h(s,a) = 3 \varepsilon_h(s,a).
    \end{align*}
    This is the desired result.
\end{proof}

\section{Experiment Details}
\label{sec:experimental_detail}

In this section, we describe the setup of the domains discussed in the main paper in greater detail. We also provide detailed descriptions of baseline algorithms, as well as all hyper-parameters used and how they were selected.

\subsection{Motivating Example}

The motivating example is a 5-by-5 grid-world domain with discrete states and discrete actions described by the four possible directions of movement into an adjacent cell. The environment is made stochastic by introducing random action noise as follows. Desired actions are taken only with probability $0.8$, while the remaining time a random action is taken. Furthermore, transitions that would take the agent outside of the boundaries of the grid leave the agent in its current position. The cost structure is defined as follows. The goal state is terminal and provides a reward of $+20$. Each time step incurs a fixed penalty of $-1$, on top of any other rewards or costs incurred.

To recover the properties of risk-aware and risk-neutral GPI claimed in the main text, we first learn the source policies $\pi_1$ and $\pi_2$ and their utilities $\mathcal{Q}_\beta^{\pi_1}$ and $\mathcal{Q}_\beta^{\pi_2}$ using a variant of the classic value iteration algorithm adapted to maximize the entropic utility (see Algorithm \ref{alg:value_iteration}). We consider the non-discounted setting ($\gamma = 1$), and iterate until an absolute error less than $\varepsilon_{exit} = 10^{-12}$ is achieved between two consecutive iterations\footnote{Please note that convergence of value iteration is guaranteed due to the existence of absorbing states and because the underlying MDP is ergodic.}. The two source policies are then recovered by acting greedily with respect to the learned utilities. 

In order to implement GPI, we evaluate these two resulting policies on the target task by adapting the iterative procedure in Algorithm \ref{alg:value_iteration} for \emph{policy evaluation}. Essentially, line 10 of the algorithm is replaced by $r(s,a,s') + \gamma \mathcal{Q}(s',\pi(s'))$ for $\pi \in \lbrace \pi_1, \pi_2\rbrace$. We repeat this procedure twice to produce two sets of value functions: a set $\lbrace \mathcal{Q}_0^{\pi_1}, \mathcal{Q}_0^{\pi_2} \rbrace$ for $\beta = 0$\footnote{For $\beta = 0$, Algorithm \ref{alg:value_iteration} reduces to standard value iteration.} and a set $\lbrace \mathcal{Q}_{-0.1}^{\pi_1}, \mathcal{Q}_{-0.1}^{\pi_2} \rbrace$ for $\beta = -0.1$. The two GPI policies are then defined as $\pi_{\beta}(s) \in \argmax_{a} \max_{i=1,2} \mathcal{Q}_\beta^{\pi_i}(s,a)$ for $\beta \in [0, -0.1]$. Finally, we generate the histogram of returns by simulating episodes of length $T = 35$, throughout which actions are selected from $\pi_\beta$, and computing the cumulative reward obtained on each.

\begin{algorithm}[!tb]
  \caption{Value Iteration for Entropic Utility Maximization}
  \label{alg:value_iteration}
\begin{algorithmic}[1]
    \State {\bfseries Requires} $\varepsilon_{exit} > 0$, $\gamma \in [0,1]$, $\beta \in \mathbb{R}$, $\langle \mathcal{S}, \mathcal{A}, r, P \rangle \in \mathcal{M}$
    \State {\bfseries for} $s \in \mathcal{S}, \, a \in \mathcal{A}$ {\bfseries do} $\mathcal{Q}(s,a) \gets 0$
    \For{$n=1, 2\dots \infty$}
        \IndentLineComment{Update $\mathcal{Q}(s,a)$ for all state-action pairs}
        \For{$s \in \mathcal{S}, \, a \in \mathcal{A}$}
            \IndentLineComment{Perform one iteration of (\ref{eqn:entropic_bellman_discounted}) with the greedy policy derived from $\mathcal{Q}$}
            \State $\mathcal{Q}'(s,a) \gets 0$
            \For{$s' \in \mathcal{S}$}
                \If{$s'$ is terminal}
                    \State $\mathrm{target} \gets r(s,a,s')$
                \Else
                     \State $\mathrm{target} \gets r(s,a,s') + \gamma \max_b \mathcal{Q}(s',b)$
                \EndIf
                \State $\mathcal{Q}'(s,a) \gets \mathcal{Q}'(s,a) + P(s' | s, a) \, e^{\beta \times \mathrm{target}}$
            \EndFor
            \State $\mathcal{Q}'(s,a) \gets \frac{1}{\beta} \log \mathcal{Q}'(s,a)$
        \EndFor
        \LineComment{Check for convergence in utility values}
        \State $\varepsilon \gets \max_{s,a} | \mathcal{Q}'(s,a) - \mathcal{Q}(s,a) |$
        \State {\bfseries if} $\varepsilon < \varepsilon_{exit}$ {\bfseries then return} $\mathcal{Q}'$
        \LineComment{If not converged, then continue with value iteration}
        \State {\bfseries for} $s \in \mathcal{S}, \, a \in \mathcal{A}$ {\bfseries do} $\mathcal{Q}(s,a) \gets \mathcal{Q}'(s,a)$
    \EndFor
\end{algorithmic}
\end{algorithm}

\subsection{Four-Room}
\label{subsec:shapes_experiment}

The four-room domain consists of a family of discrete-state discrete-action MDPs $\mathcal{M}$ defined as follows. The world is defined as a set of discrete cells arranged in a grid of dimensions 13-by-13, such that at each time instant, the agent occupies a specific cell with some $x$- and $y$-coordinates $(p_x, p_y) \in \lbrace 0, \dots 12 \rbrace^2$. 

As the agent explores the space, it can collect objects belonging to one of 3 possible classes. While the initial positions of these objects remains fixed throughout the experiment, their existence is determined by whether or not they have already been collected by the agent in a given episode (the same object cannot be picked up multiple times in a given episode). In our configuration, there are 6 instances of objects belonging to each class, for a total of $n_o = 18$ collectible objects. Therefore, the state space $\mathcal{S} = \lbrace 0, 1 \rbrace^{n_o} \times \lbrace 0, \dots 12 \rbrace^2$ consists of the concatenation of the agent's current position $(p_x, p_y)$ and a set of binary variables indicating whether or not each object has already been picked up by the agent. All objects are reset at the beginning of each episode. Actions are defined as $\mathcal{A} = \lbrace \mathrm{left}, \mathrm{up}, \mathrm{right}, \mathrm{down} \rbrace$ that move the agent to an adjacent cell in the corresponding direction. In the case that the destination cell lies outside the grid, then the agent remains in the current cell at the next time instant. 

The goal cell `G' provides a fixed reward of $+1$ and immediately terminates the episode upon entry. The reward $r_c$ associated with each object class $c \in \lbrace 1, 2, 3 \rbrace$ is reset every time a new task begins, and is sampled from a uniform distribution on $[-1, +1]$. Occupying a trap cell that triggers at a particular time instant defines a failure, and is communicated to the agent by incurring a penalty of $-2$ and immediately terminating the episode. However, occupying a trap cell does not automatically guarantee a failure. Instead, a failure is only triggered with probability $0.05$ independently at every time instant during which the agent occupies a trap cell. This additional reward stochasticity can be implemented without breaking the existing successor feature framework by introducing a fictitious terminal state $s_f$ to indicate failure, which is reached at random when in cells marked `X'. This state augmentation induces a modified MDP with a deterministic reward of $-2$ on arrival to state $s_f$, whose associated transitions are stochastic in nature. Crucially, this state augmentation transformation applies uniformly to all task instances, and thus does not break our assumptions about $\mathcal{M}$. We use a discount factor of $\gamma = 0.95$.

Exact state features $\bm{\phi}(s,a,s')$ are provided directly to the agent. Specifically, we define $\phi_c(s,a,s')$ for every class of objects $c \in \lbrace 1, 2, 3 \rbrace$ to take the value 1 if the agent occupies a cell with an object of class $c$ in state $s'$ and 0 otherwise. Similarly, we define $\phi_g(s,a,s')$ to take the value 1 if $s'$ corresponds to the goal cell and 0 otherwise. Unlike \citet{barreto2017successor}, the four-room domain also contains an additional failure state with non-zero reward, as described above, and this must also be incorporated into the SF representation. This can be done by defining $\phi_f(s,a,s')$ that takes the value 1 if $s'$ corresponds to the state $s_f$ and 0 otherwise\footnote{It is not practical to redefine the task space with the augmented state $s_f$ in an actual implementation. Instead, we simulate this by providing the state features $\bm{\phi}$ with a binary variable indicating failure. This does not change the SF implementation, since the occurrence of a failure event can be deduced using the $\mathrm{done}$ flag (indicating arrival in a terminal state) and the state $s'$.}. The state features $\bm{\phi} \in \mathbb{R}^5$ are then the concatenation of $\phi_c$, $\phi_g$ and $\phi_f$. These features are sparse, but can represent the reward functions of all possible task instances in $\mathcal{M}$ exactly. Finally, we define $\mathbf{w}_c = r_c$, $\mathbf{w}_g = 1$ and $\mathbf{w}_f = -2$, and it is now clear that $r(s,a,s') = \tr{{\bm{\phi}}(s,a,s')} \mathbf{w}$ holds. 

Each time a new task is created, a new $\tilde{\bm{\psi}}^\pi$ and $\tilde{\Sigma}^\pi$ are created. The training loop of RaSFQL then proceeds according to Algorithm \ref{alg:mvsfql}. We set $\alpha = 0.5$ and $\varepsilon = 0.12$, based on preliminary experiments for Q-learning. We also set $\bar{\alpha} = 0.1$ for learning $\tilde{\Sigma}^\pi$ and $\alpha_w = 0.5$ for learning $\mathbf{w}$ with gradient descent. Rollouts are limited to $T = 200$ steps for all algorithms. 

The baseline used for comparison is the \emph{probabilistic policy reuse} framework of \citet{fernandez2006probabilistic} (PRQL), here adapted for learning risk-sensitive behaviors. In order to do this, we incorporate the \emph{smart exploration} strategy of \citet{gehring2013smart}. This strategy is fundamentally similar to our mean-variance approach, since it also incorporates second-moment or reward-variance information into action selection in a similar way. The controllability bonus $C^\pi(s,a)$ in each state-action pair is learned using a Q-learning approach by using the negative of the absolute Bellman residuals $-|\delta|$ as pseudo-rewards, and learned in parallel to the Q-values in practical implementations. The penalty for $C(s,a)$ is denoted as $\omega$, and is fundamentally similar to $\beta$ used by SFQL. The resulting algorithm, which we call RaPRQL, is described in detail in Algorithm \ref{alg:smart}. 

\begin{algorithm}[!tb]
  \caption{RaPRQL with Smart Exploration}
  \label{alg:smart}
\begin{algorithmic}[1]
    \State {\bfseries Requires} $m, T, N_e \in \mathbb{N}, \,  \varepsilon, \eta \in [0, 1], \, \alpha, \rho, \tau > 0, \, \omega \in \mathbb{R}$, $M_1, \dots M_m \in \mathcal{M}$
    
  \For{$t=1, 2\dots m$}
        \State Initialize $Q^t(s,a), C^t(s,a)$ to small random or zero values
        \State {\bf{for}} {$k = 1, 2 \dots t$} {\bf{do}} $\mathrm{score}_k \gets 0$, $\mathrm{used}_k \gets 0$
        \State $c \gets t$
        \State $R \gets 0$
        \LineComment{Commence training on task $M_t$}
        \For{$n_e = 1, 2 \dots N_e$}
        \State Initialize $M_t$ with initial state $s$
        \For{$h=0,1\dots T$} 
            \IndentLineComment{Select actions according to Q-values plus controllability bonus}
            \State {\bf{if}} $c \not= t$ {\bf{then}} $\mathrm{use\_prev\_policy} \sim \mathrm{Bernoulli}(\eta)$ {\bf{else}} $\mathrm{use\_prev\_policy} \gets \mathrm{false}$
            \If{$\mathrm{use\_prev\_policy}$} 
                \IndentLineComment{Action is selected from $\pi_c$, the source policy being used}
                \State $a \gets \argmax_b \lbrace Q_h^c(s,b) + \omega C_h^c(s,b) \rbrace$
            \Else
                \IndentLineComment{Action is selected from $\pi_t$, the policy being learned}
                \State $\mathrm{random\_a} \sim \mathrm{Bernoulli}(\varepsilon)$
                \State \parbox[t]{\dimexpr\textwidth-\leftmargin-\labelsep-\labelwidth}{%
                \bf{if} $\mathrm{random\_a}$ \bf{then} $a \sim \mathrm{Uniform}(\mathcal{A})$ \bf{else} $a \gets \argmax_b \lbrace Q_h^t(s,b) + \omega C_h^t(s,b) \rbrace$ \strut}    
            \EndIf
            \State Take action $a$ in $M_t$ and observe $r$ and $s'$
            \LineComment{Update the Q-values for the current task}
            \State $\delta_h \gets r + \max_b Q_{h+1}^t(s',b) - Q_h^t(s,a)$
            \State $Q_h^t(s,a) \gets Q_h^t(s,a) + \alpha \delta_h$
            \LineComment{Update the controllability bonus for the current task}
            \State $C_h^t(s,a) \gets C_h^t(s,a) + \alpha \rho (-|\delta_h| - C_h^t(s,a))$
            \State $R \gets R + r$
            \State $s \gets s'$
        \EndFor
        \LineComment{Update average return obtained by following policy $\pi_c$}
        \State $\mathrm{score}_c \gets \frac{\mathrm{score}_c \times \mathrm{used}_c + R}{\mathrm{used}_c + 1}$
        \LineComment{Sample a new source policy}
        \State {\bf{for}} $k = 1, 2 \dots t$ {\bf{do}} $p_k \gets \frac{e^{\tau \times \mathrm{score}_k}}{\sum_j e^{\tau \times \mathrm{score}_j}}$
        \State $c \sim \mathrm{Multinomial}(p_1, p_2, \dots p_t)$
        \State $\mathrm{used}_c \gets \mathrm{used}_c + 1$
        \State $R \gets 0$
        \EndFor
  \EndFor
\end{algorithmic}
\end{algorithm}

Similar to RaSFQL, every time a new task is created, a new $Q^\pi$ and $C^\pi$ are created for RaPRQL for learning new policies. We set $\alpha = 0.5$ for fair comparison with RaSFQL, and $\rho = 0.1$ based on the original implementation \citep{gehring2013smart}. The performance is highly sensitive to the parameters $\eta$ and $\tau$ used by PRQL. To select these two hyper-parameters, we follow \citet{barreto2017successor} and run a grid search for $\eta \in \lbrace 0.1, 0.3, 0.5 \rbrace$ and $\tau \in \lbrace 1, 10, 100 \rbrace$, selecting the combination of $\eta$ and $\tau$ that resulted in the highest cumulative return over 128 task instances. This validation experiment is repeated for every value of $\omega$. 

\subsection{Reacher}
\label{subsec:reacher_experiment}

The physics simulator used for the reacher domain is provided by the open-source \texttt{pybullet} and \texttt{pybullet-gym} packages \citepsupp{coumans2021,benelot2018}. We adapted the Python environment in the latter package to handle multiple target goal locations as required in our problem setting. Please note that this package is released under the MIT license.

The state space $\mathcal{S} \subset \mathbb{R}^4$ consists of the angles and angular velocities of the robotic arm's two joints. The two-dimensional action space $\mathcal{A} \subset [-1, +1]^2$ is discretized using 3 values per dimension, corresponding to maximum positive ($+1$) and negative ($-1$) and zero torque for each actuator, resulting in a total of 9 possible actions. At the beginning of each episode, the angle of the central joint is sampled from a uniform distribution on $[-\pi, +\pi]$, while the angle of the outer joint is sampled from a uniform distribution on $[-\pi/2, +\pi/2]$, and the angular velocities are initialized to zero. Furthermore, state transitions are made stochastic by adding zero-mean Gaussian noise to actions with standard deviation $0.03$, and then clipping the actions to $[-1, +1]$. 

The reward received at each time step is $1 - 4\delta$, where $\delta$ is the Euclidean distance between the target position and the tip of the robotic arm. We define 12 target locations, of which 4 are used for training and the remaining 8 for testing. Furthermore, circular regions of radius $\delta_f = 0.06$ are placed around 6 of the 12 target locations (2 training and 4 testing) in which failures occur spontaneously with probability $p_f = 0.035$. Once a failure occurs, a cost of $c_f = 3$ is incurred and the episode continues without termination. This implies that the expected reward, as a function of the distance $\delta$, is\footnote{The reasoning here has simplified some of the aspects of the environment, ignoring the effects of multiple risk regions that could alter the trajectories, limited-length episodes and discounting.}
\begin{equation*}
   R(\delta) = 
   \begin{cases}
        1 - 4 \delta &\mbox{ if } \delta > \delta_f \\
        1 - 4 \delta - c_f \times p_f &\mbox { if } \delta \leq \delta_f.
   \end{cases}
\end{equation*}
Therefore, a rational\footnote{Of course, a rational agent would want to keep the tip as close to the target location as possible, and so would want $\delta = 0$.} risk-neutral agent would prefer to enter inside the failure region if it holds that $1 - c_f \times p_f \geq 1 - 4 \delta_f$, or in other words if
\begin{equation*}
    c_f \times p_f \leq 4 \delta_f.
\end{equation*}
Clearly, given our choice of values for $c_f, p_f$ and $\delta_f$, the above condition holds in our setting. Setting up the reward structure and risk in this way makes it possible to control the trade-off between risk and reward, and thus the anticipated behavior of the agents, in a principled way. We also apply discounting of future reward using $\gamma = 0.9$. 

The state features are vectors $\bm{\phi}(s,a,s') \in \mathbb{R}^{13}$, in which the first 12 components consist of $1 - 4 \delta_g$, where $\delta_g$ are the Euclidean distances to each of the goal locations $g$. The last component takes the value 1 if a failure event occurs and 0 otherwise. As done in the four-room experiment, state features are provided to the agent. However, target goal locations $\mathbf{w}\in \mathbb{R}^{13}$ are not learned in this instance, but provided directly to the agent as well. Specifically, we set $\mathbf{w}_g = 1$ for the goal with index $g$ and $\mathbf{w}_{13} = c_f = -3$, and set all other elements to zero. This recovers the correct reward function $r(s,a,s')$ for all task instances as described above.

The overall training and testing procedures closely mimic \citet{barreto2017successor}. The successor features $\bm{\psi}^\pi$ and their distribution $\bm{\Psi}^\pi$ are represented as multi-layer perceptrons (MLP) with two hidden layers of size 256 and $\mathrm{tanh}$ non-linearities. The SFC51 and RaSFC51 architectures are generally identical and require output layers of dimensions $\mathbb{R}^{9\times 51 \times 13}$, with a $\mathrm{softmax}$ activation function applied with respect to the second dimension. Similarly, C51 and RaC51 also require output layers but of dimensions $\mathbb{R}^{9 \times 51}$ and $\mathrm{softmax}$ applied with respect to the second dimension. For SFDQN, the output of the network is linear with dimensions $\mathbb{R}^{9\times 13}$. We also use target networks for both SFC51/RaSFC51 and C51/RaC51, which are updated every $1,\!000$ transitions by copying weights from the learning networks. For SFC51, RaSFC51 and SFDQN, separate MLPs are used to learn each policy. To allow C51 and RaC51 to generalize across target locations, we apply \emph{universal value function approximation} \citep{schaul2015universal} and incorporate the target position into the state. This makes C51 essentially identical to the DQN baseline in \citet{barreto2017successor}, except that DQN is replaced by C51. For C51-based agents, recall that the range of possible values of $\bm{\phi}$ must also be specified. For SFC51 and RaSFC51, we use $\phi_{min}^d = -1$ and $\phi_{max}^d = 1$ for $d = 1, 2 \dots 12$ and use $\phi_{min}^{13} = 0$ and $\phi_{max}^{13} = 1$, which corresponds to a relatively tight bound for state features described in the previous paragraph. For C51 and RaC51, we set the bounds to $V_{min} = -30$ and $V_{max} = 10$, which corresponds to a tight bound for the discounted return. These intervals are discretized into $N = 51$ atoms for learning histograms, as recommended in the original paper \citep{bellemare2017distributional}. 

Agents are trained on all 4 training task instances sequentially one at a time, for $200,\!000$ time steps per task using an epsilon-greedy policy with $\varepsilon = 0.1$. Analogous to \citet{barreto2017successor}, every sample is used to train all 4 policies simultaneously for SFC51, RaSFC51 and SFDQN. A randomized replay buffer of infinite capacity stores all previously-observed transitions $(s,a,\bm{\phi}, s')$ from all 4 training tasks, to avoid ``catastrophic forgetting" of previously learned task instances. Each update of the network is based on a mini-batch of size 32 sampled uniformly from the replay buffer, and uses the Adam optimizer with a learning rate of $10^{-3}$.  Please note that these parameters, and those in the previous paragraph, are generally identical to those used in \citet{barreto2017successor}. Testing follows an epsilon-greedy policy with $\varepsilon = 0.03$ and greedy actions are selected according to risk-aware GPI, e.g. $a^* \in \argmax_a \max_{i \in \lbrace 1, \dots 4 \rbrace} \lbrace \tr{\bm{\tilde{\psi}}^{\pi_i}(s,a)}\mathbf{w}_j + \beta \tr{\mathbf{w}_j} \tilde{\Sigma}^{\pi_i}(s,a) \mathbf{w}_j \rbrace$. Recall that test rewards $\mathbf{w}_j$ are provided to the agent. We set the episode length to $T = 500$ time steps for training and testing. All visualizations are based on estimating the test return at regular intervals of $5,\!000$ time steps, calculated as the average performance of 5 independent rollouts.

Since the performance varies for different target locations, \citet{barreto2017successor} applies a normalization procedure to compare the performance between tasks in a fair manner. We apply the same procedure, by first training a standard C51 agent from scratch on each training and test task 10 times, and recording the average performance at the beginning and end of training, $\bar{G}_b$ and $\bar{G}_a$, respectively. The normalized return illustrated in all figures is then calculated as $G_n = (G - \bar{G}_b) / (\bar{G}_a - \bar{G}_b)$. 

\subsection{Additional Details for Reproducibility}
\label{subsec:machines}

\paragraph{Reproducing Four-Room.}
The four-room experiment was run on an Alienware m17 R3, whose software and hardware specifications are provided in Table \ref{table:laptop_config}. Please note that while this machine has a GPU and tensorflow installed, neither were used in this experiment.

\begin{table}[!htb]
\centering
    \begin{tabular}{ccc} \toprule
        Component & Description & Quantity \\ \midrule
        Operating System & Windows 10 Home & \\
        Python & 3.8.5 (Anaconda) & \\
        tensorflow & 2.3.1 & \\ \midrule
        System Memory & 32 GB &  \\
        Hard Disk & 953.9GB & 1 \\
        CPU & Intel i7-10875H @ 2.30GHz (turbo-boost @ 5.1GHz) & 1 \\
        GPU & Nvidia RTX 2080 Super 8GB & 1\\
        \bottomrule
    \end{tabular}
\caption{Software and hardware configuration used to run all experiments for the four-room domain.}
\label{table:laptop_config}
\end{table}

\paragraph{Reproducing Reacher.}
The reacher experiment was run on a Lenovo ThinkStation P920 workstation, whose software and hardware specifications are described in Table \ref{table:workstation_config}. 

\begin{table}[!htb]
\centering
    \begin{tabular}{ccc} \toprule
        Component & Description & Quantity \\ \midrule
        Operating System & Ubuntu 18.04 & \\
        Python & 3.8.5 & \\
        tensorflow & 2.4.0 & \\ \midrule
        System Memory & 187 GB &  \\
        Hard Disk & 953.9GB & 5 \\
        CPU & Intel Xeon Gold 6234 @ 3.30GHz (turbo-boost @ 4GHz) & 32 \\
        GPU & Nvidia Quadro RTX 8000 48GB & 2\\
        \bottomrule
    \end{tabular}
\caption{Software and hardware configuration used to run all experiments for the reacher domain.}
\label{table:workstation_config}
\end{table}

\paragraph{Other Factors.}
Please note that seeds were not fixed during the experiment but generated in each trial using Python's default seed generation algorithm. This allows us to average the performance of all algorithms over different seed values and initializations. No internal modifications to the Python environment nor to any of its installed packages were made. No effort to overclock the machines' CPUs or GPUs beyond their factory settings were made in order to decrease the overall computation time (see below).

\paragraph{Computation Time.} The majority of the computation time in running the experiment was allocated to the reacher domain, partially because of the size of the network architectures required to learn meaningful policies (2 hidden layers consisting of 256 neurons), and the number of samples required to draw meaningful conclusions for all baselines. The computation time is considerably greater for RaSFC51 (around 28-36 hours per trial) than it is for RaC51 (around 6-8 hours per trial), which is expected since the former must train 4 neural networks while the latter must train only one. This could potentially lead to negative environmental impacts if the model is to be deployed on complex problems in real-world settings. At the same time, the potential speed-ups demonstrated by RaSFC51 as compared to RaC51 could reduce the \emph{overall} training time considerably and offset the total energy requirement of learning policies with a satisfactory variance-adjusted return. Parallelization of the training loop could also be beneficial and provide significant time and cost savings in practice.

\section{Additional Ablation Studies and Plots}
\label{sec:additional_plots}

In this section, we include the full details and results of the ablation studies described in the main text, and additional analysis that had to be left out of the main paper due to space limitations. 

\subsection{Four-Room}
\label{subsec:shapes_ablation}

We can study the effect of $\beta$ on the return performance and risk-sensitivity of the learned behaviors by repeating the four-room experiment (Appendix \appendixref{subsec:shapes_experiment}{C.2}) for various values of $\beta$. In particular, we trained RaSFQL for $\beta \in \lbrace 0, -1/2, -1, -2, -4 \rbrace$ ($\omega$ for RaPRQL), and recorded the cumulative reward and number of failures across all 128 training task instances. The results of these experiments are summarized in Figure \ref{fig:shapes_ablation_beta}. We see that the performance of RaSFQL degrades gracefully as $\beta$ decreases (a relative drop in cumulative reward of approximately $25\%$ is observed when $\beta$ is decreased from $0$ to $-4$), while the corresponding degradation for RaPRQL is considerably more pronounced (a relative drop in cumulative reward of roughly $75\%$ is observed for an identical change in $\omega$). Meanwhile, the number of cumulative failures of RaSFQL is generally lower than RaPRQL for every pair of identical values of $\beta$ and $\omega$. In fact, for $\beta \in \lbrace -1, -2, -4\rbrace$, the cumulative numbers of failures are increasing at sub-linear rates, which implies that risk-avoidance behavior is becoming more prominent as the number of training task instances increases. 

\begin{figure}[!tb]
    \centering
    \includegraphics[width=0.499\linewidth]{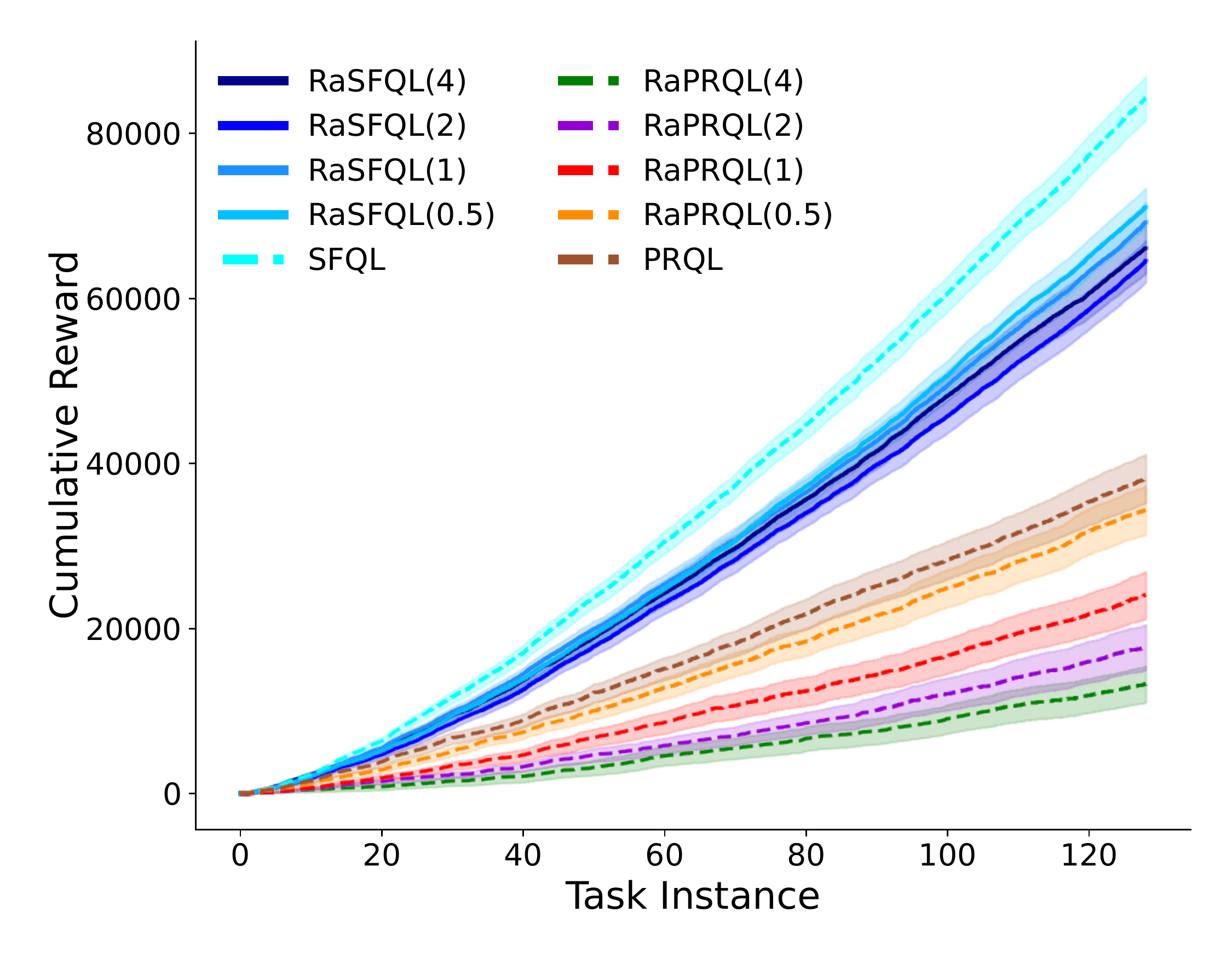}%
    \includegraphics[width=0.499\linewidth]{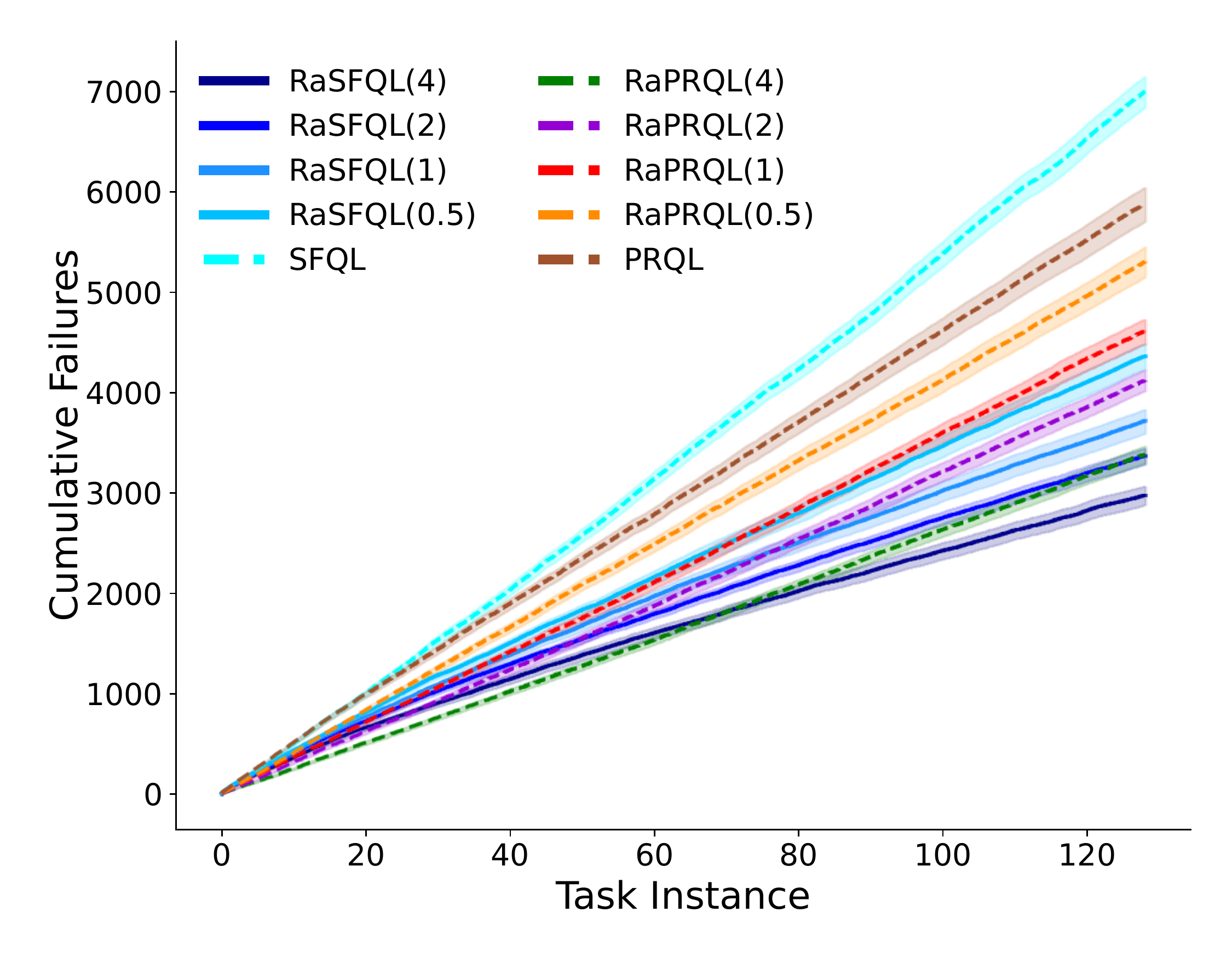}
    \caption{\textbf{Left:} cumulative reward collected across all training tasks in the four-room domain, for various values of $\beta$ for RaSFQL ($\omega$ for RaPRQL). \textbf{Right:} cumulative number of failures across all training tasks in the four-room domain, for various values of $\beta, \omega$. Please note that legend entries in parentheses indicate the negative values of $\beta$ and $\omega$. Shaded error bars indicate one standard error over 30 independent runs of each algorithm.}
    \label{fig:shapes_ablation_beta}
\end{figure}

In order to better understand the kind of risk-averse behaviors being learned, we instantiated 27 novel test task instances by enumerating $\mathbf{w}_i \in \lbrace -1, 0, 1 \rbrace$ for every object class $i = 1, 2, 3$. We then tested the performance of the GPI policy obtained from the training procedure described in the main paper, by simulating 100 rollouts following the epsilon-greedy policy with $\varepsilon = 0.1$ on each of the test tasks. Please note that no training was ever performed on the test tasks. The state visitation counts across all 100 trajectories were computed for every task instance and arranged in a 3D-lattice as indicated in Figure \ref{fig:shapes_rollouts}. We repeated this procedure twice: once for RaSFQL with $\beta = -2$ and once for SFQL. Interestingly, RaSFQL and SFQL learn behaviors that are similar to each other when looking at the same task, but each of them exploits different regions of the state space depending on the reward. However, RaSFQL almost always learns to avoid the dangerous objects in the bottom-left and top-right rooms, whereas SFQL does not necessarily do so. This discrepancy is most evident, for instance, when $\mathbf{w}_2 = 1$ and for $(\mathbf{w}_1, \mathbf{w}_2, \mathbf{w}_3) = (1, -1, -1)$ and $(\mathbf{w}_1, \mathbf{w}_2, \mathbf{w}_3) = (1, -1, 0)$. 

\begin{figure}[!tb]
    \centering
    \includegraphics[width=0.499\linewidth]{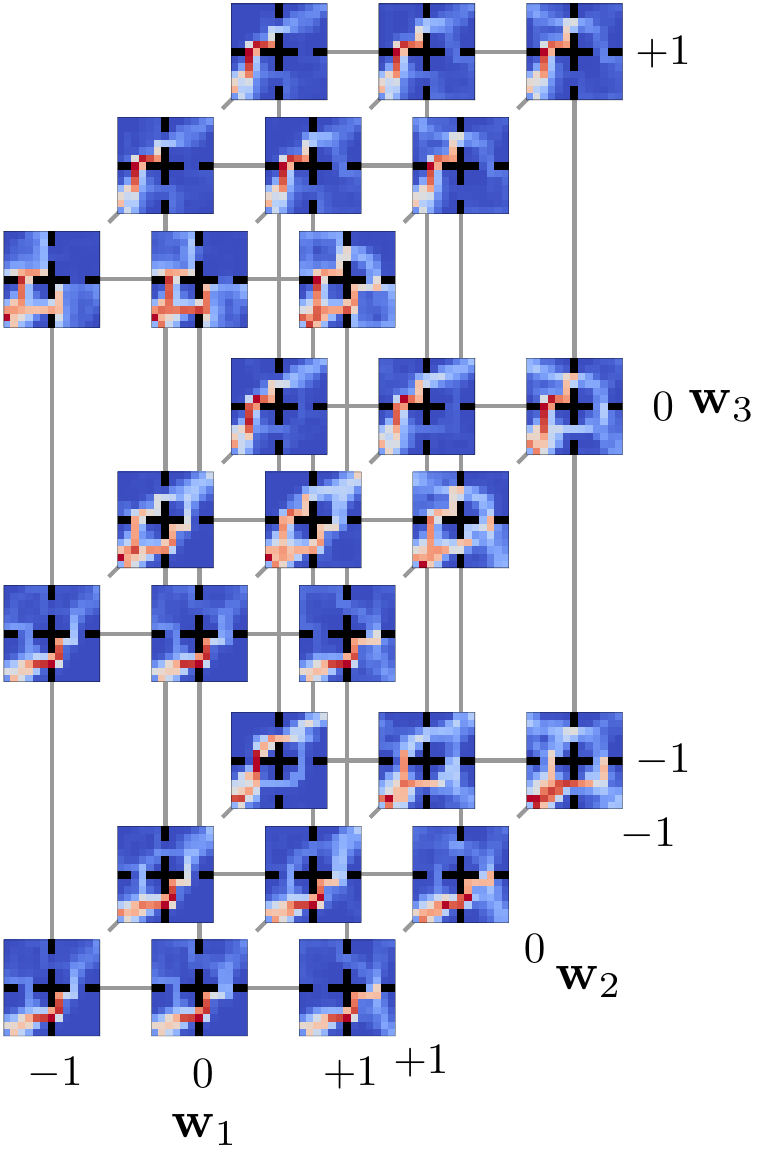}%
    \includegraphics[width=0.499\linewidth]{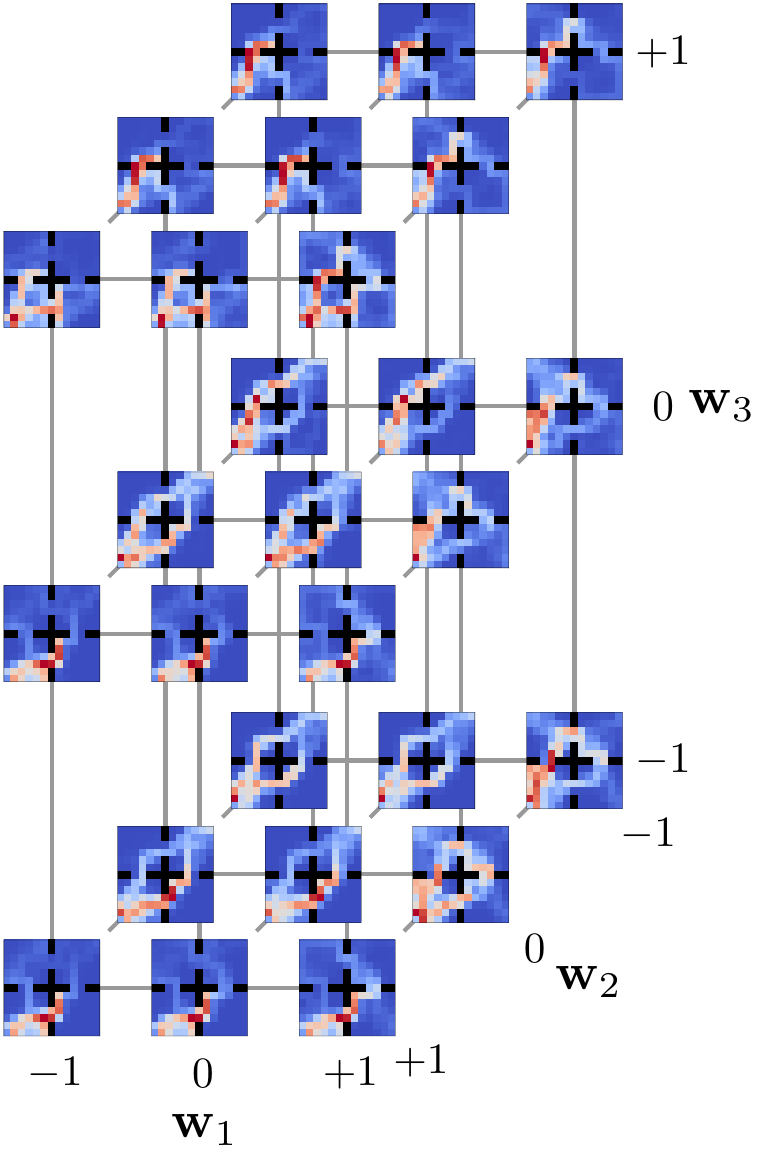}
    \caption{Visitation counts over 100 rollouts from behavior/training policies (epsilon-greedy with $\varepsilon=0.12$) derived from GPI after training on all 128 task instances. The behavior policies are illustrated on 27 novel task instances in which the reward $\mathbf{w}$ varies, e.g. $\mathbf{w}_1, \mathbf{w}_2, \mathbf{w}_3 \in \lbrace -1, 0, 1 \rbrace$. \textbf{Left:} Behavior policies derived from GPI for RaSFQL with $\beta = -2$. \textbf{Right:} Behavior policies derived from GPI for standard SFQL. Visitation counts are averaged over 30 independent runs for each algorithm.}
    \label{fig:shapes_rollouts}
\end{figure}

\subsection{Reacher}
\label{subsec:reacher_ablation}

For the reacher domain, we conducted a similar analysis of the risk-avoidance behaviors elicited by the entropic utility by repeating the experiment (Appendix \appendixref{subsec:reacher_experiment}{C.3}) for different values of $\beta$. The left plot in Figure \ref{fig:reacher_ablation_beta} illustrates the total number of failures across all test tasks for RaSFC51 and RaC51 for $\beta \in \lbrace 0, -1, -2, -3, -4 \rbrace$. The number of breakdowns decreases predictably as $\beta$ is decreased from zero. Unlike the four-room domain however, the number of failures of RaSFC51 is modestly \emph{greater} than RaC51 for the same values of $\beta$. In order to better understand how efficiently the trade-off between risk and reward is handled by these two approaches, we decided to compute an alternative measure of return by dividing the normalized return by the total number of failures. Intuitively, this quantity provides an estimate of the expected reward collected between between successive failure events. The right plot contained in Figure \ref{fig:reacher_ablation_beta}, which compares this quantity for RaSFC51 and RaC51 for different values of $\beta$, shows that RaSFC51 is actually much more efficient at managing the trade-off between risk and reward for larger-magnitude values of $\beta$. This is not surprising, given that RaSFC51 can obtain much high return than RaC51 for a comparable number of failures, when $\beta = -3$ and $\beta = -4$. In fact, for $\beta = -4$, the number of failures of RaSFC51 and RaC51 become equivalent as both methods learn sufficiently conservative policies. Even in this case, successor features combined with GPI allow RaSFC51 to generalize much better on novel tasks than RaC51.

\begin{figure}[!tb]
    \centering
    \includegraphics[width=0.499\linewidth]{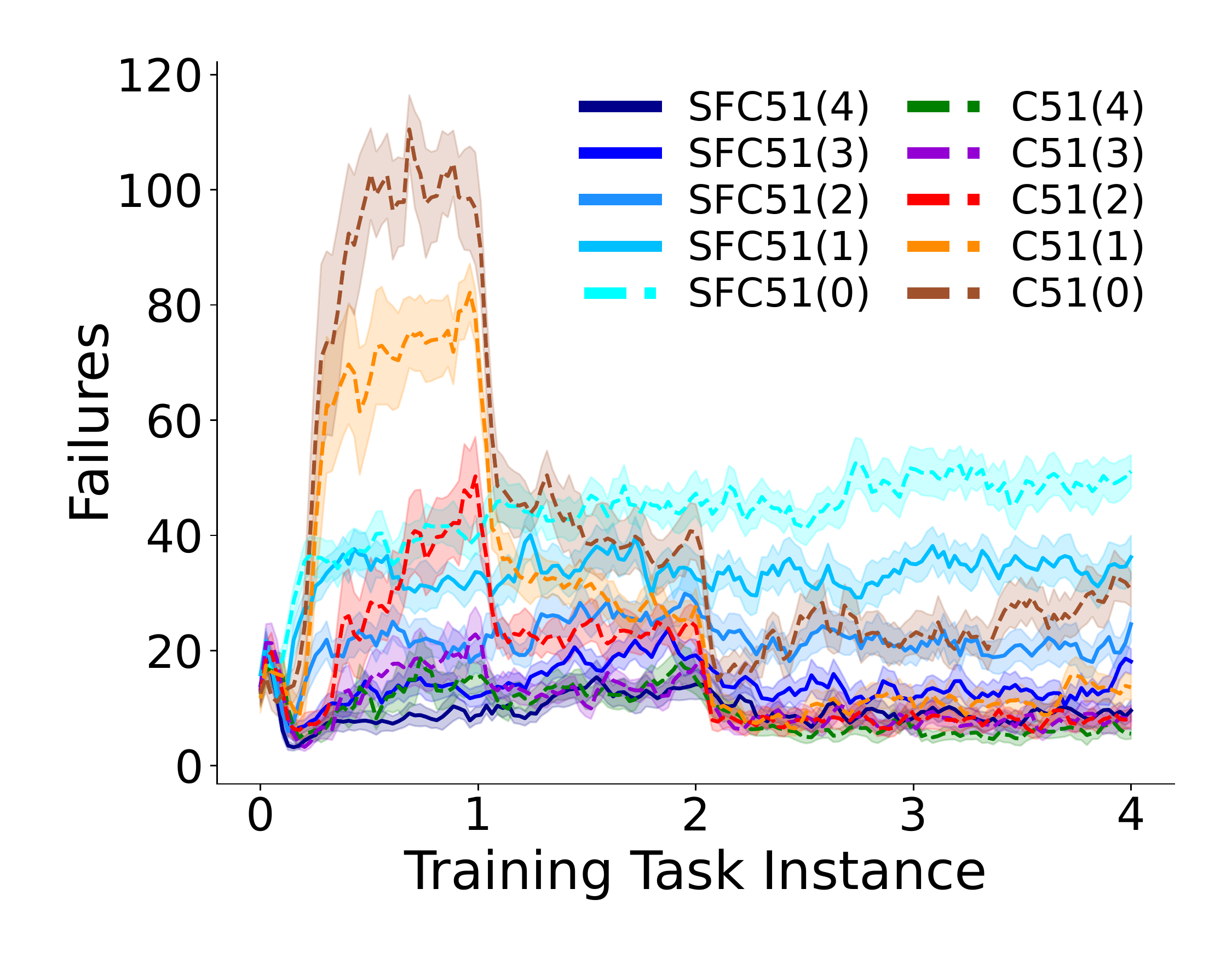}%
    \includegraphics[width=0.499\linewidth]{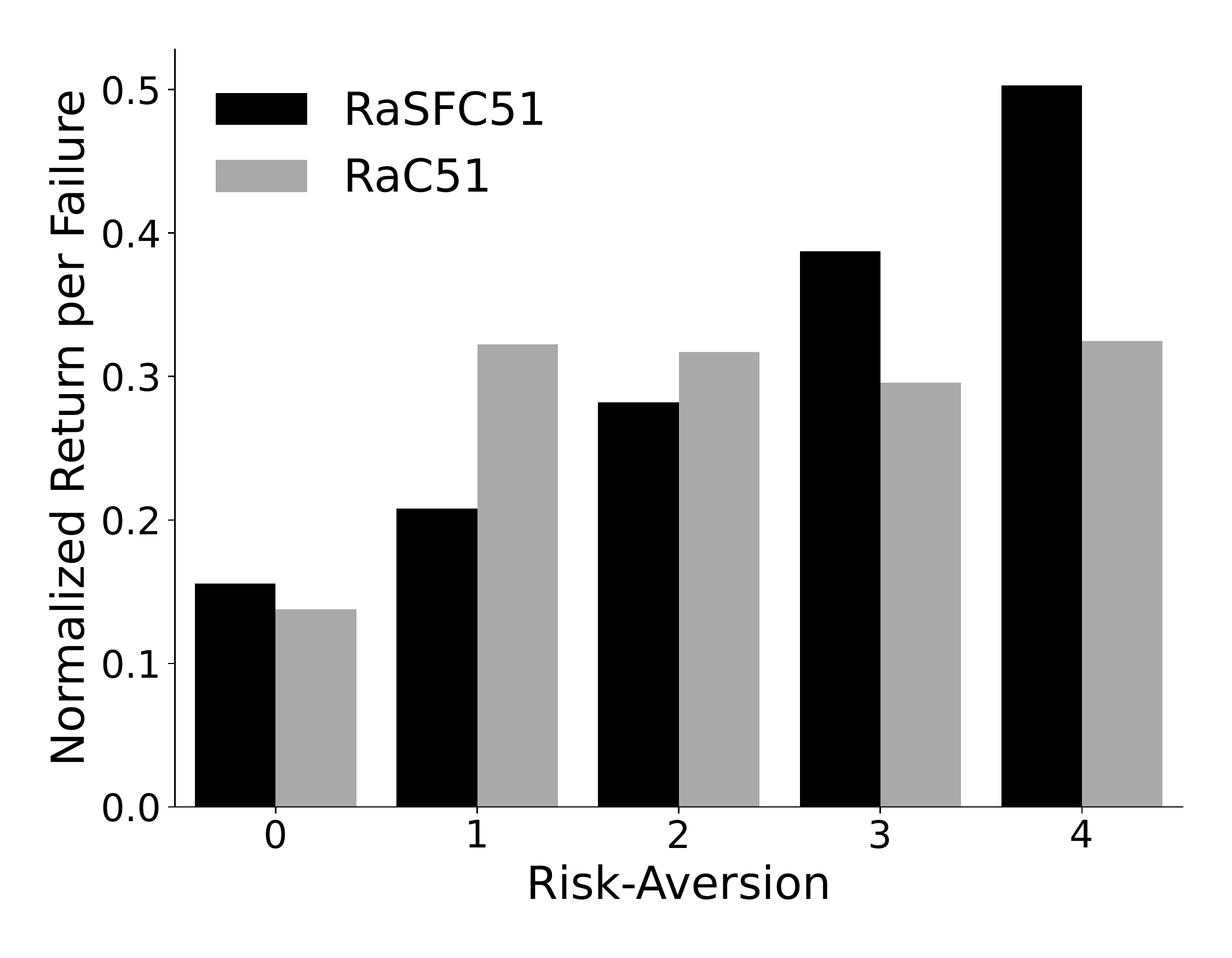}
    \caption{Ablation study for the number of failures on test tasks in the reacher domain, as a function of $\beta$. \textbf{Left:} Total failures across all test tasks for various values of $\beta$. Legend entries in parentheses indicate the negative values of $\beta$. \textbf{Right:} In order to assess the trade-off between return and possibility of failure, we divide the normalized return, averaged across all test tasks, by the total number of failures for each value of $\beta$. The resulting measure is compared between RaSFC51 and RaC51. The x-axis indicates the negative values of $\beta$.}
    \label{fig:reacher_ablation_beta}
\end{figure}

As suggested in the main text, one possible conjecture is that RaSFC51 learns to correctly solve the test tasks, requiring the robotic arm to hover closer to the edge of the risky areas, while RaC51 does not. The presence of environment stochasticity, errors in function approximation, and the stochasticity of the epsilon-greedy policy used during testing could exacerbate this. Comparing the rollouts produced by RaSFC51 and RaC51 in training tasks in Figure \ref{fig:reacher_rollouts_train}, and testing tasks in Figure \ref{fig:reacher_rollouts_test}, confirms that RaSFC51 is much better at task generalization than RaC51. Here, RaSFC51 learns to hover right at the boundaries of the high-variance regions, preferring not to enter them whenever possible. On the other hand, risk-neutral SFC51 is completely unaware of the risky areas, focusing exclusively on minimizing the distance to the target location, but is able to successfully locate all targets. RaC51 demonstrates similar risk-aware behaviors as RaSFC51, but cannot reliably locate the target on some of the test task instances. 

\begin{figure}[!tb]
    \centering
    \includegraphics[width=0.18\linewidth]{figs/figs_ablation/reacher_rollouts_task_0_risk_30.pdf}%
    \includegraphics[width=0.18\linewidth]{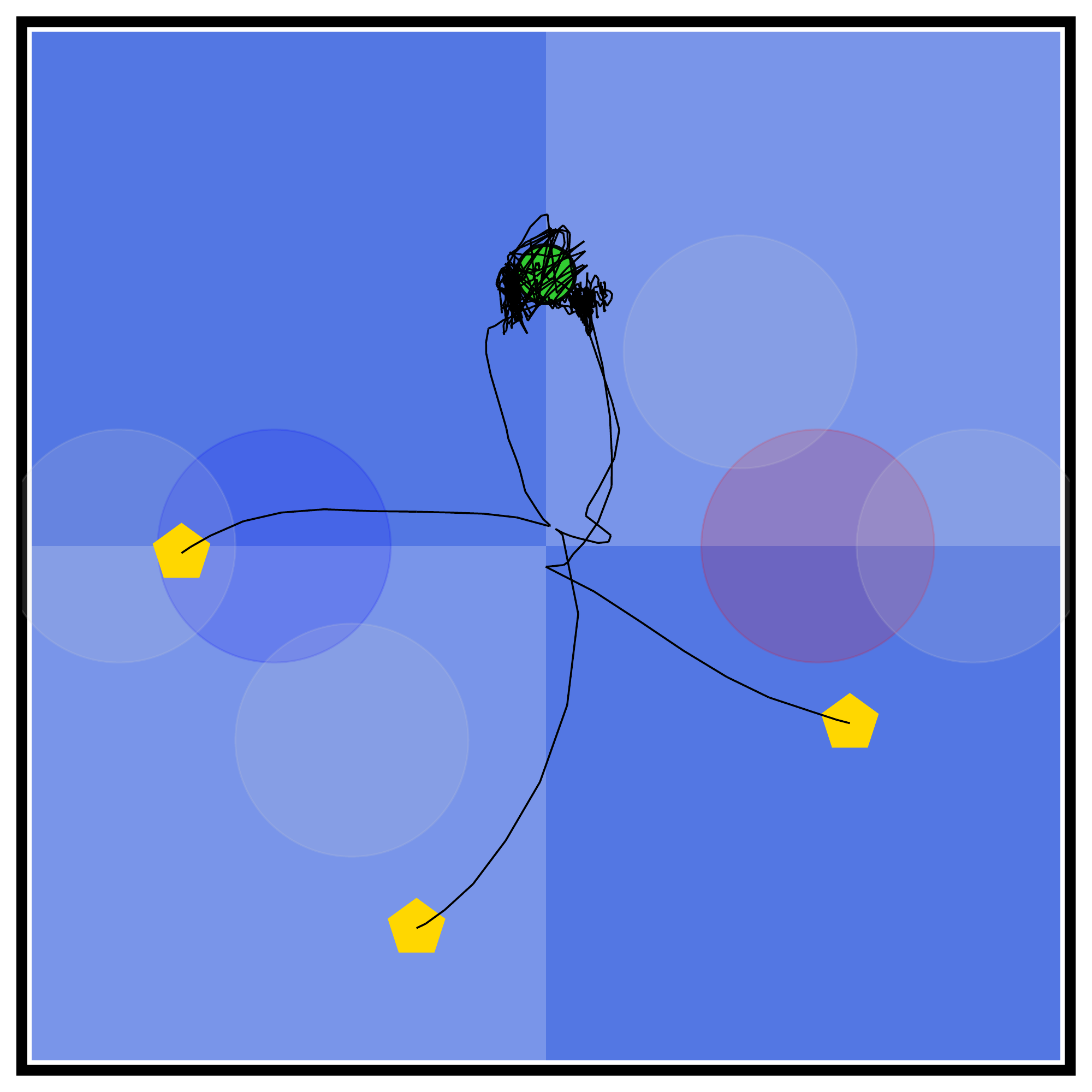}%
    \includegraphics[width=0.18\linewidth]{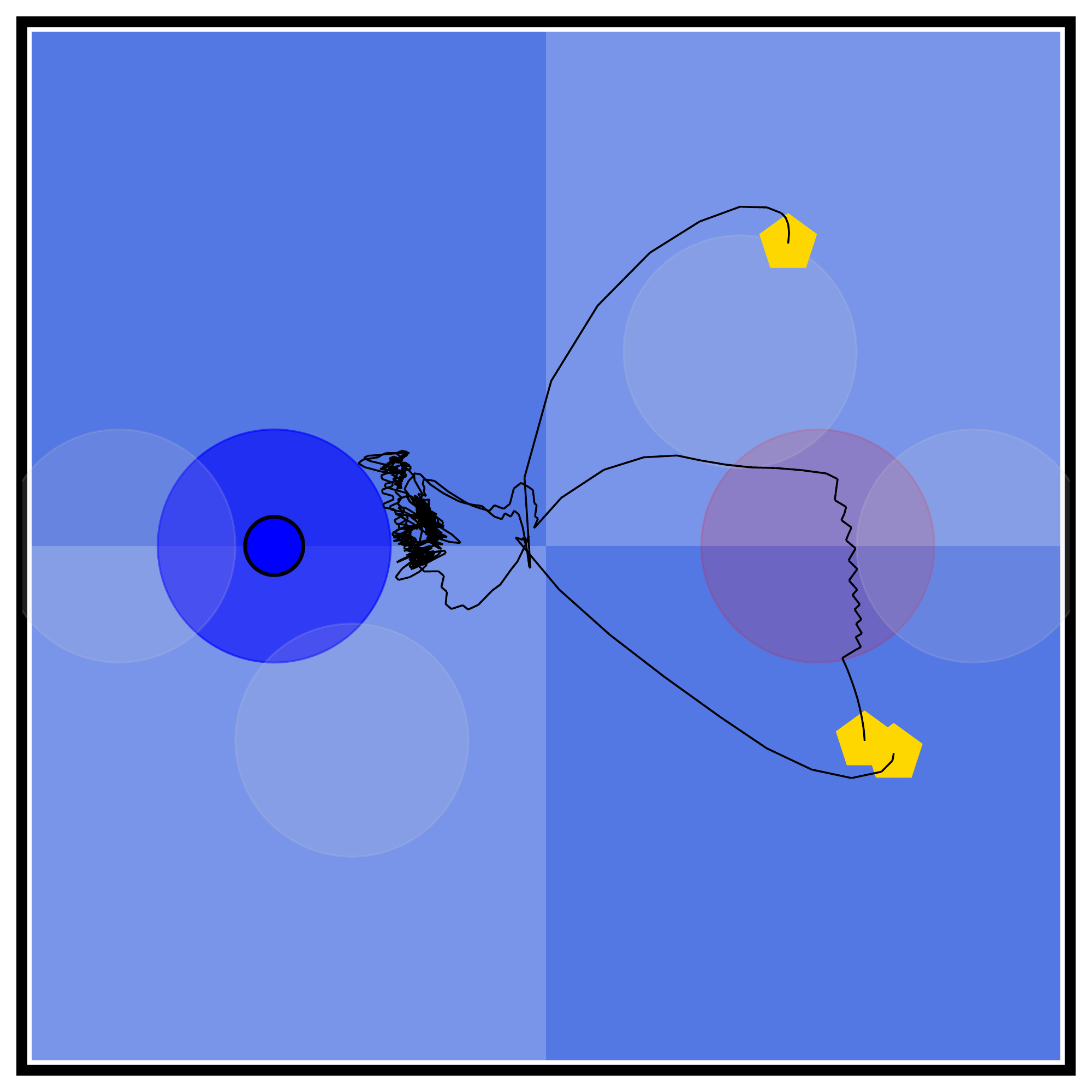}%
    \includegraphics[width=0.18\linewidth]{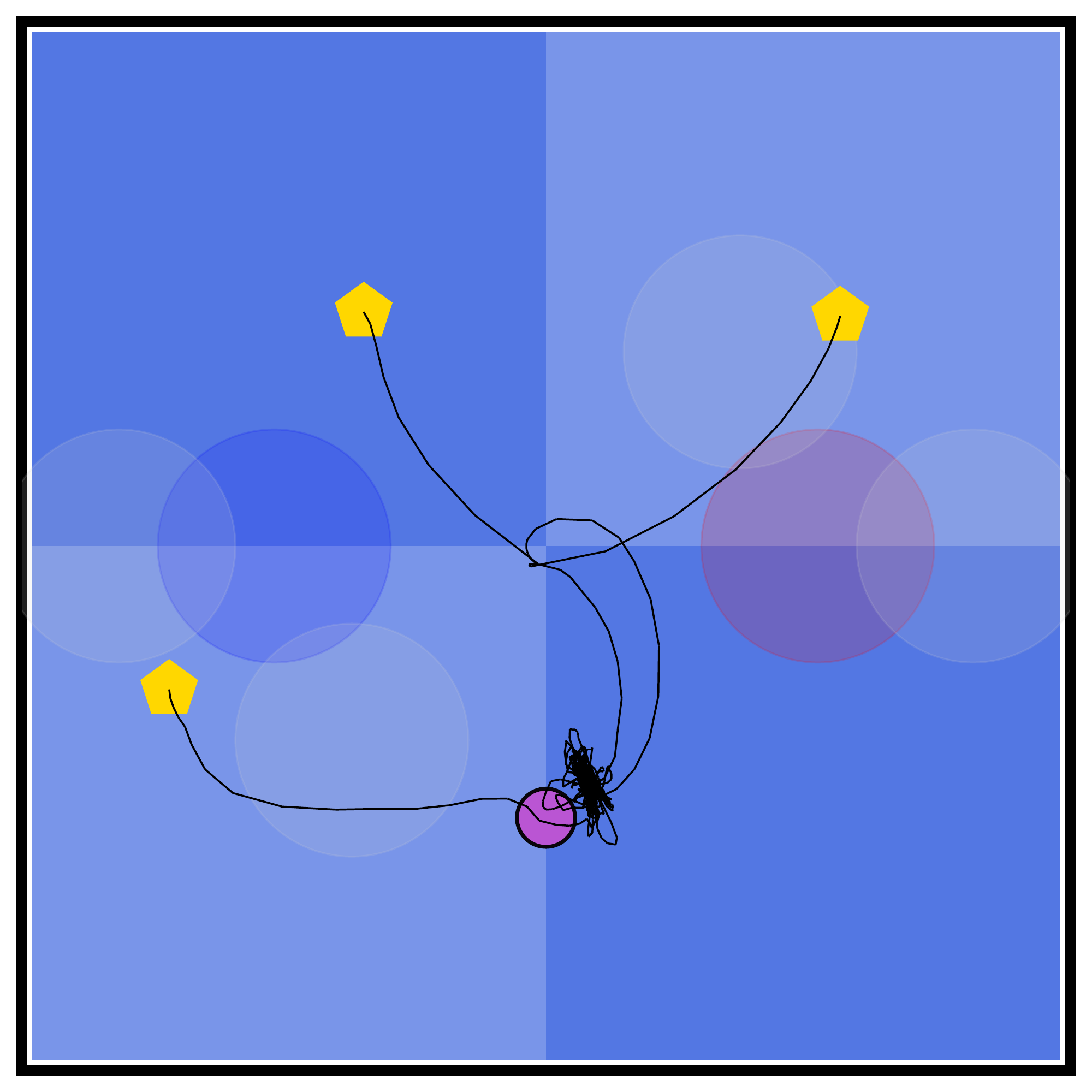}%
    \includegraphics[width=0.012\linewidth]{figs/figs_ablation/reacher_beta_1_axis.pdf}
    
    \includegraphics[width=0.18\linewidth]{figs/figs_ablation/reacher_rollouts_task_0_risk_00.pdf}%
    \includegraphics[width=0.18\linewidth]{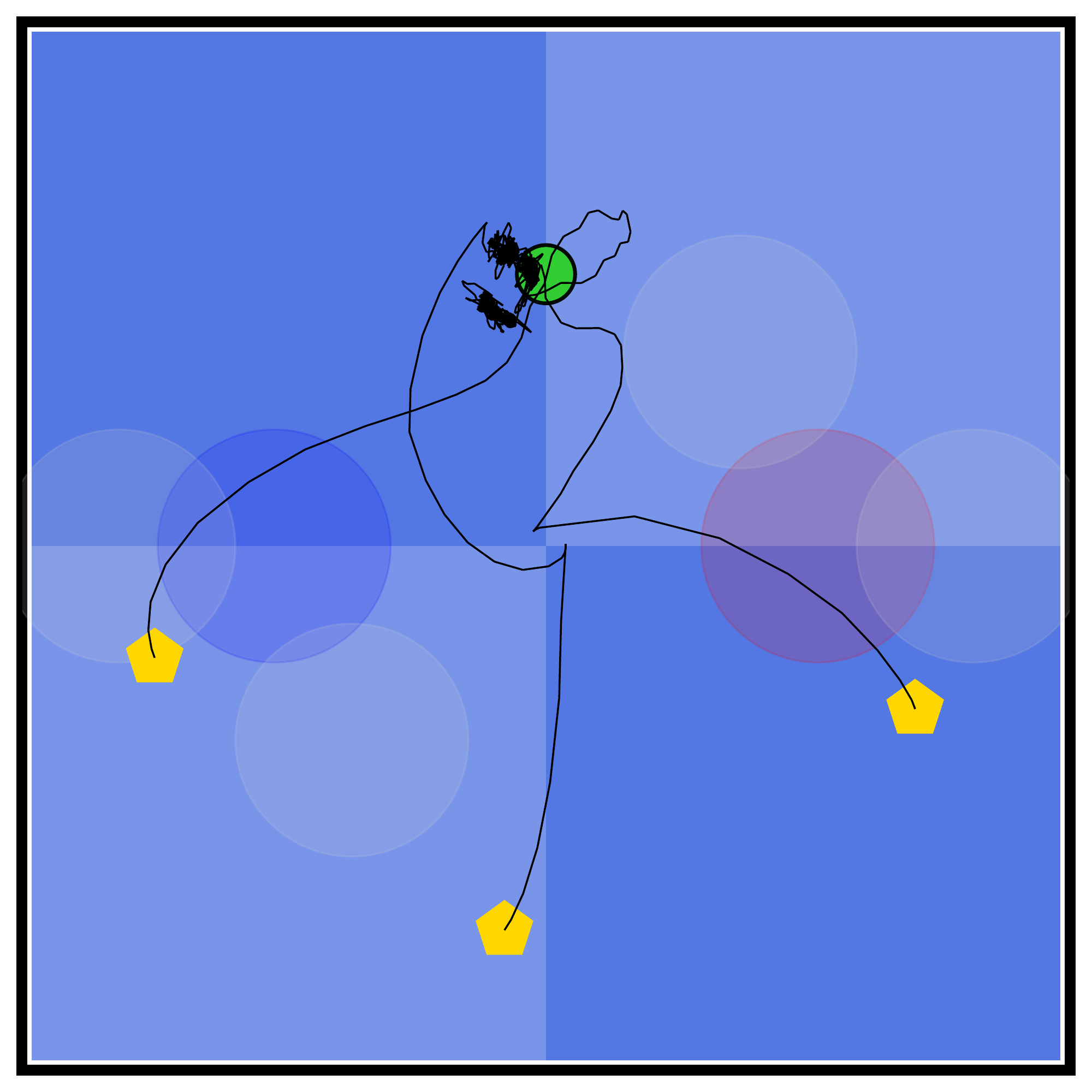}%
    \includegraphics[width=0.18\linewidth]{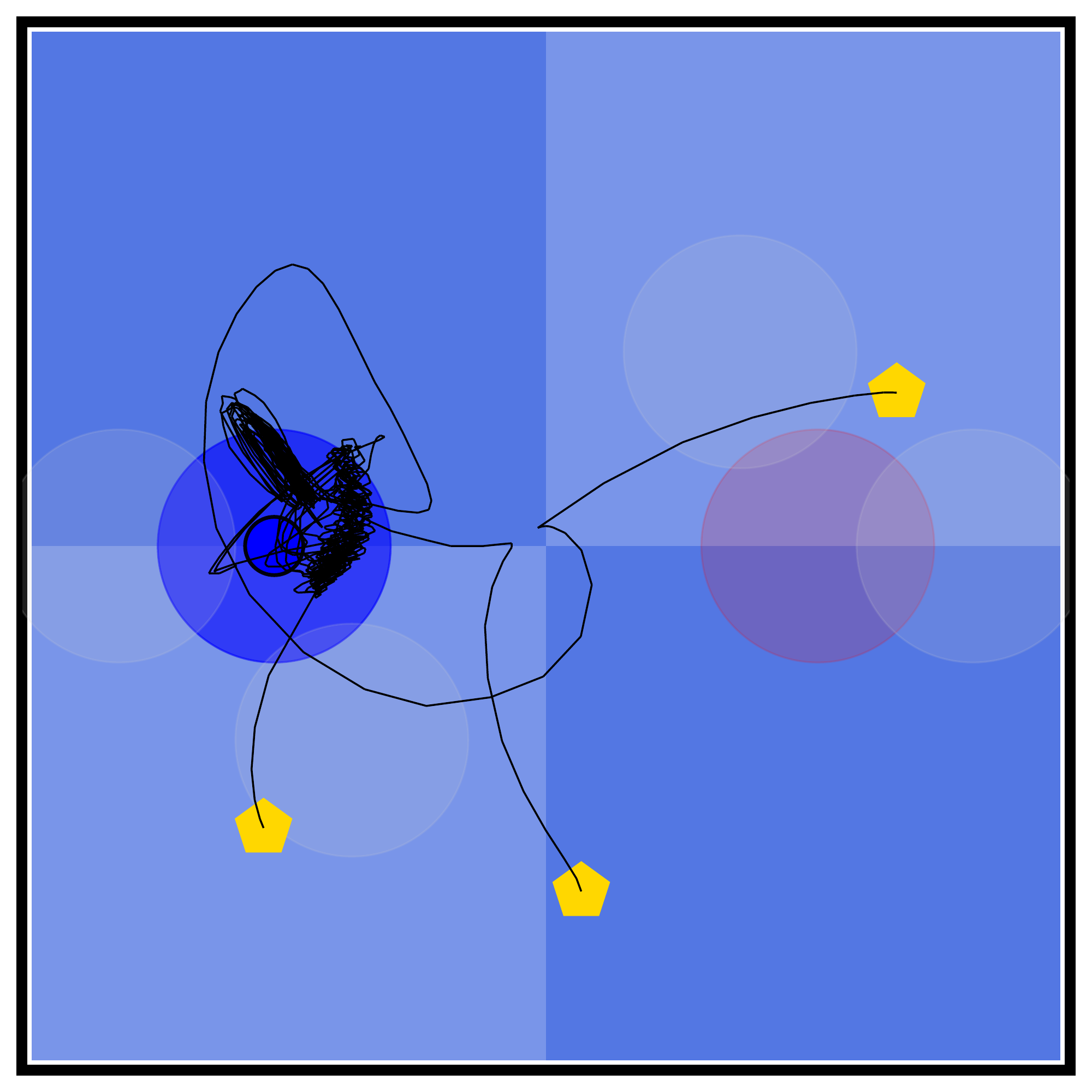}%
    \includegraphics[width=0.18\linewidth]{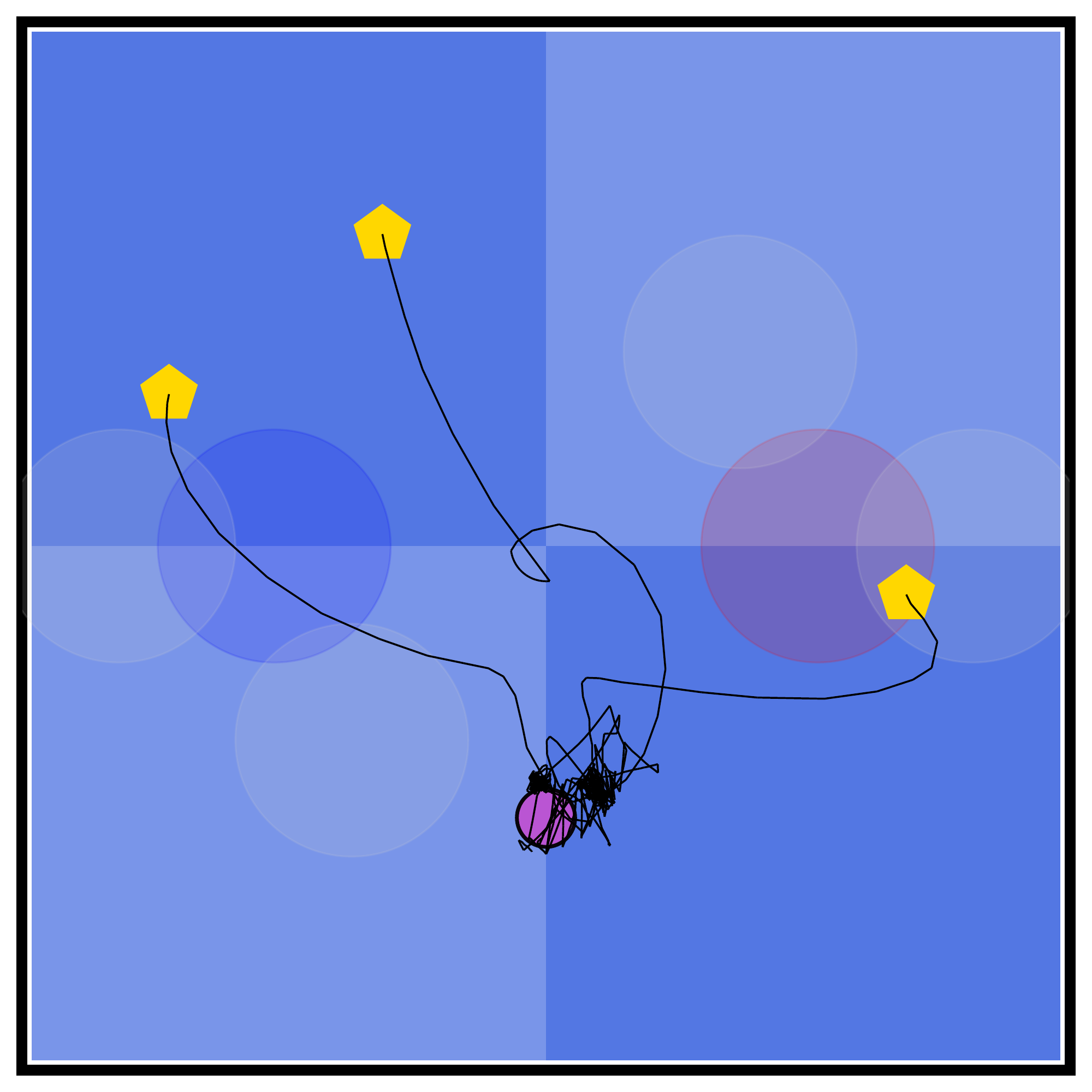}%
    \includegraphics[width=0.012\linewidth]{figs/figs_ablation/reacher_beta_2_axis.pdf}
    
    \includegraphics[width=0.18\linewidth]{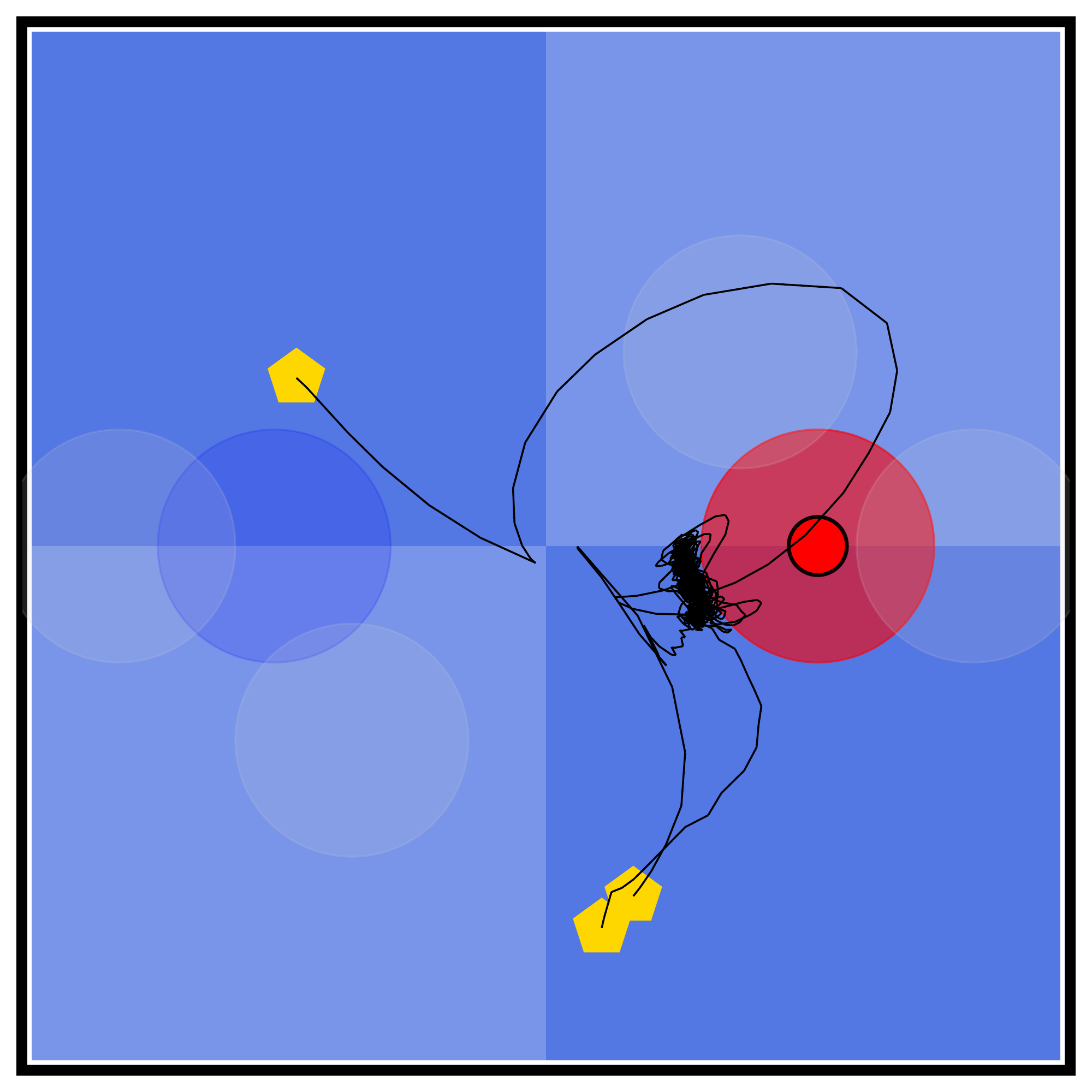}%
    \includegraphics[width=0.18\linewidth]{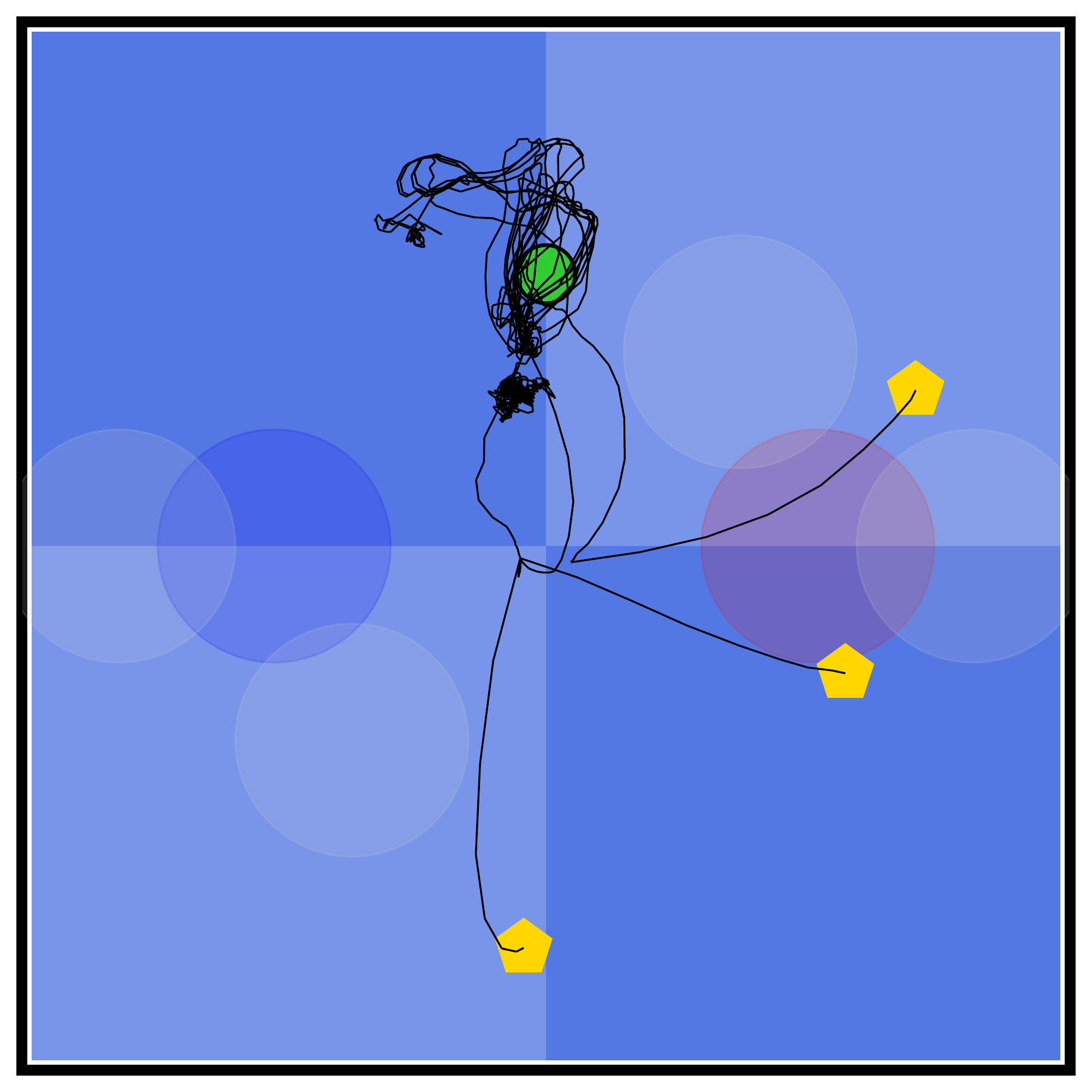}%
    \includegraphics[width=0.18\linewidth]{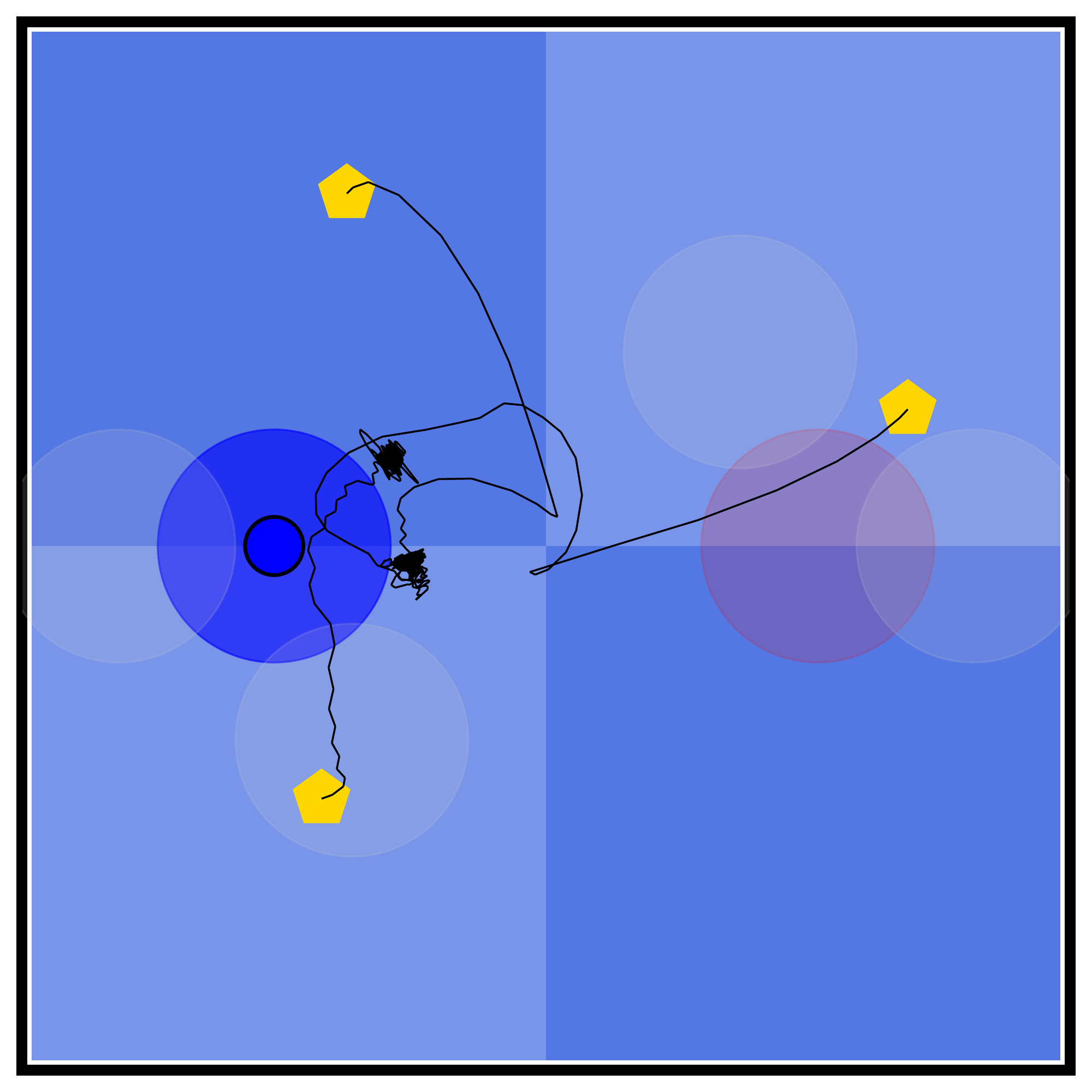}%
    \includegraphics[width=0.18\linewidth]{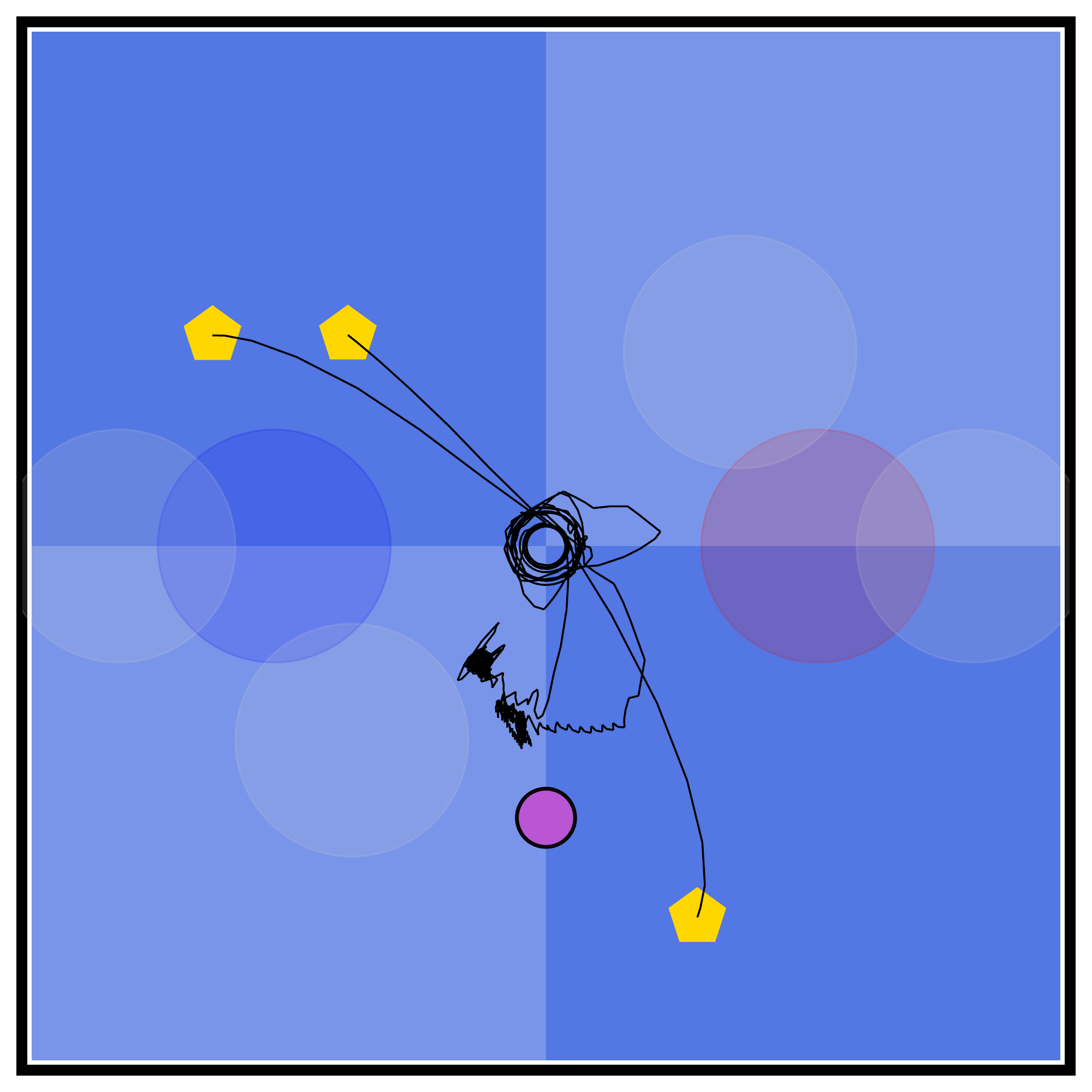}%
    \includegraphics[width=0.012\linewidth]{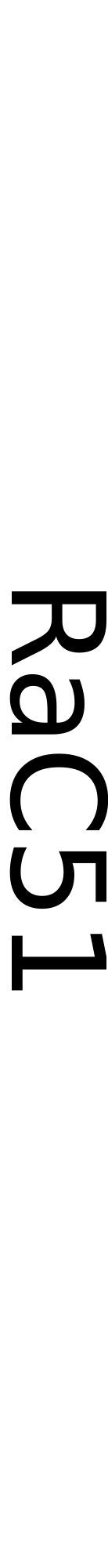}
    \caption{Evolutions of the robotic arm tip position in three rollouts of the reacher domain according to the GPI policy obtained after training on all 4 tasks. Here, all 4 training tasks are shown.}
    \label{fig:reacher_rollouts_train}
\end{figure}

\begin{figure}[!tb]
    \centering
    \includegraphics[width=0.125\linewidth]{figs/figs_ablation/reacher_rollouts_task_4_risk_30.pdf}%
    \includegraphics[width=0.125\linewidth]{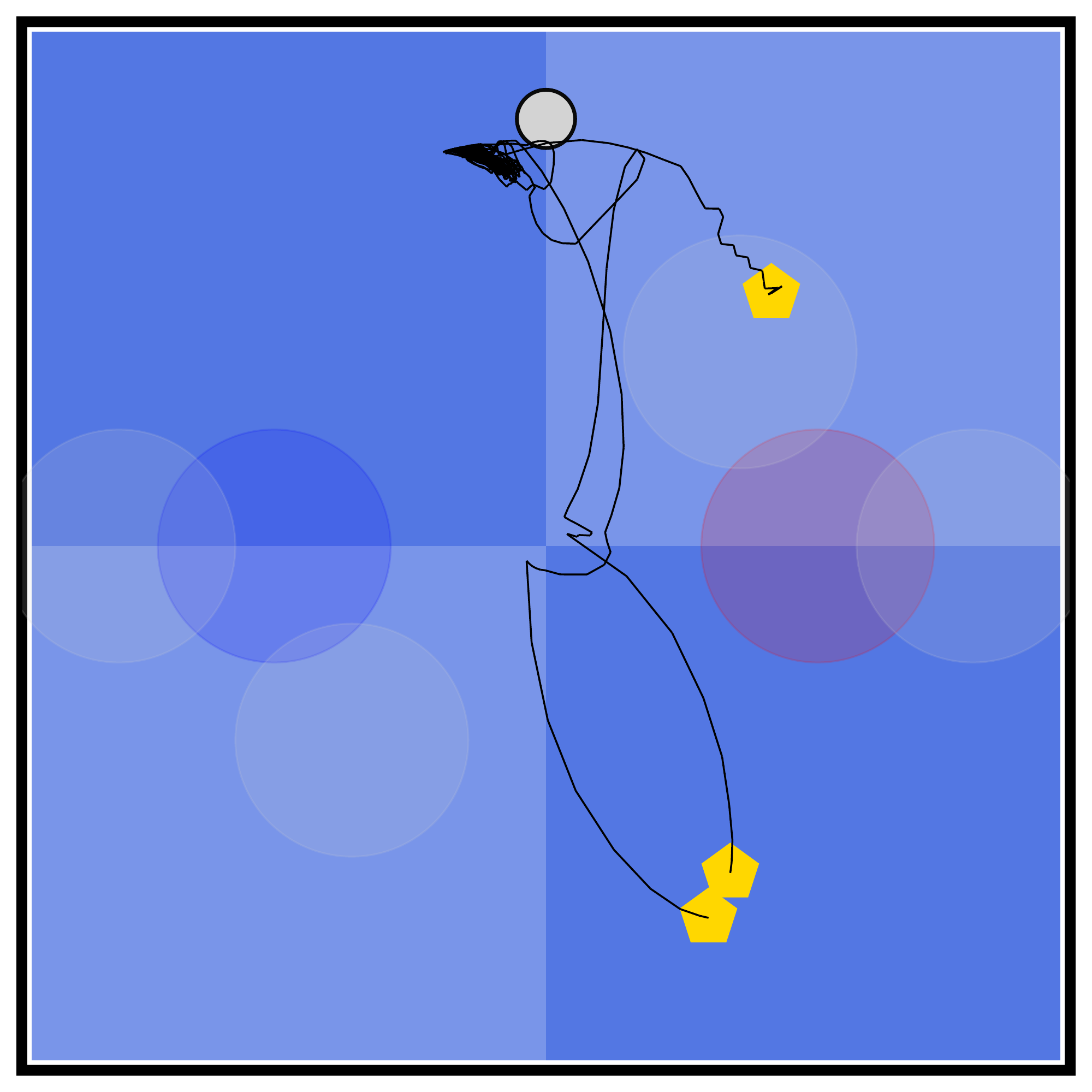}%
    \includegraphics[width=0.125\linewidth]{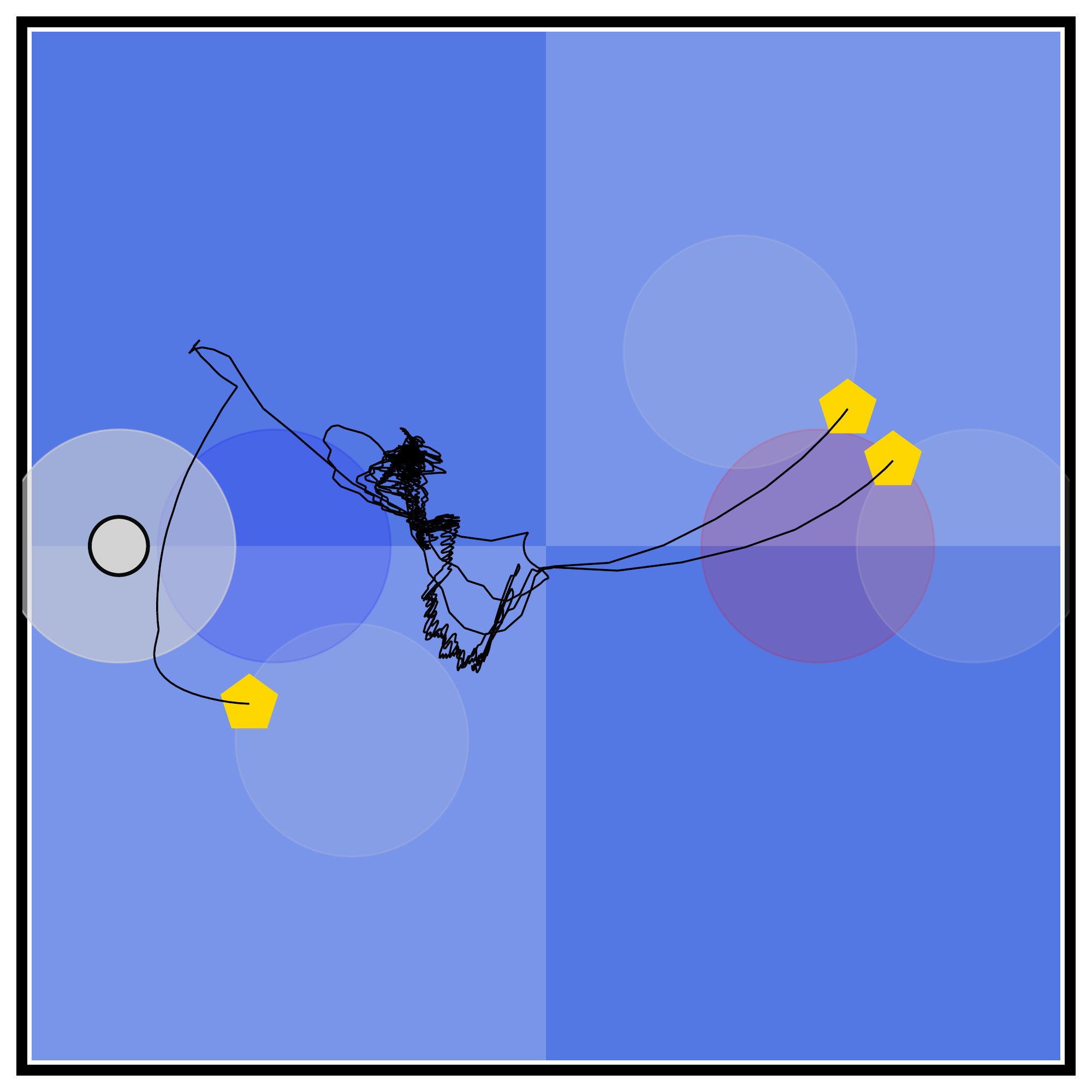}%
    \includegraphics[width=0.125\linewidth]{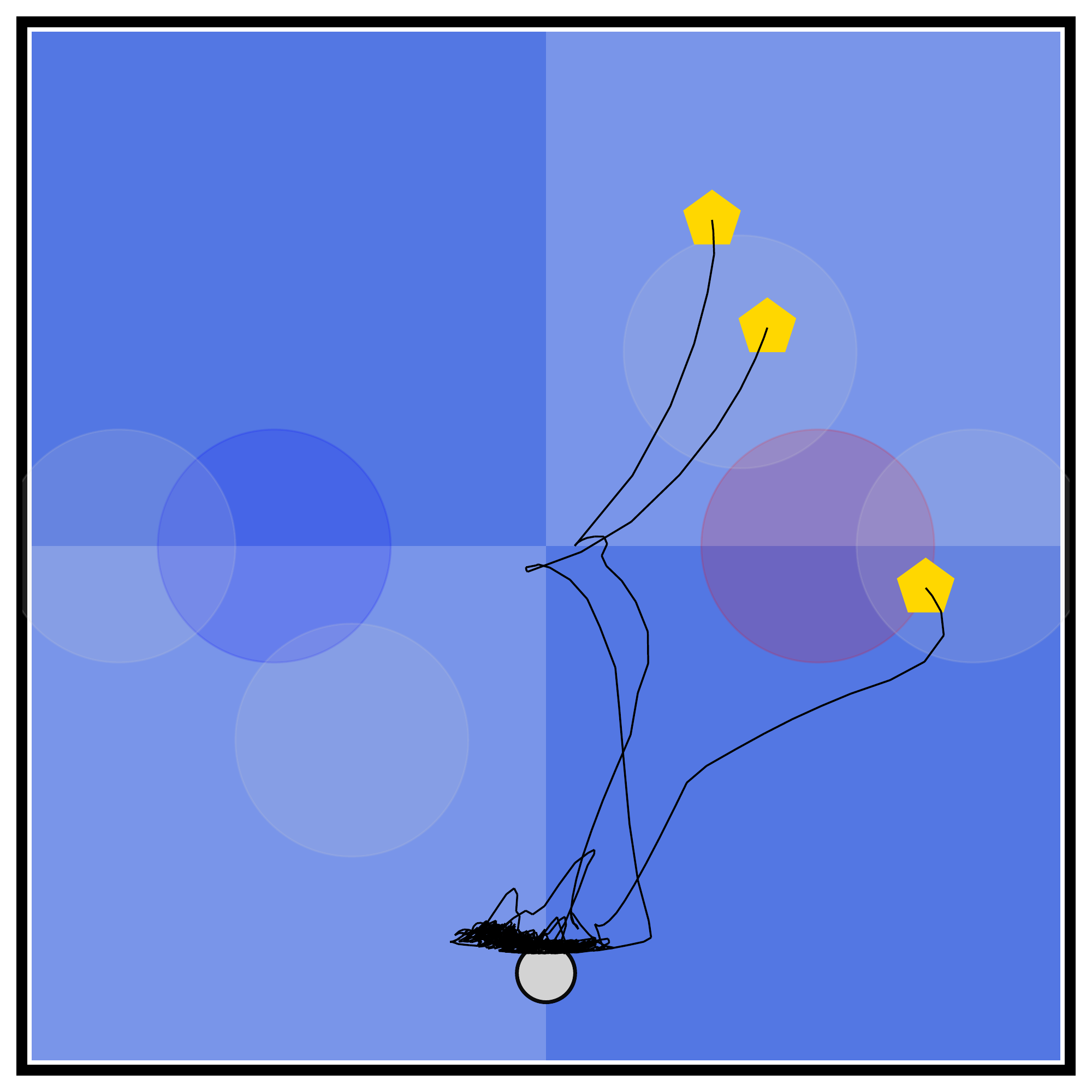}%
    \includegraphics[width=0.125\linewidth]{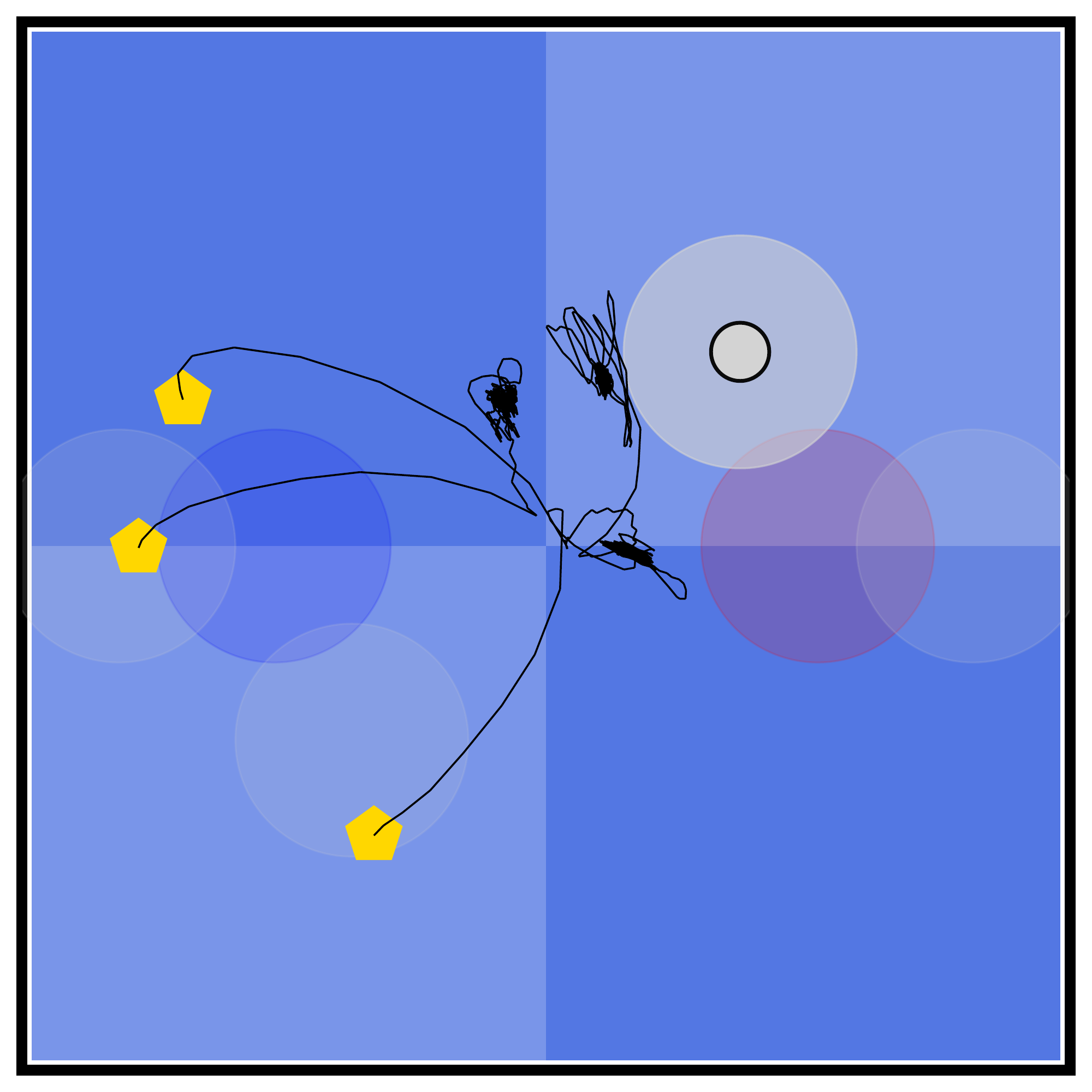}%
    \includegraphics[width=0.125\linewidth]{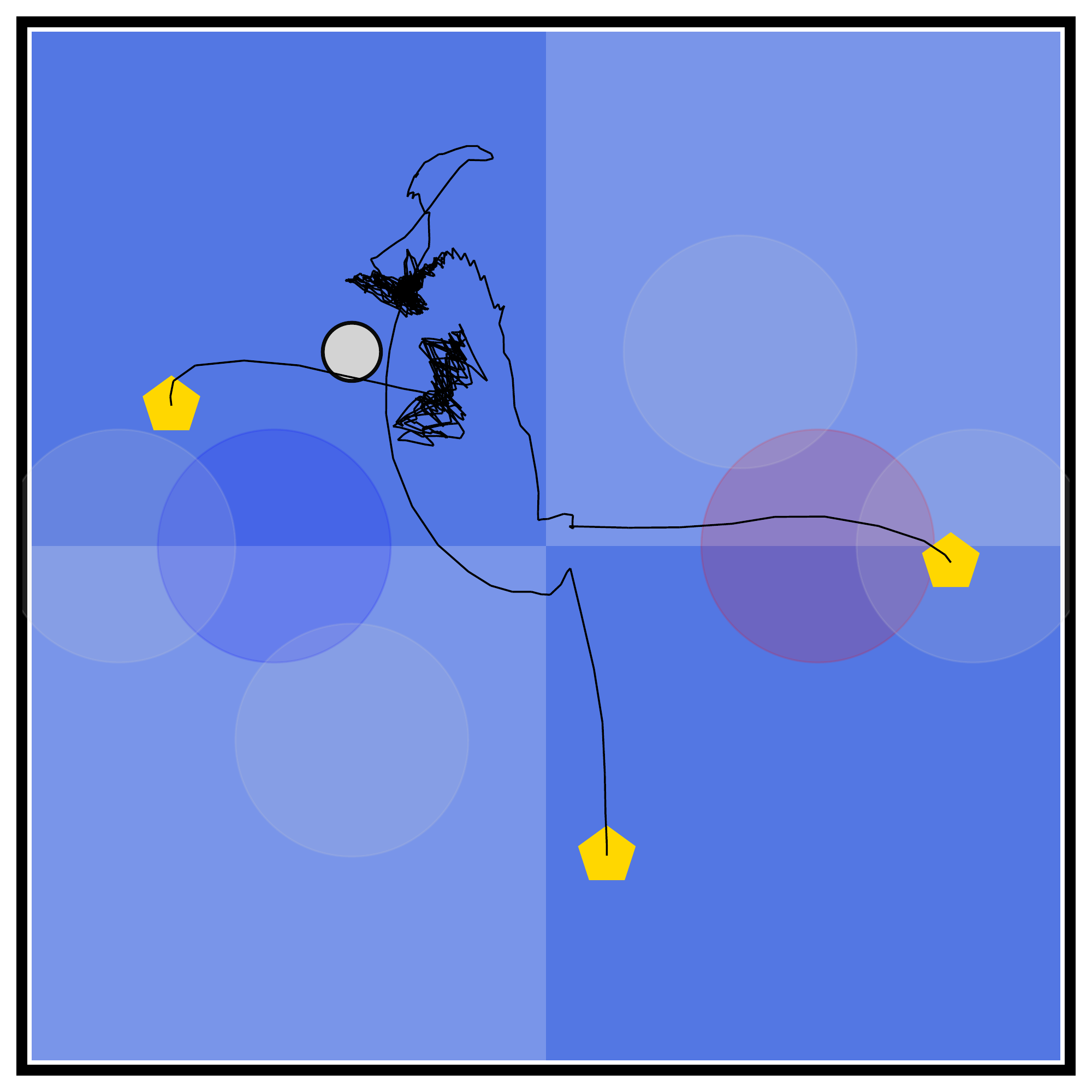}%
    \includegraphics[width=0.125\linewidth]{figs/figs_ablation/reacher_rollouts_task_10_risk_30.pdf}%
    \includegraphics[width=0.125\linewidth]{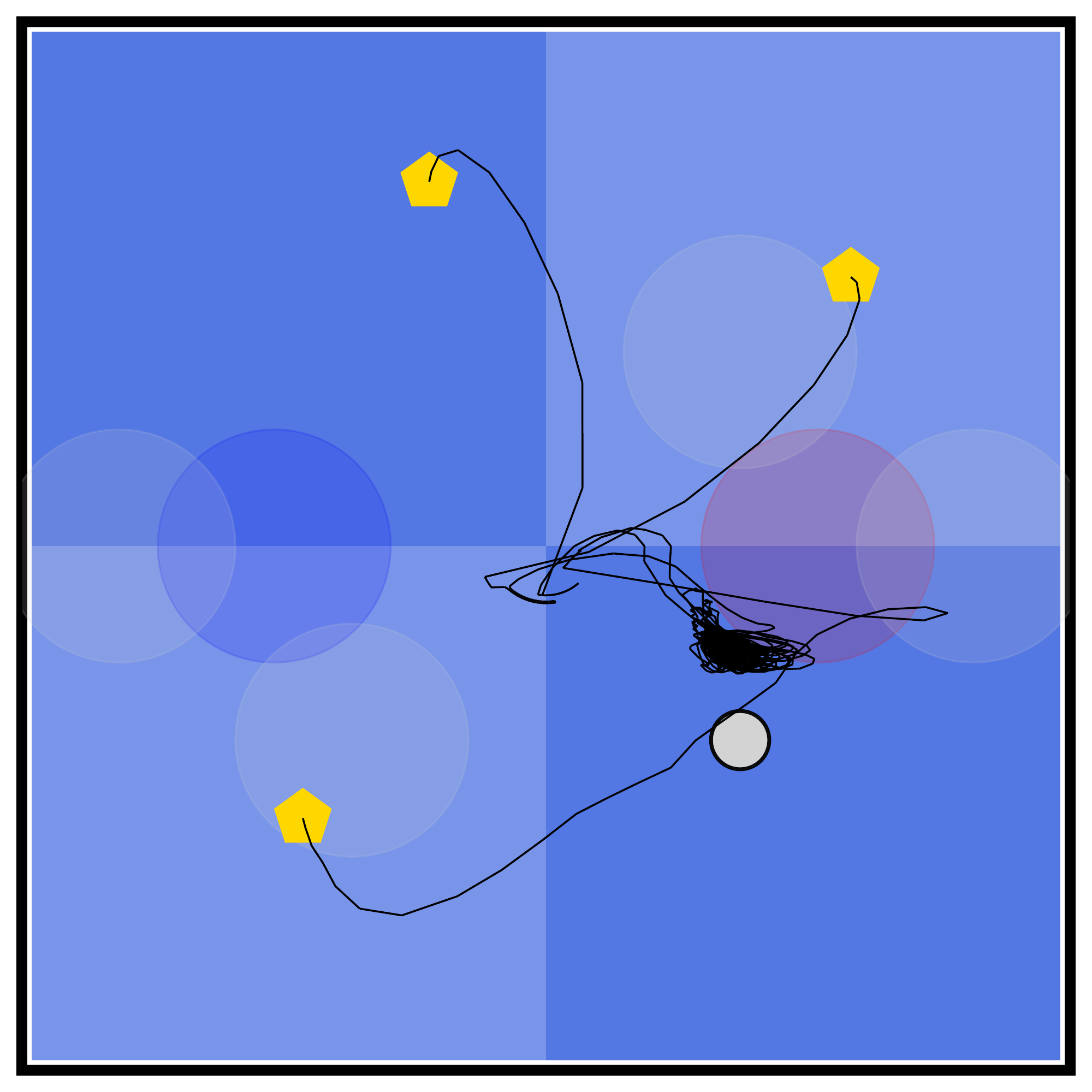}%
    \includegraphics[width=0.009\linewidth]{figs/figs_ablation/reacher_beta_1_axis.pdf}
    
    \includegraphics[width=0.125\linewidth]{figs/figs_ablation/reacher_rollouts_task_4_risk_00.pdf}%
    \includegraphics[width=0.125\linewidth]{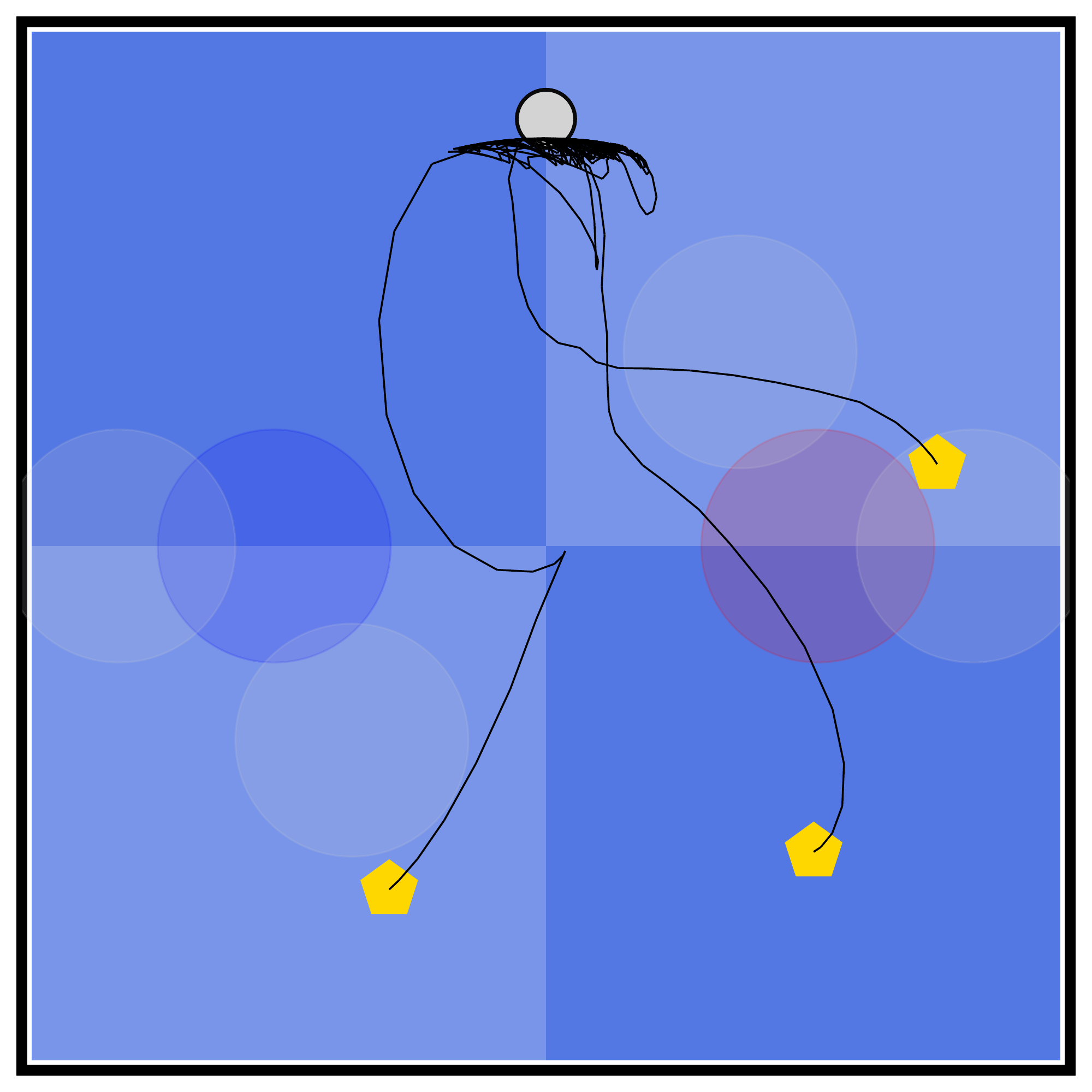}%
    \includegraphics[width=0.125\linewidth]{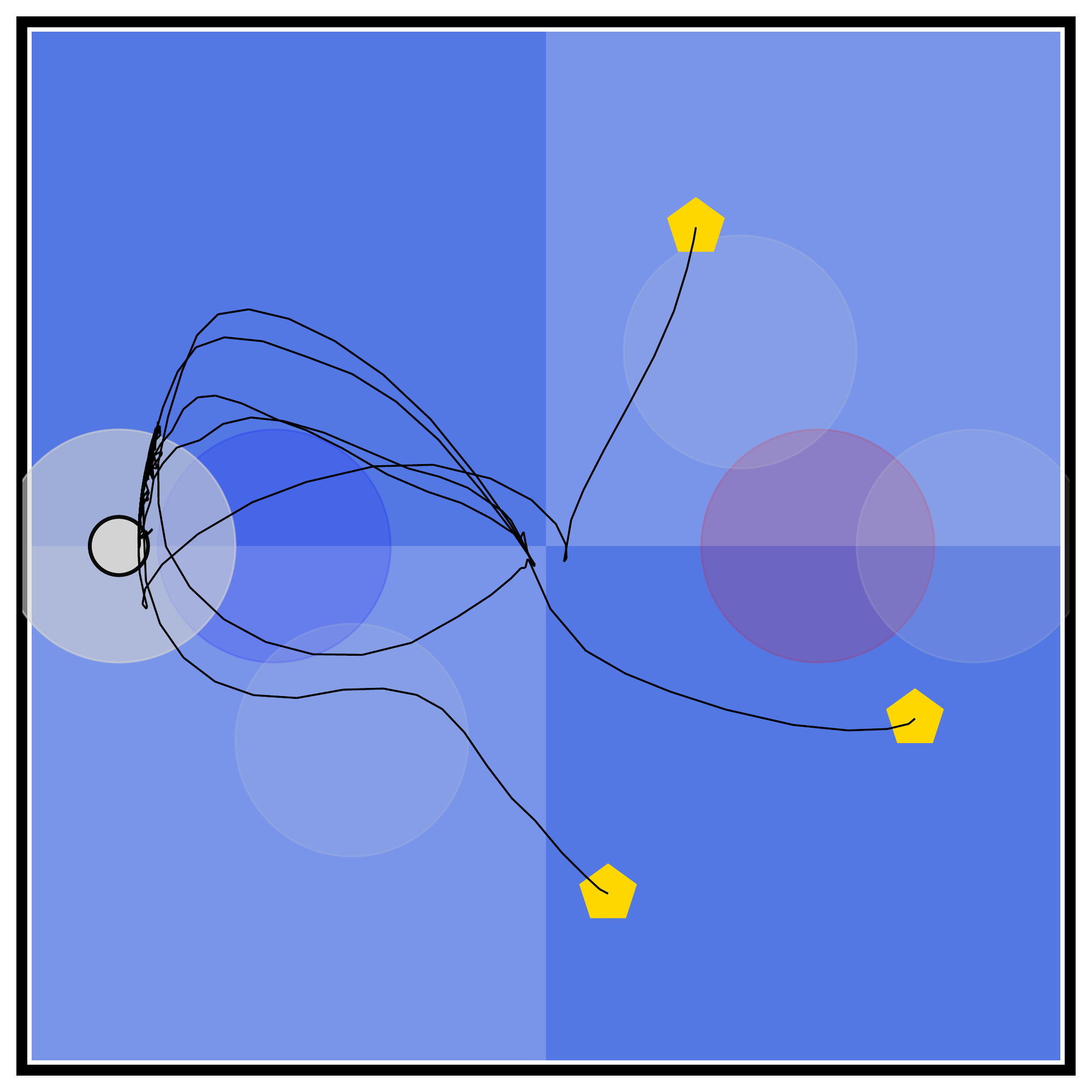}%
    \includegraphics[width=0.125\linewidth]{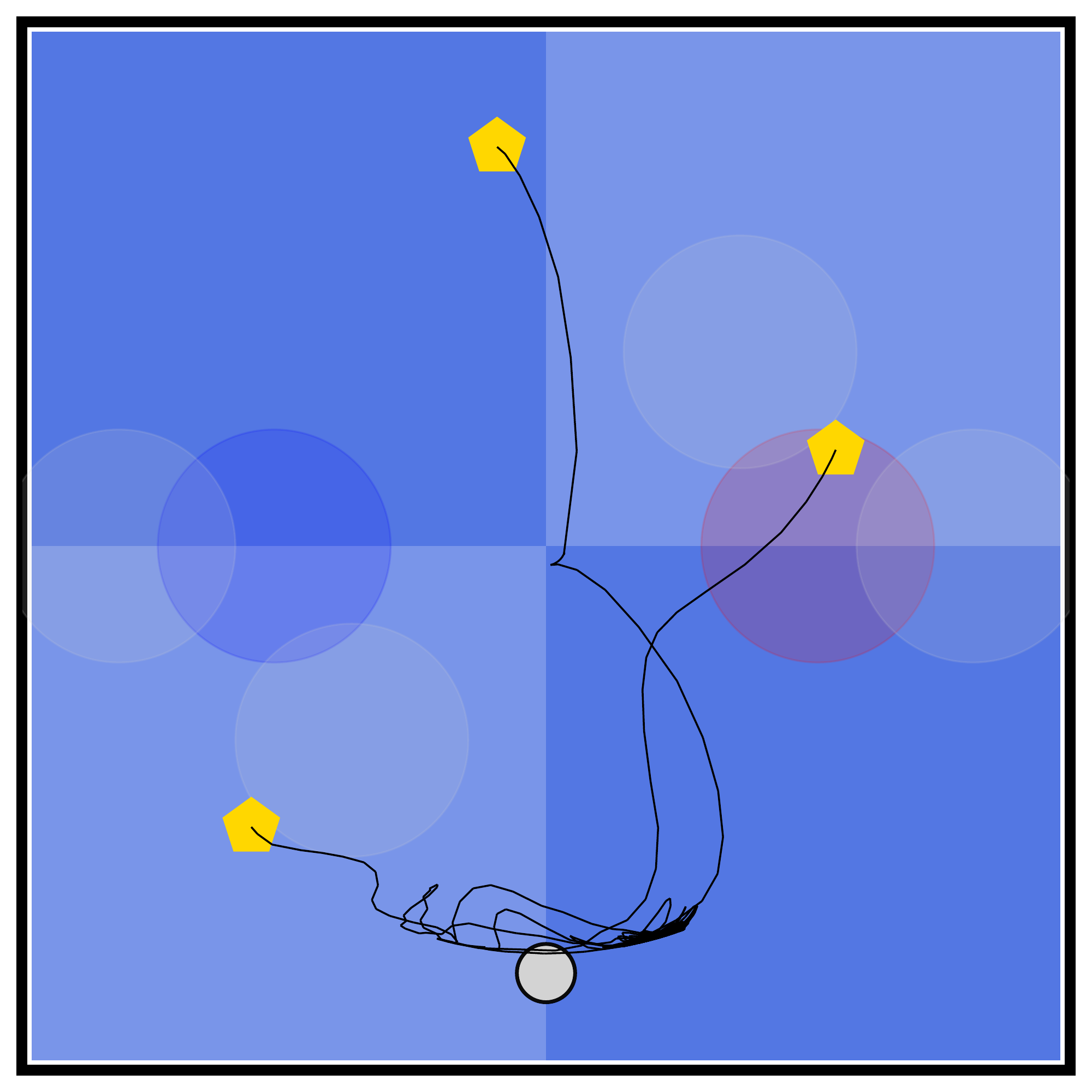}%
    \includegraphics[width=0.125\linewidth]{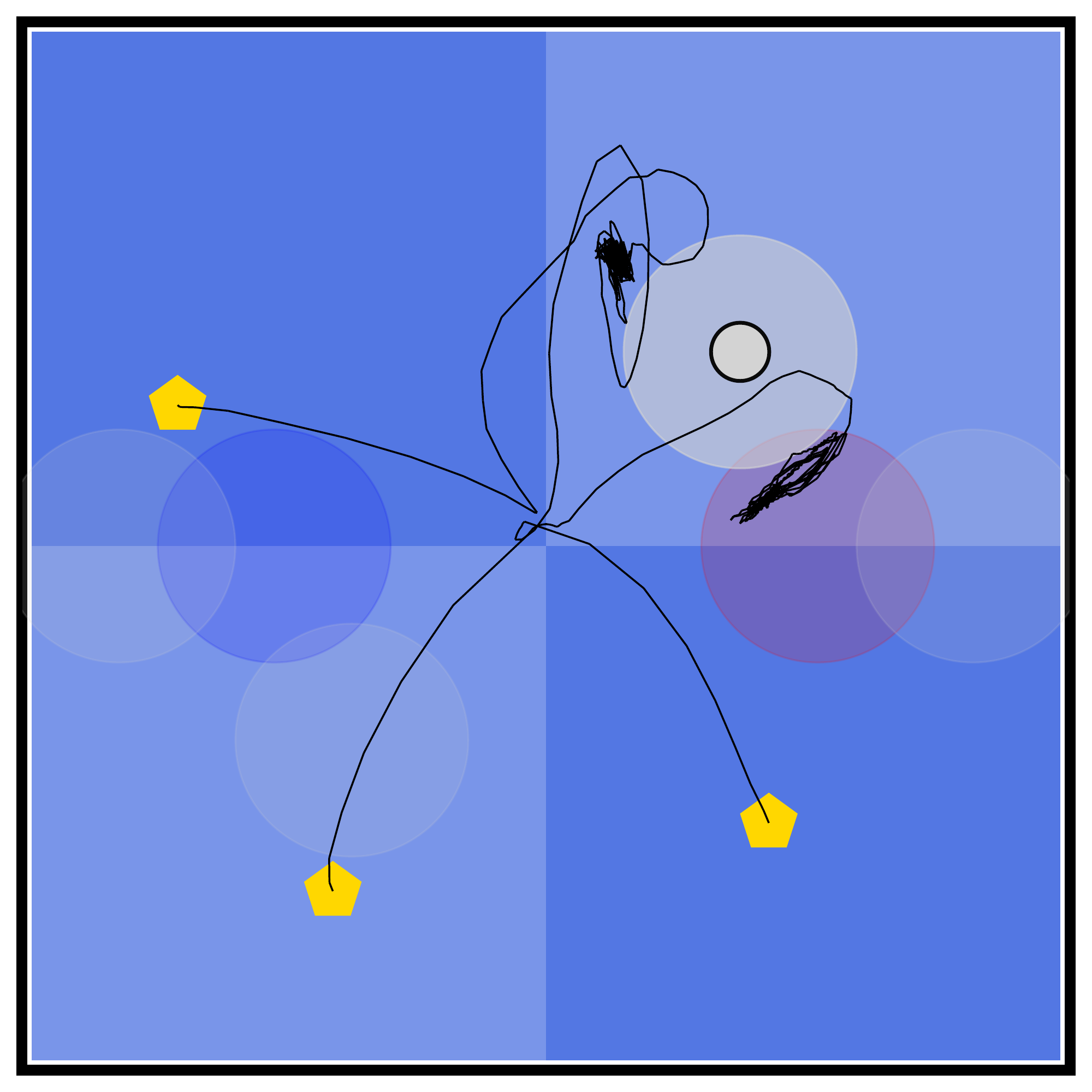}%
    \includegraphics[width=0.125\linewidth]{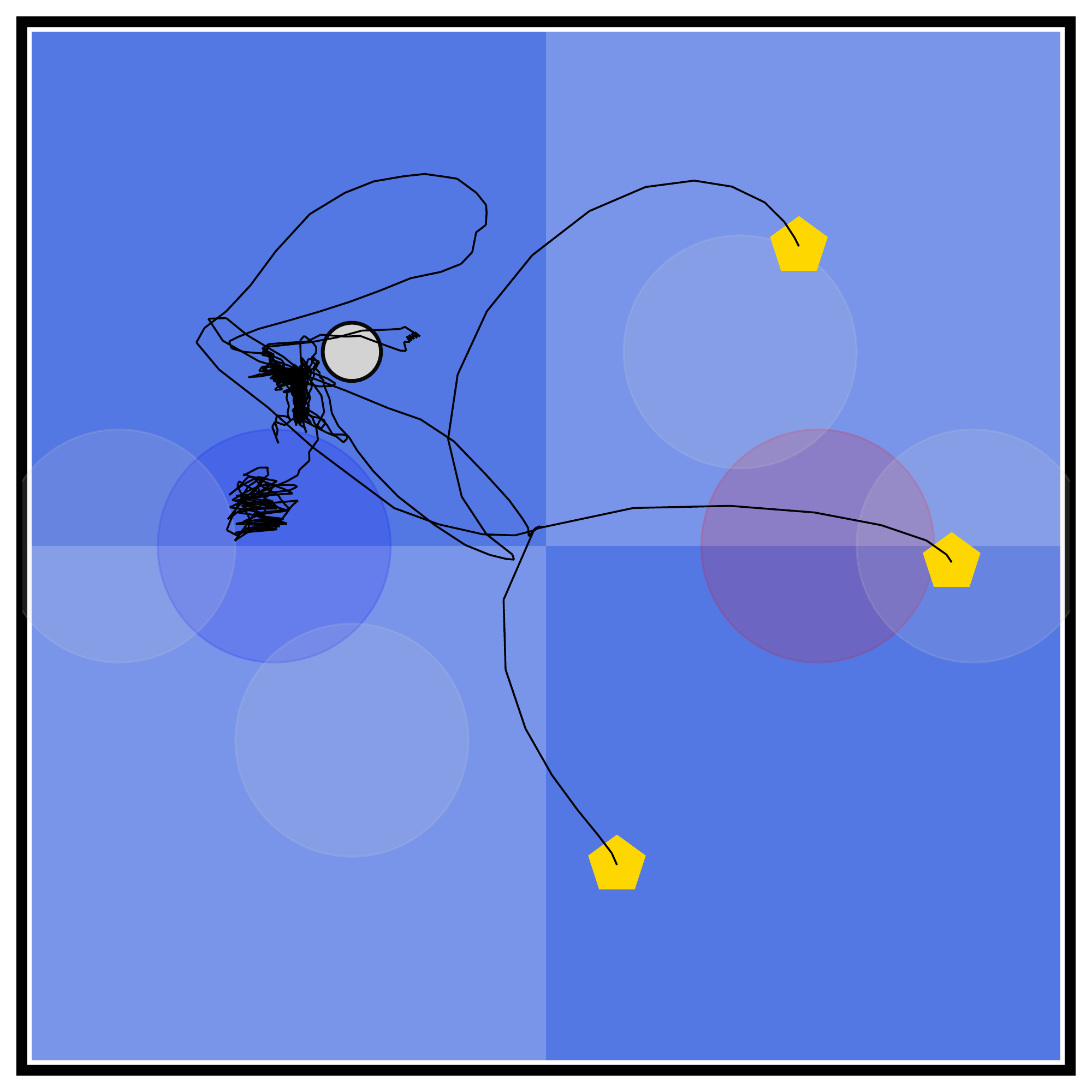}%
    \includegraphics[width=0.125\linewidth]{figs/figs_ablation/reacher_rollouts_task_10_risk_00.pdf}%
    \includegraphics[width=0.125\linewidth]{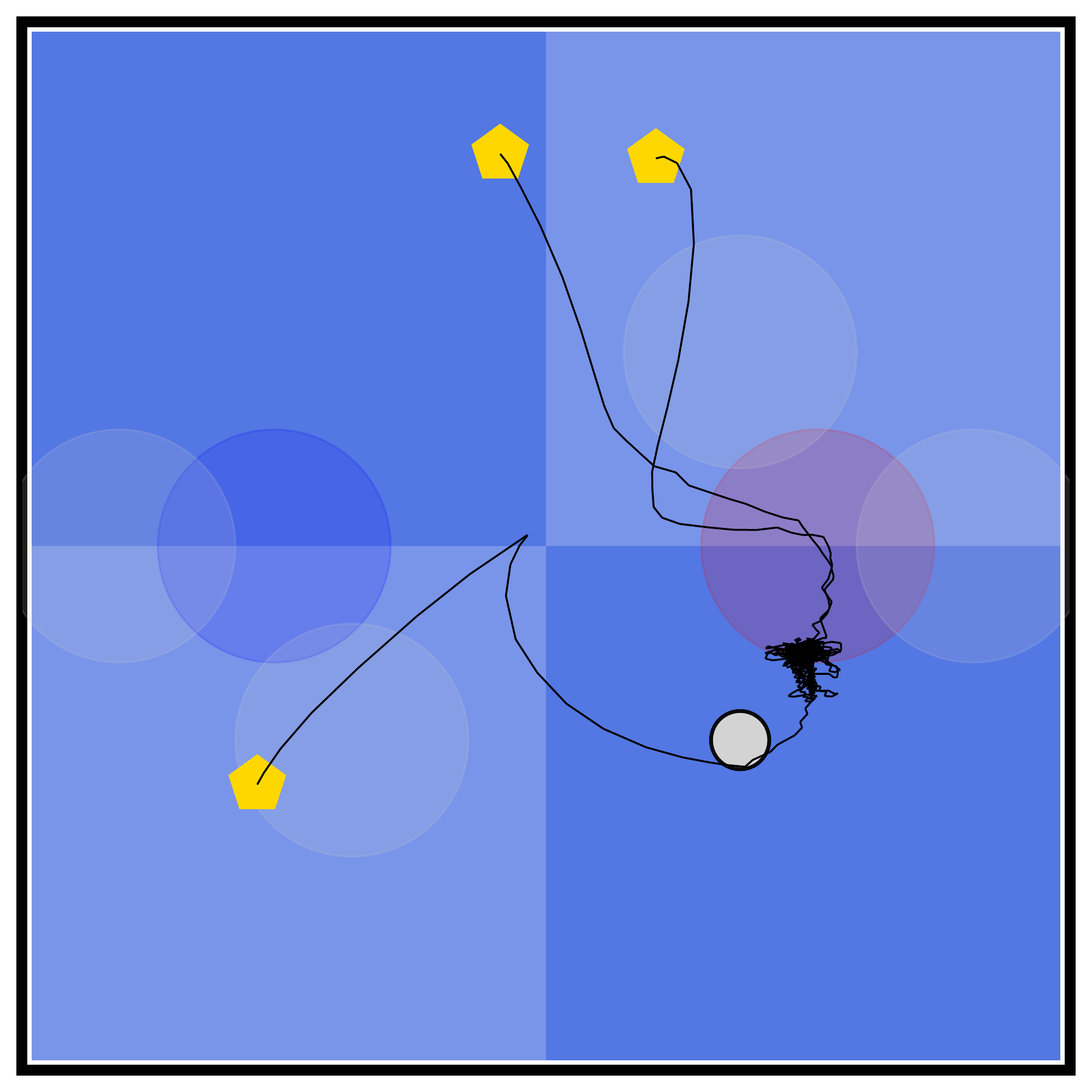}%
    \includegraphics[width=0.009\linewidth]{figs/figs_ablation/reacher_beta_2_axis.pdf}
    
    \includegraphics[width=0.125\linewidth]{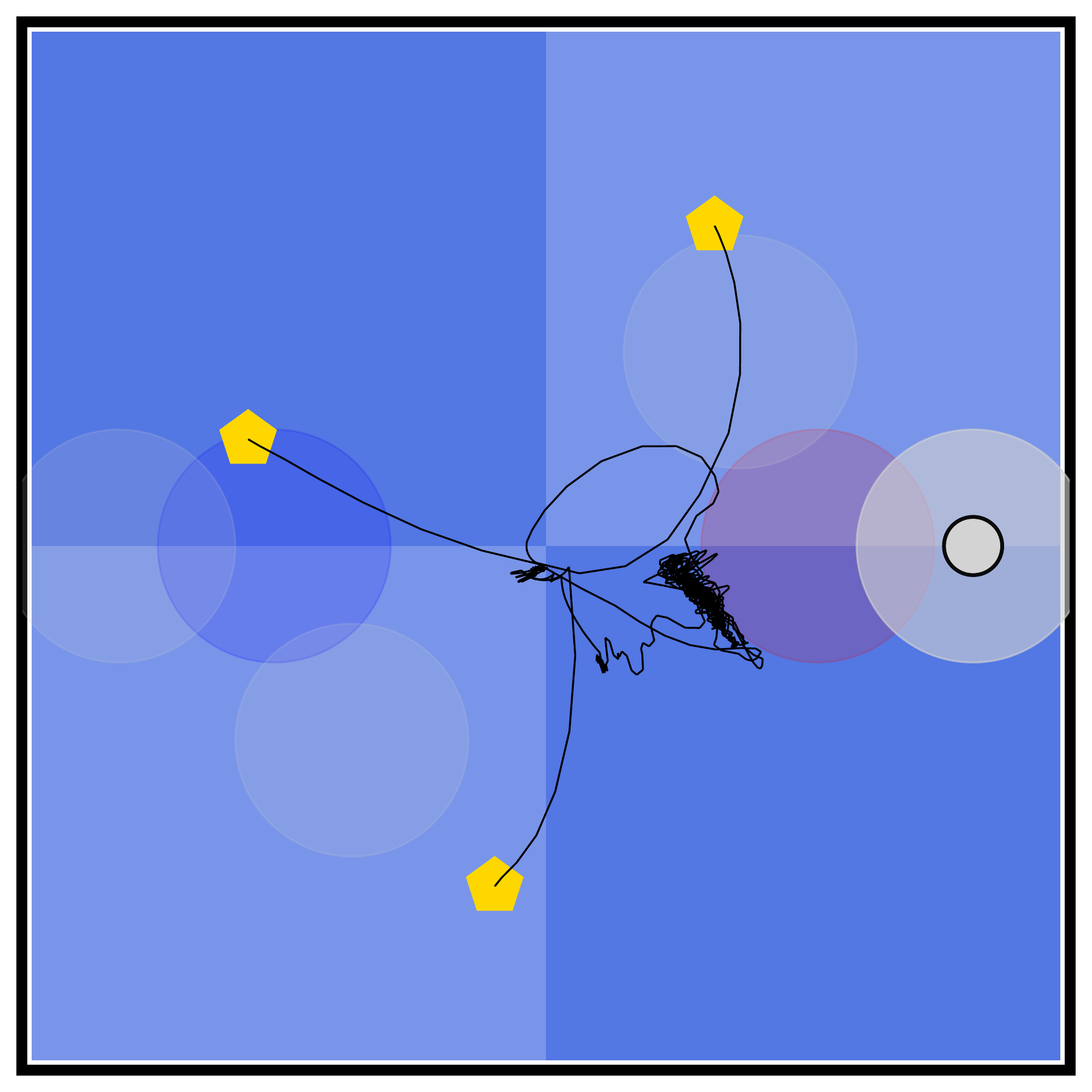}%
    \includegraphics[width=0.125\linewidth]{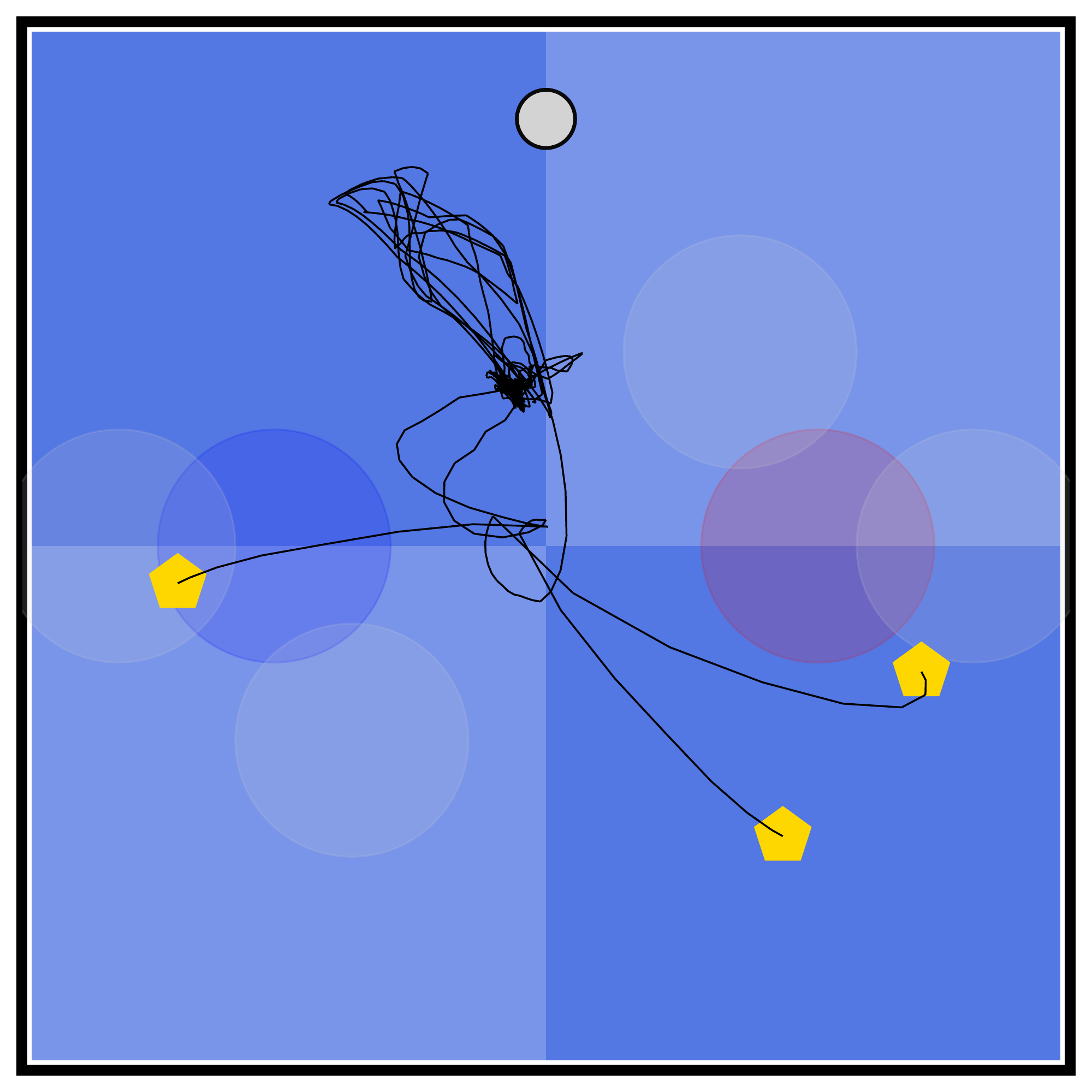}%
    \includegraphics[width=0.125\linewidth]{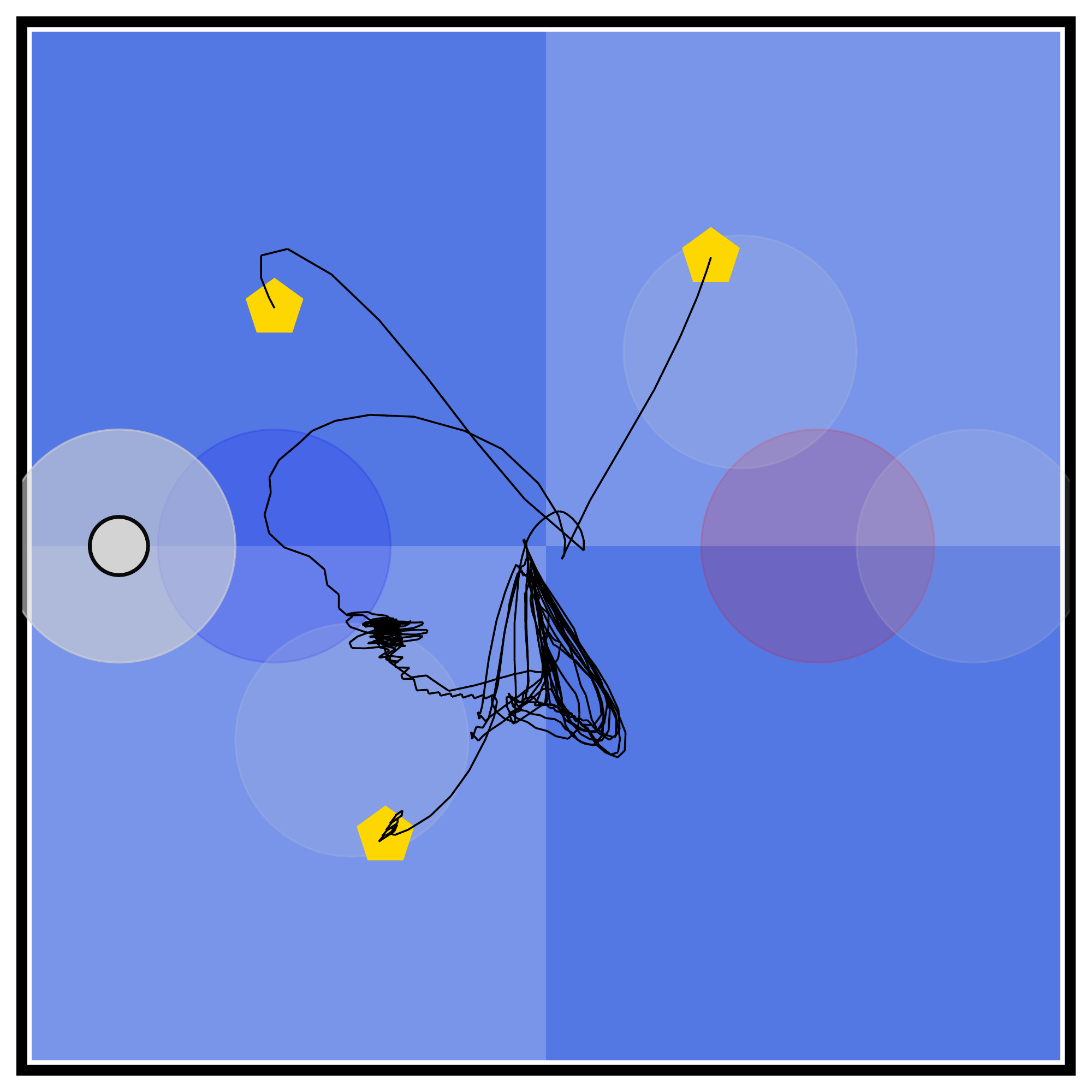}%
    \includegraphics[width=0.125\linewidth]{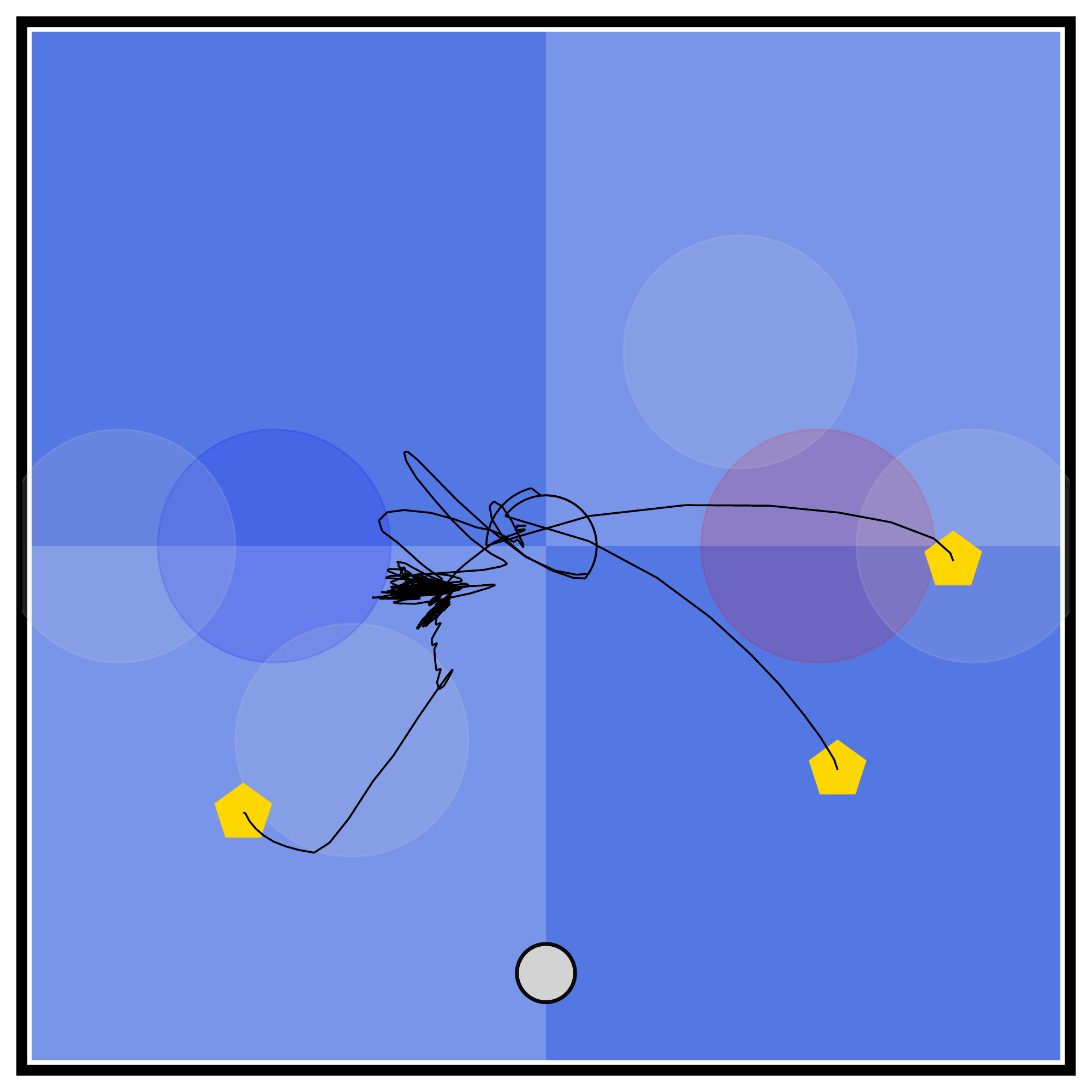}%
    \includegraphics[width=0.125\linewidth]{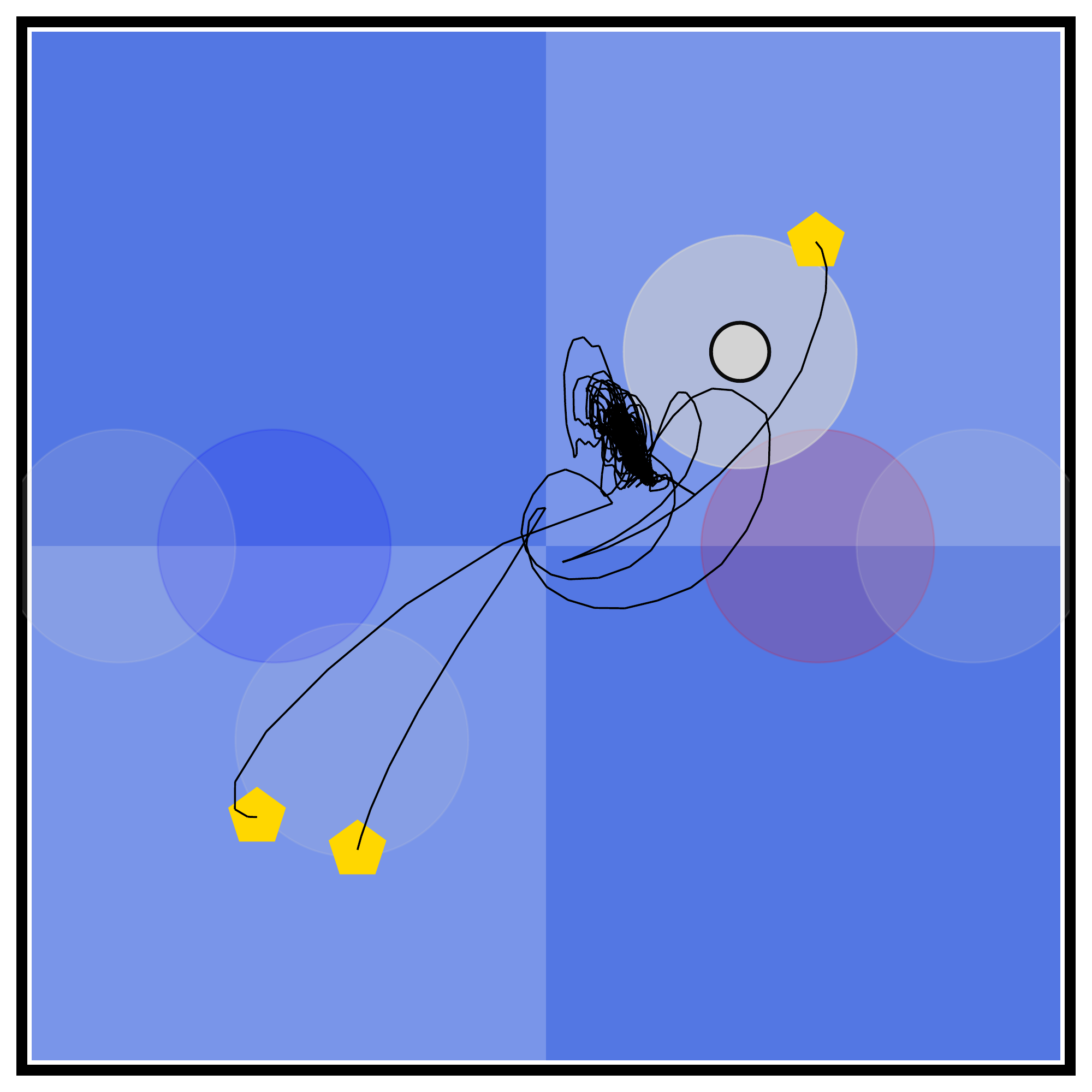}%
    \includegraphics[width=0.125\linewidth]{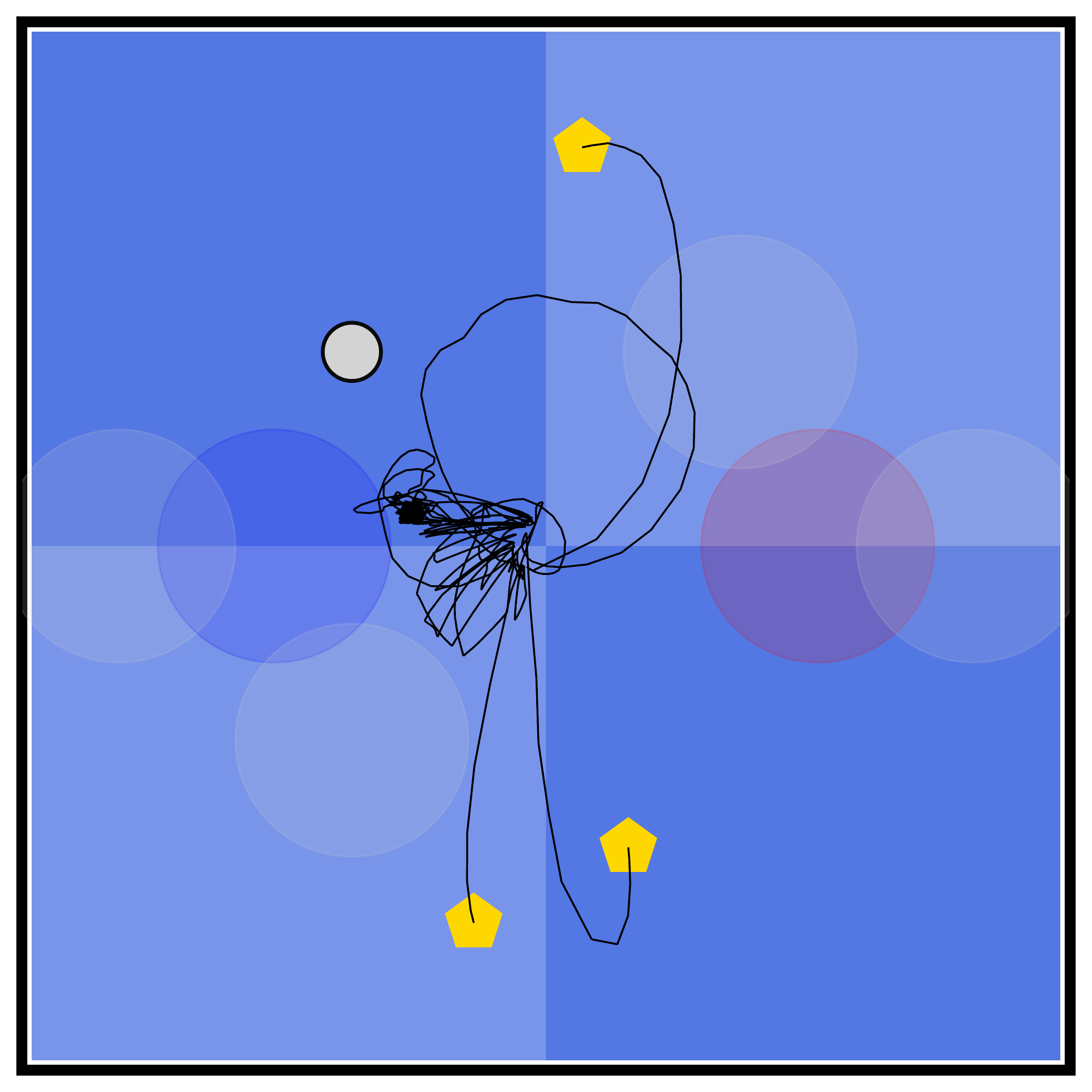}%
    \includegraphics[width=0.125\linewidth]{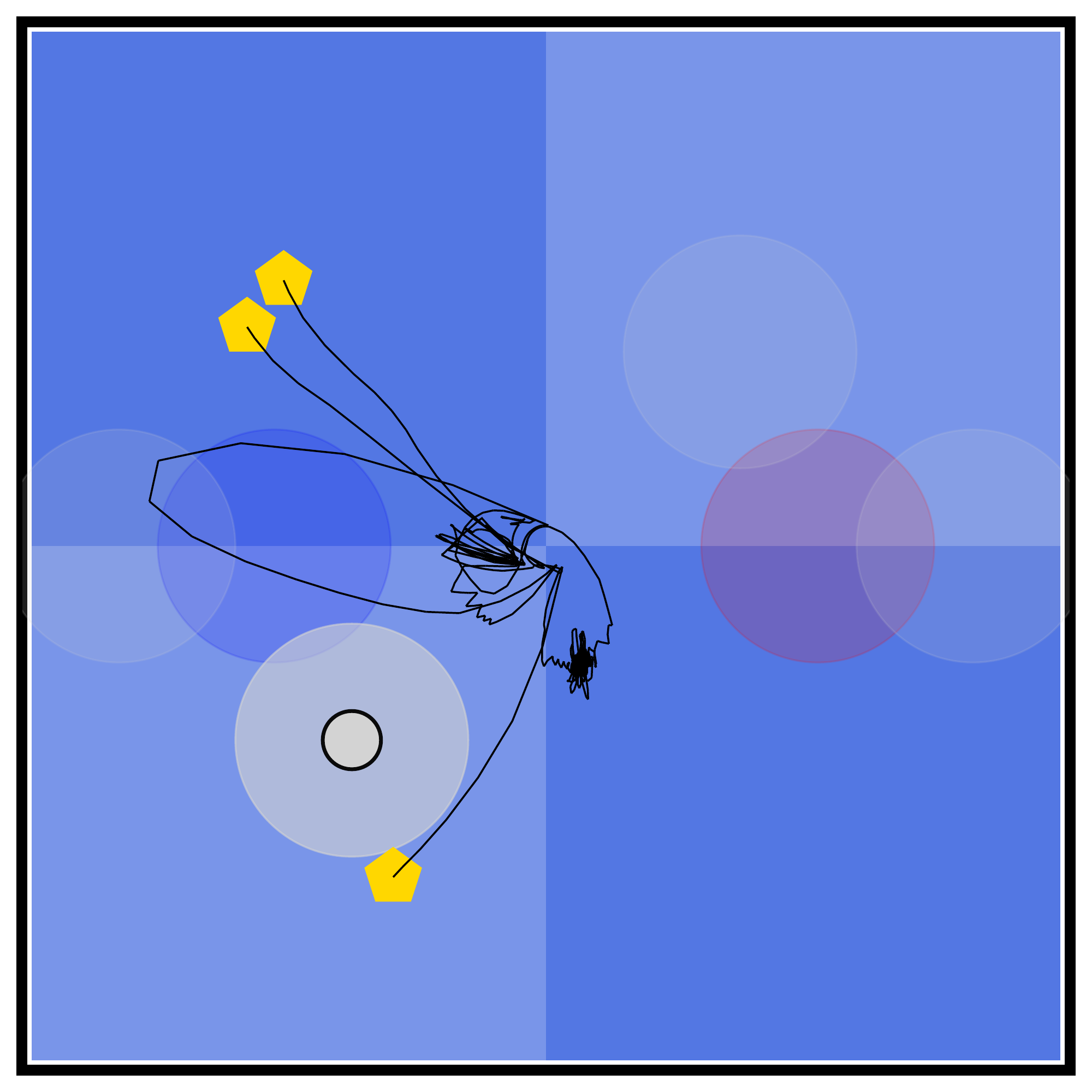}%
    \includegraphics[width=0.125\linewidth]{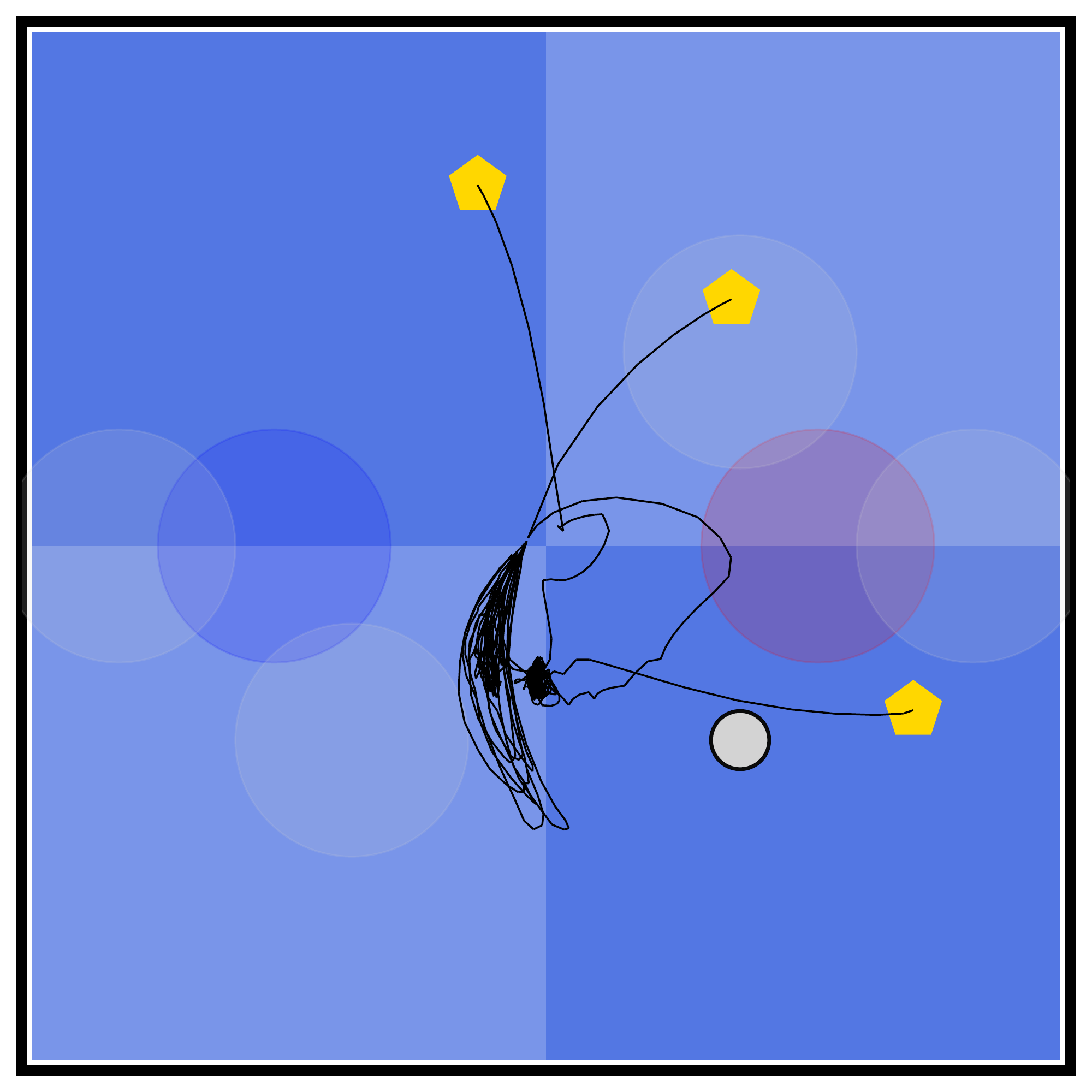}%
    \includegraphics[width=0.009\linewidth]{figs/figs_appendix/reacher_beta_3_axis.pdf}
    \caption{Evolutions of the robotic arm tip position in three rollouts of the reacher domain according to the GPI policy obtained after training on all 4 tasks. Here, all 8 test tasks are shown.}
    \label{fig:reacher_rollouts_test}
\end{figure}

A similar conclusion can also be drawn by observing the heat-maps of the learned mean-variance objectives in Figure \ref{fig:reacher_value_functions}. For SFC51, these objectives take the highest values precisely at the target locations, whereas for RaSFC51 these take the highest values slightly away from the targets in regions of low volatility. This is expected as the utility of hovering very close to a target location centered in a risky region should be lower than hovering outside the risky region, for a sufficiently risk-averse agent. Moreover, the first 4 rows correspond to training task values and the last 8 correspond to test task values. Because a similar pattern described above can also be observed in test tasks, the ability of SFs to generalize expected return estimates to novel task instances also extends to higher-order sufficient statistics, namely the variance of return. Finally, the aggregated plots located in the top half in Figure \ref{fig:reacher_variance_plots} show that RaSFC51 learns the return variance correctly after having trained on all 4 task instances. On the other hand, the SFDQN architecture that learns the covariance using the residual method (\ref{eqn:covariance_Bellman}) is unable to learn the variance correctly, likely due to the propagation of errors and overestimation bias in $\tilde{\bm{\psi}}^{\pi_i}(s,a)$ as discussed in the main paper.

\begin{figure}[!tb]
    \centering
    \includegraphics[width=0.46\linewidth]{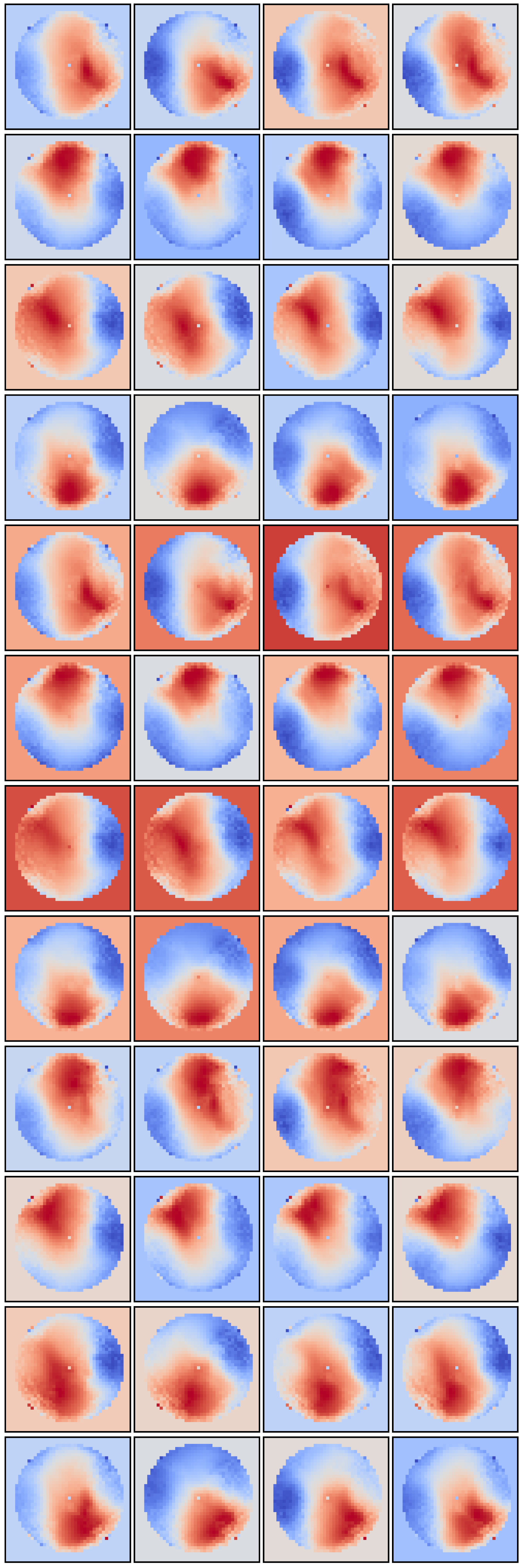}\hspace{0.05\linewidth}%
    \includegraphics[width=0.46\linewidth]{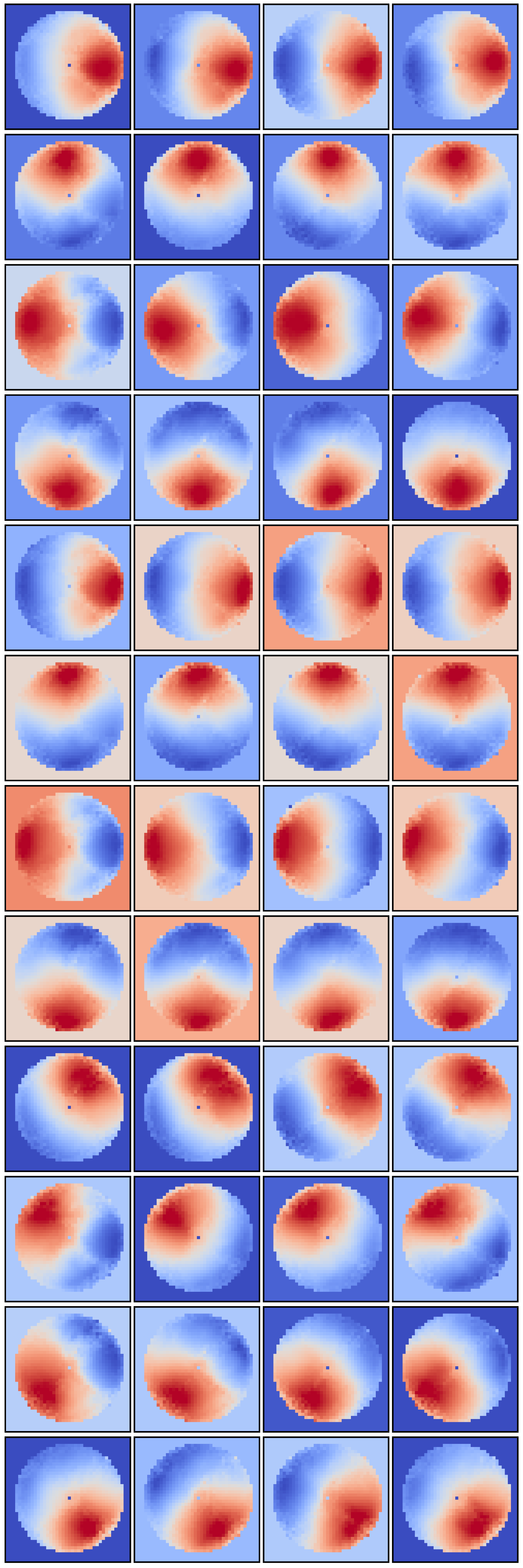}
    \caption{Each plot located in column $i$ and row $j$ illustrates the value of the mean-variance objective $\tr{\bm{\psi}^{\pi_i}(s,a)}\mathbf{w}_j - \tr{\mathbf{w}_j} \Sigma^{\pi_i}(s,a) \mathbf{w}_j$ as a function of the robotic arm tip position in $(x,y)$ coordinates for the reacher domain, after training each agent on all 4 tasks. In other words, the first 4 rows illustrate the value functions learned on the training task instances, while the last 8 rows illustrate the value functions learned on the test tasks. \textbf{Left:} mean-variance objective computed by RaSFC51 with $\beta = -3$. \textbf{Right:} mean-variance objective computed by SFC51.}
    \label{fig:reacher_value_functions}
\end{figure}

\begin{figure}[!tb]
    \centering
    \includegraphics[width=0.24\linewidth]{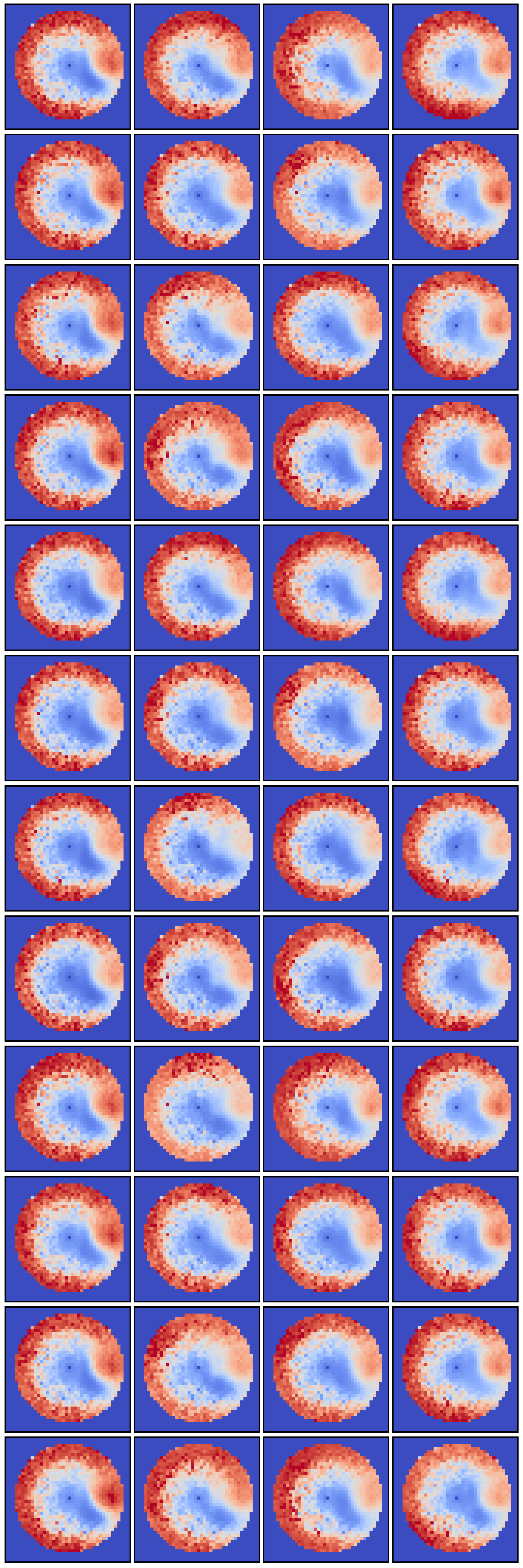}\hspace{0.01\linewidth}%
    \includegraphics[width=0.24\linewidth]{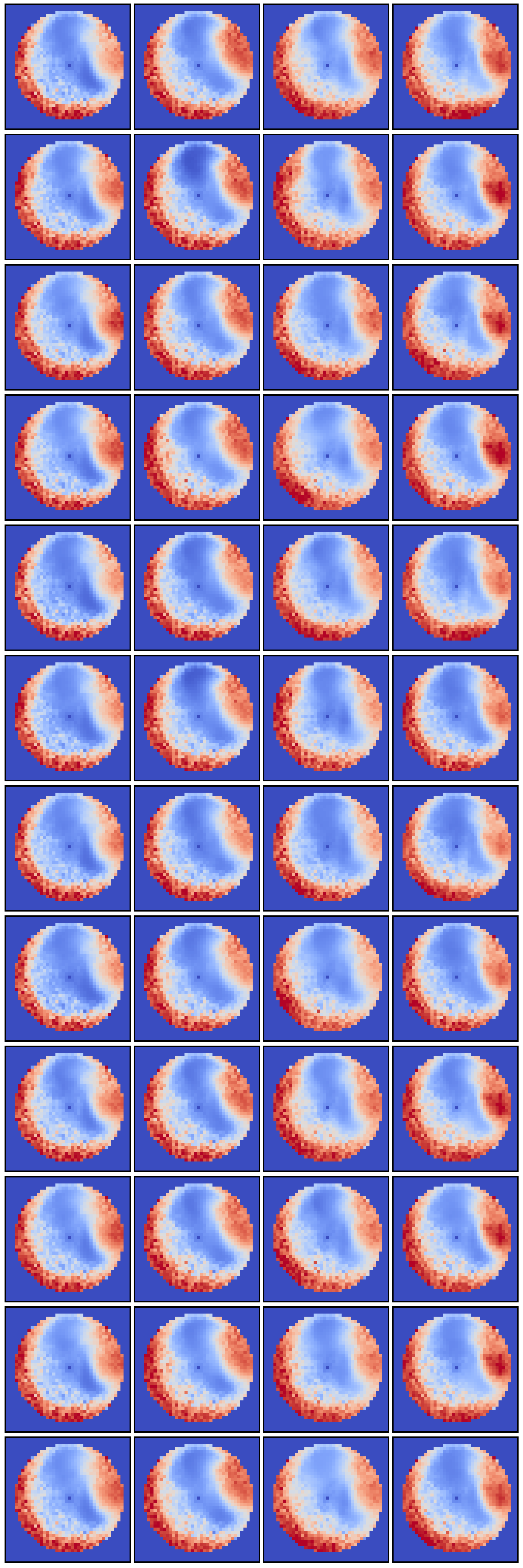}\hspace{0.01\linewidth}%
    \includegraphics[width=0.24\linewidth]{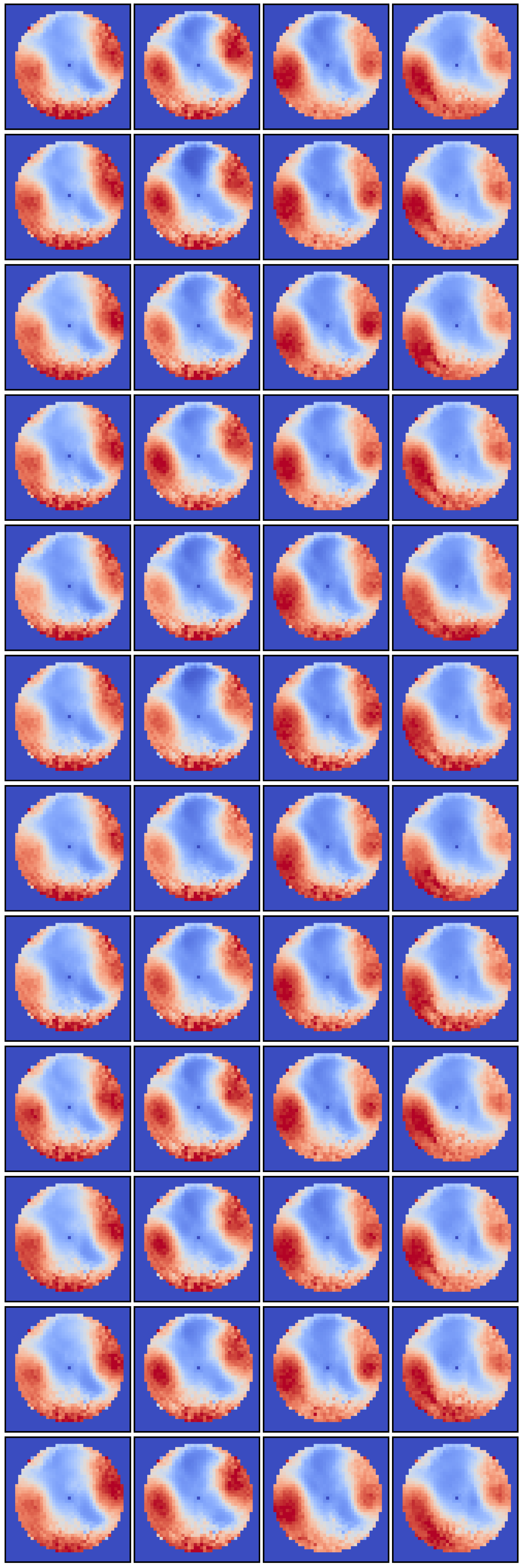}\hspace{0.01\linewidth}%
    \includegraphics[width=0.24\linewidth]{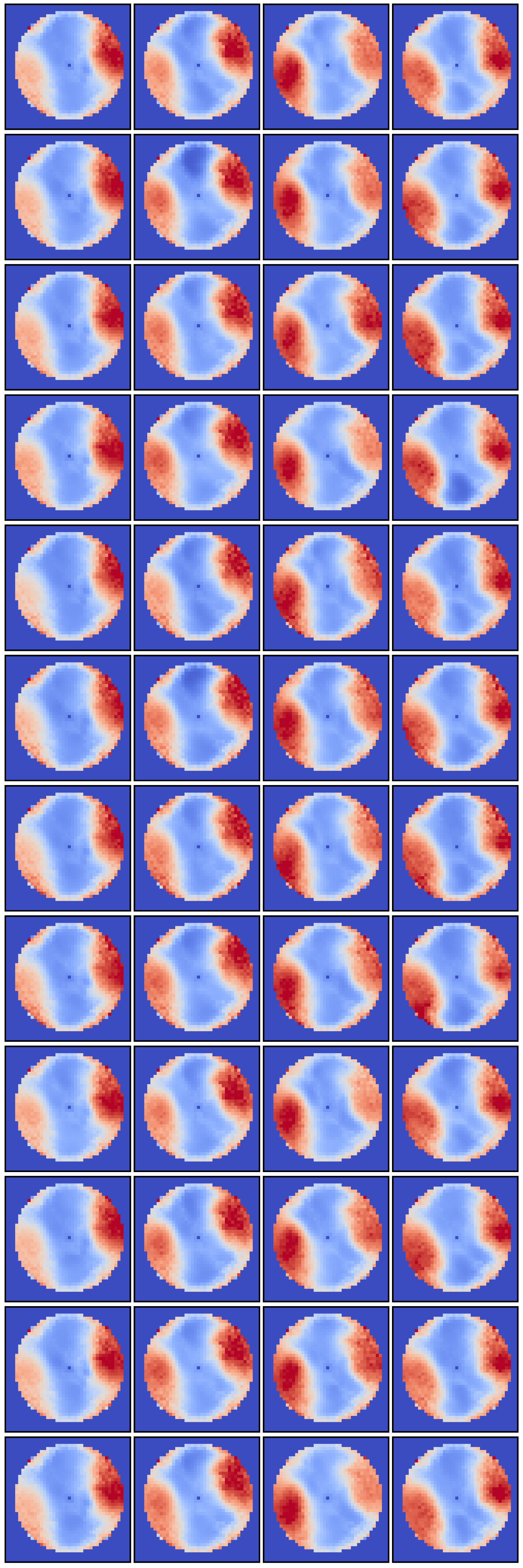}
    
    \includegraphics[width=0.24\linewidth]{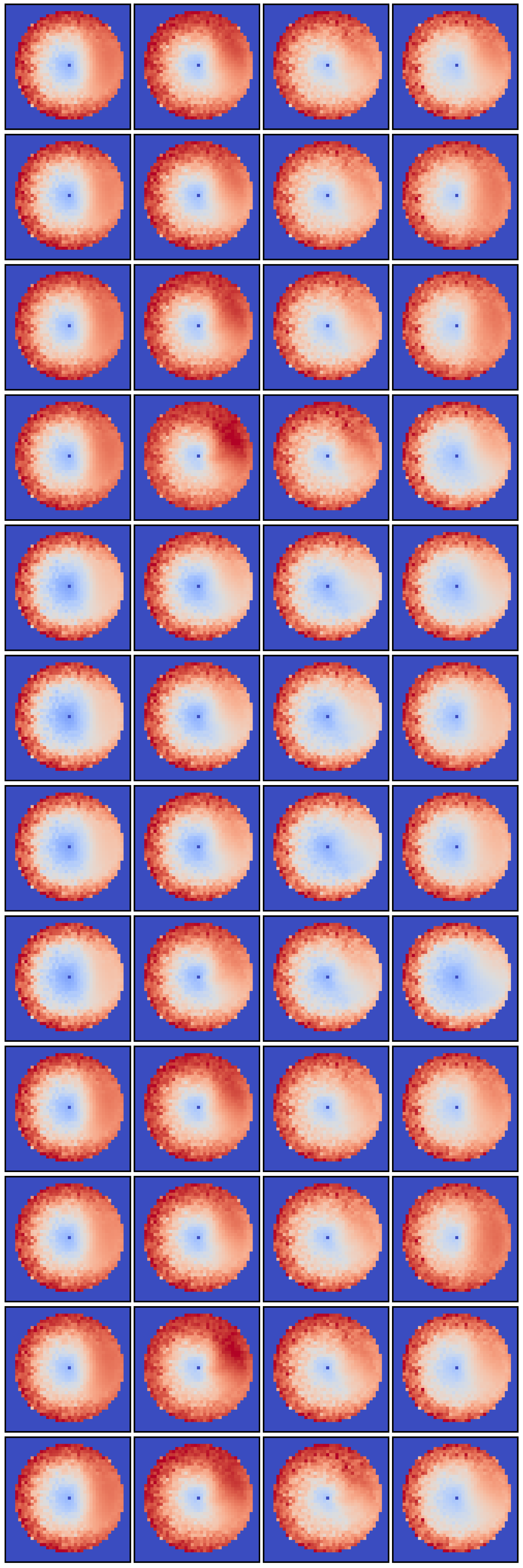}\hspace{0.01\linewidth}%
    \includegraphics[width=0.24\linewidth]{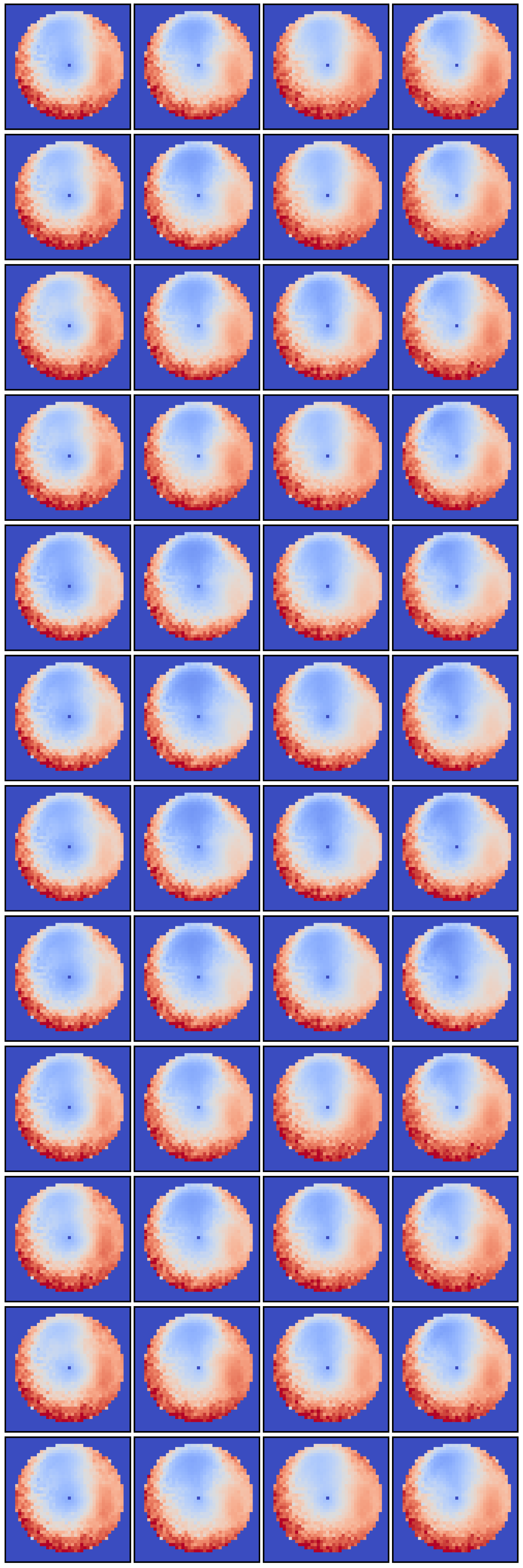}\hspace{0.01\linewidth}%
    \includegraphics[width=0.24\linewidth]{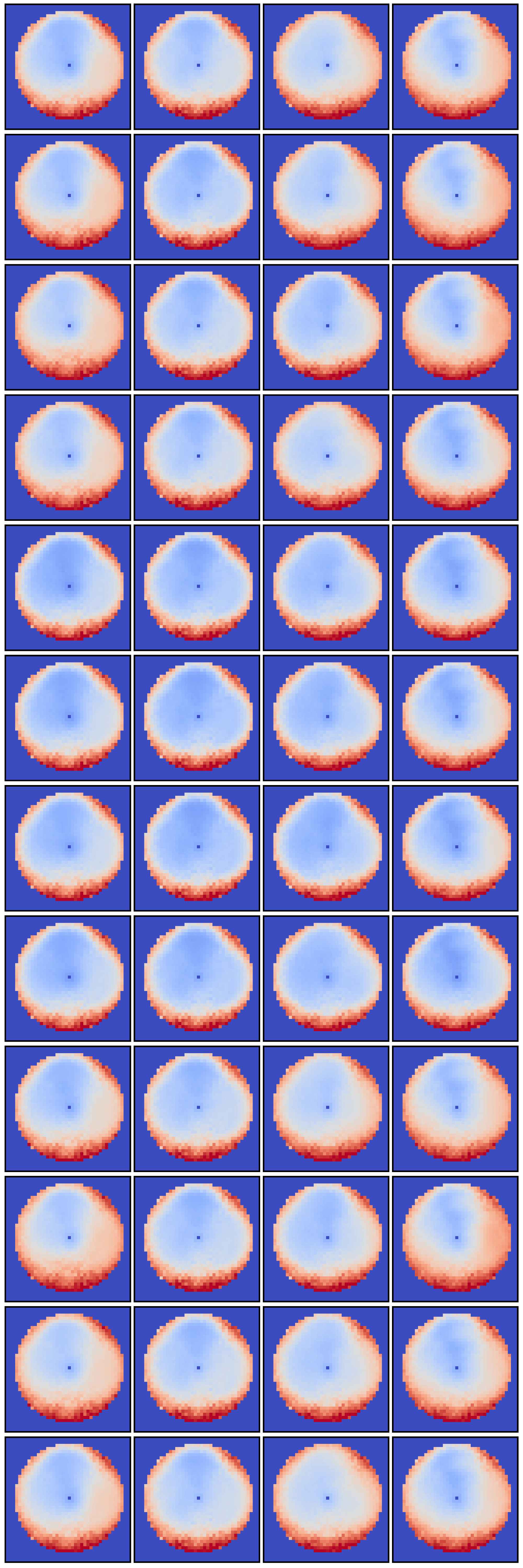}\hspace{0.01\linewidth}%
    \includegraphics[width=0.24\linewidth]{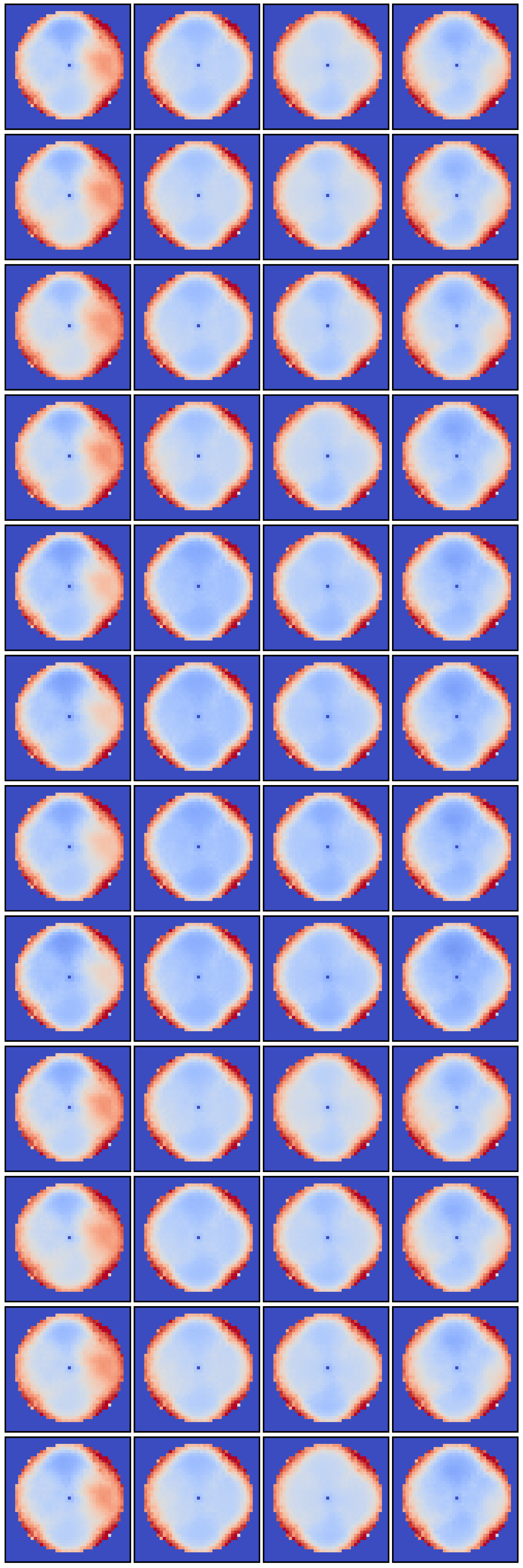}
    \caption{Each plot located in column $i$ and row $j$ illustrates the value of the variance $\tr{\mathbf{w}_j} \Sigma^{\pi_i}(s,a) \mathbf{w}_j$ as a function of the robotic arm tip position in $(x,y)$ coordinates for the reacher domain, after training on 1, 2, 3 and 4 source tasks (respectively, left to right). \textbf{Top:} variance computed by RaSFC51 with $\beta = -3$. \textbf{Bottom:} variance computed by SFDQN using (\ref{eqn:covariance_Bellman}).}
    \label{fig:reacher_variance_plots}
\end{figure}

\small
\bibliographystylesupp{plainnat}
\bibliographysupp{references}
\normalsize
\fi

\end{document}